%% file: NeurIPS2022_CAPO.tex
\documentclass{article}

\usepackage[accepted]{aistats2023}
\usepackage[round]{natbib}

\usepackage[utf8]{inputenc} % allow utf-8 input
\usepackage[T1]{fontenc}    % use 8-bit T1 fonts
\usepackage[colorlinks=true,citecolor=blue,linkcolor=blue]{hyperref}%
\usepackage{url}            % simple URL typesetting
\usepackage[flushleft]{threeparttable}
\usepackage{booktabs}       % professional-quality tables
\usepackage{amsfonts}       % blackboard math symbols
\usepackage{nicefrac}       % compact symbols for 1/2, etc.
\usepackage{microtype}      % microtypography
\usepackage{xcolor}         % colors
\usepackage{mathtools} % amsmath with fixes and additions
\usepackage{enumitem}       % for leftmargin=*
\usepackage{graphicx}
\usepackage{tikz}
\usepackage{bbm}
\usepackage{amsmath}
\usepackage{amssymb}
\usepackage{amsthm}
\usepackage{algorithm}
\usepackage{algorithmic}
\usepackage{cleveref}
\usepackage{comment}
\usepackage{epstopdf}
\DeclareMathOperator*{\argmax}{argmax} % no space, limits on side in displays
\Crefname{equation}{Equation}{Equations}
\crefname{equation}{equation}{equations}
\Crefname{figure}{Figure}{Figures}
\crefname{figure}{figure}{figures}
\Crefname{figure*}{Figure}{Figures}
\crefname{figure*}{figure}{figures}
\crefname{table}{table}{tables}

\newtheorem*{definition*}{Definition}
\newtheorem{lemma}{Lemma}
\newtheorem*{lemma*}{Lemma}
\newtheorem{remark}{Remark}
\newtheorem*{remark*}{Remark}

\newtheorem*{corollary*}{Corollary}
\newtheorem{theorem}{Theorem}
\newtheorem*{theorem*}{Theorem}
\newtheorem{prop}{Proposition}
\newtheorem*{prop*}{Proposition}

\newtheorem{condition}{Condition}

\newtheorem{claim}{Claim}
\DeclareMathOperator{\cA}{\mathcal{A}}
\DeclareMathOperator{\cS}{\mathcal{S}}

\DeclareMathOperator{\cP}{\mathcal{P}}
\DeclareMathOperator{\cB}{\mathcal{B}}

\DeclareMathOperator{\cO}{\mathcal{O}}
\DeclareMathOperator{\bbP}{\mathbb{P}}
\DeclareMathOperator{\bR}{\mathbb{R}}
\DeclareMathOperator{\bbI}{\mathbb{I}}

\DeclareMathOperator{\sgn}{sign}

\newcommand{\pch}[1]{\textcolor{black}{#1}}
\input{math_commands.tex}

\begin{document}

\twocolumn[

%\aistatstitle{Off-Policy Policy Optimization via Coordinate Ascent \\With Global Convergence}
\aistatstitle{Coordinate Ascent for Off-Policy RL with Global Convergence Guarantees}
%\aistatstitle{Off-Policy Actor-Critic via Coordinate Ascent Policy Optimization: \\
%A Derivative-Free Method With Global Convergence}
%\aistatstitle{Coordinate Ascent Policy Optimization: Decoupling Off-Policy Policy Improvement and Behavior Policies}

\aistatsauthor{ Hsin-En Su\textsuperscript{*1} \And Yen-Ju Chen\textsuperscript{*1} \And  Ping-Chun Hsieh\textsuperscript{1} \And Xi Liu\textsuperscript{2} }

\aistatsaddress{ \texttt{\{mru.11,pinghsieh\}@nycu.edu.tw, xliu1@fb.com}  \\ \textsuperscript{1}Department of Computer Science, National Yang Ming Chiao Tung University, Hsinchu, Taiwan\\
\textsuperscript{2}Applied Machine Learning, Meta AI, Menlo Park, CA, USA\\
\textsuperscript{*}Equal Contribution} ]
%\aistatsaddress{ littlecrazymouse.cs06@nctu.edu.tw \\ Department of Computer Science \\ National Yang Ming Chiao Tung University \\ Hsinchu, Taiwan \And littlecrazymouse.cs06@nctu.edu.tw \\ Department of Computer Science \\ National Yang Ming Chiao Tung University \\ Hsinchu, Taiwan \And littlecrazymouse.cs06@nctu.edu.tw \\ Department of Computer Science \\ National Yang Ming Chiao Tung University \\ Hsinchu, Taiwan } ]

%\maketitle
\begin{abstract}
We revisit the domain of off-policy policy optimization in RL from the perspective of coordinate ascent. One commonly-used approach is to leverage the off-policy policy gradient to optimize a surrogate objective -- the total discounted in expectation return of the target policy with respect to the state distribution of the behavior policy. However, this approach has been shown to suffer from the distribution mismatch issue, and therefore significant efforts are needed for correcting this mismatch either via state distribution correction or a counterfactual method. 
In this paper, we rethink off-policy learning via Coordinate Ascent Policy Optimization (CAPO), an off-policy actor-critic algorithm that decouples policy improvement from the state distribution of the behavior policy without using the policy gradient.
This design obviates the need for distribution correction or importance sampling in the policy improvement step of off-policy policy gradient.
We establish the global convergence of CAPO with general coordinate selection and then further quantify the convergence rates of several instances of CAPO with popular coordinate selection rules, including the cyclic and the randomized variants of CAPO.
We then extend CAPO to neural policies for a more practical implementation. Through experiments, we demonstrate that CAPO provides a competitive approach to RL in practice.
\end{abstract}

\input{1-intro-v4}

\input{2-preliminaries}

%\input{3-SGD}
\input{4-method}
\input{6-experiments}
\input{7-related}
\input{8-conclusion}

{
%\small
\bibliographystyle{unsrtnat}
\bibliography{neurips2022_reference}
}

\newpage
\appendix
\onecolumn
\section*{Appendix}
\input{9-appendix}

\end{document}

%% file: math_commands.tex
\usepackage{amsmath,amsfonts,bm}

\def\eqref#1{equation~\ref{#1}}
\def\1{\bm{1}}

\def\eps{{\epsilon}}

\DeclareMathAlphabet{\mathsfit}{\encodingdefault}{\sfdefault}{m}{sl}
\SetMathAlphabet{\mathsfit}{bold}{\encodingdefault}{\sfdefault}{bx}{n}

\newcommand{\E}{\mathbb{E}}

\DeclareMathOperator{\sign}{sign}

%% file: 1-intro-v4.tex
\section{Introduction}
% 1st Paragraph: PG
Policy gradient (PG) has served as one fundamental principle of a plethora of benchmark reinforcement learning algorithms \citep{degris2012offpac, lillicrap2016Continuous,gu2017q,mnih2016asynchronous}.
In addition to the empirical success, PG algorithms have recently been shown to enjoy provably global convergence guarantees in the \textit{on-policy} settings, including the true gradient settings \citep{agarwal2020theory,bhandari2019global,mei2020global,Cen2022fast} and the Monte-Carlo stochastic gradient settings \citep{liu2020improved,mei2021understanding}.
%Given the recent progress, one fundamental question to ask is whether similar convergence guarantees could also be established for the more practical \textit{actor-critic} stochastic PG algorithms.
However, on-policy PG is known to suffer from data inefficiency and lack of exploration due to the tight coupling between the learned target policy and the sampled trajectories. As a result, in many cases, \textit{off-policy} learning is preferred to achieve better exploration with an aim to either increase sample efficiency or address the committal behavior in the on-policy learning scenarios \citep{mei2021understanding,chung2021beyond}.
To address this, the off-policy PG theorem \citep{degris2012offpac,imani2018off,maei2018convergent} and the corresponding off-policy actor-critic methods, which are established to optimize a \textit{surrogate objective} defined as the total discounted return of the target policy in expectation with respect to the state distribution of the \textit{behavior policy}, has been proposed and widely adopted to decouple policy learning from trajectory sampling \citep{wang2016sample,gu2017interpolated,chung2021beyond,ciosek2018expected,espeholt2018impala}.
%Despite the recent progress, most of the existing works on establishing global convergence of PG mainly focus on the \textit{on-policy} learning settings, where the data samples are collected by following exactly the learned policy.
%Such committal behavior of stochastic on-policy PG results from the lack of exploration due to the tight coupling between the learned policy and the sampled trajectories.}

Despite the better exploration capability, off-policy PG methods are subject to the following fundamental issues: (i) \textit{Correction for distribution mismatch}: The standard off-policy PG methods resort to a surrogate objective, which ignores the mismatch between on-policy and the off-policy state distributions. Notably, it has been shown that such mismatch could lead to sub-optimal policies as well as poor empirical performance \citep{liu2020off}. As a result, substantial efforts are needed to correct this distribution mismatch \citep{imani2018off,liu2020off,zhang2020provably}.  (ii) \textit{Fixed behavior policy and importance sampling}: The formulation of off-policy PG presumes the use of a static behavior policy throughout training as it is designed to optimize a surrogate objective with respect to the behavior policy. However, in many cases, we do prefer that the behavior policy varies with the target policy (e.g., epsilon-greedy exploration) as it is widely known that importance sampling could lead to significant variance in gradient estimation, especially when the behavior policy substantially deviates from the current policy.
%Moreover, the off-policy PG expression involves an expectation over the state-action visitation distribution under the behavior policy. As a result, importance sampling is required to obtain an unbiased estimate of PG. However, it is widely known that importance sampling would lead to significant variance in gradient estimation, especially when the behavior policy substantially deviates from the current policy.
As a result, one fundamental research question that we would like to answer is: ``\textit{How to achieve off-policy policy optimization with global convergence guarantees, but without the above limitations of off-policy PG?}''

To answer this question, in this paper we take a different approach and propose an alternative off-policy policy optimization framework termed Coordinate Ascent Policy Optimization (CAPO), which revisits the policy optimization problem through the lens of {coordinate ascent}.
\pch{Our key insight is that \textit{the distribution mismatch and the fixed behavior policy issues in off-policy PG both result from the tight coupling between the behavior policy and the objective function in policy optimization.} To address this issue, we propose to still adopt the original objective of standard on-policy PG, but from the perspective of \textit{coordinate ascent} with the update coordinates determined by the behavior policy.
Through this design, we can completely decouple the objective function from the behavior policy while still enabling off-policy policy updates.}
Under the canonical tabular softmax parameterization, where each ``coordinate" corresponds to a parameter specific to each state-action pair, CAPO iteratively updates the policy by performing coordinate ascent for those state-action pairs in the mini-batch, without resorting to the full gradient information or any gradient estimation.  
While being a rather simple method in the optimization literature, coordinate ascent and the resulting CAPO enjoy two salient features that appear rather useful in the context of RL:
\vspace{-2mm}
\begin{itemize}[leftmargin=*]
    \vspace{-1mm}
    \item With the simple coordinate update, CAPO is capable of improving the policy by following any policy under a mild condition, directly enabling off-policy policy updates with an adaptive behavior policy. This feature addresses the issue of fixed behavior policy.
    \vspace{-1mm}
    \item Unlike PG, which requires having either full gradient information (the true PG setting) or an unbiased estimate of the gradient (the stochastic PG setting), updating the policy in a coordinate-wise manner allows CAPO to obviate the need for true gradient or unbiasedness while still retaining strict policy improvement in each update. As a result, this feature also obviates the need for distribution correction or importance sampling in the policy update.
\end{itemize}
\vspace{-1mm}

To establish the global convergence of CAPO, we need to tackle the following main challenges: (i) In the coordinate descent literature, one common property is that the coordinates selected for the update are either determined according to a deterministic sequence (e.g., cyclic coordinate descent) or drawn independently from some distribution (e.g., randomized block coordinate descent) \citep{nesterov2012efficiency}.
By contrast, given the highly stochastic and non-i.i.d. nature of RL environments, in the general update scheme of CAPO, we impose no assumption on the data collection process, except for the standard condition of infinite visitation to each state-action pair \citep{singh2000convergence,munos2016safe}.
(ii) The function of total discounted expected return is in general non-concave, and the coordinate ascent methods could only converge to a stationary point under the general non-concave functions. 
Despite the above, we are able to show that the proposed CAPO algorithm attains a globally optimal policy with properly-designed step sizes under the canonical softmax parameterization.
\pch{(iii) In the optimization literature, it is known that the coordinate ascent methods can typically converge slowly compared to the gradient counterpart. Somewhat surprisingly, we show that CAPO achieves comparable convergence rates as the true on-policy PG \citep{mei2020global}. Through our convergence analysis, we found that this can be attributed to the design of the state-action-dependent variable step sizes.}
% Neural CAPO

Built on the above results, we further generalize CAPO to the case of neural policy parameterization for practical implementation.
Specifically, Neural CAPO (NCAPO) proceeds by the following two steps: (i) Given a mini-batch of state-action pairs, we leverage the tabular CAPO as a subroutine to obtain a collection of reference action distributions for those states in the mini-batch. (ii) By constructing a loss function (e.g., Kullback-Leibler divergence), we guide the policy network to update its parameters towards the state-wise reference action distributions.
Such update can also be interpreted as solving a distributional regression problem.

\textbf{Our Contributions.}
In this work, we revisit off-policy policy optimization and propose a novel policy-based learning algorithm from the perspective of coordinate ascent. The main contributions can be summarized as follows:

\vspace{-2mm}
\begin{itemize}[leftmargin=*]
  \vspace{-1mm}
  \item We propose CAPO, a simple yet practical off-policy actor-critic framework with global convergence, and naturally enables direct off-policy policy updates with more flexible use of adaptive behavior policies, without the need for distribution correction or importance sampling correction to the policy gradient.
  %that requires no distribution correction that works with arbitrary behavior policies.
  \vspace{-1mm}
  \item We show that the proposed CAPO converges to a globally optimal policy under tabular softmax parameterization for general coordinate selection rules and further characterize the convergence rates of CAPO under multiple popular variants of coordinate ascent.
  We then extend the idea of CAPO to learning general neural policies to address practical RL settings.
  %\item We show that policy gradient suffers from inaccurate gradient direction, cause by inaccurate estimation of $d^\pi(s)$.
  \vspace{-1mm}
  \item Through experiments, we demonstrate that NCAPO achieves comparable or better empirical performance than various popular benchmark methods in the MinAtar environment \citep{young19minatar}.
  %\item Through CAPO, We connect the relation between NPG and the overlooked coordinate ascent methods in the literature of RL.
  %\vspace{-1mm}
\end{itemize}

\textbf{Notations}. Throughout the paper, we use $[n]$ to denote the set of integers $\{1,\cdots, n\}$. For any $x\in \mathbb{R}\backslash \{0\}$, we use $\sgn(x)$ to denote $\frac{x}{\lvert x\rvert}$ and set $\sgn(0)=0$. We use $\mathbb{I}\{\cdot\}$ to denote the indicator function.

%% file: 2-preliminaries.tex
\section{Preliminaries}
\label{section:prelim}
\textbf{Markov Decision Processes.} We consider an infinite-horizon Markov decision process (MDP) characterized by a tuple $(\cS, \cA, \mathcal{P}, r, \gamma, \rho)$, where (i) $\mathcal{S}$ denotes the state space, (ii) $\cA$ denotes a \textit{finite} action space, (iii) $\cP:\cS\times\cA\rightarrow \Delta(\cS)$ is the transition kernel determining the transition probability $\mathcal{P}(s'\rvert s,a)$ from each state-action pair $(s,a)$ to a next state $s'$, where $\Delta(\cS)$ is a probability simplex over $\cS$, (iv) $r: \mathcal{S} \times \mathcal{A} \rightarrow[0,1]$ is the reward function, %determines the expected immediate reward $r(s,a)$ obtained by taking an action $a$ at a given state $s$,
(v) $\gamma \in (0,1)$ is the discount factor, and (vi) \pch{$\rho$ is the initial state distribution.} 
%determines how much does the agent cares about the future reward, 
In this paper, we consider learning a stationary parametric stochastic policy denoted as $\pi_{\theta}:\cS \rightarrow \Delta(\cA)$, which specifies through a parameter vector $\theta$ the action distribution from a probability simplex $\Delta(\cA)$ over $\cA$ for each state.
For a policy $\pi_\theta$, the value function $V^{\pi_\theta}: \mathcal{S} \rightarrow \mathbb{R}$ is defined as the sum of discounted expected future rewards obtained by starting from state $s$ and following $\pi_\theta$, i.e.,
\begin{equation}
V^{\pi_\theta}(s):=\mathbb{E}\bigg[\sum_{t=0}^{\infty} \gamma^{t} r(s_{t}, a_{t})\bigg\vert \pi_\theta, s_{0}=s\bigg],
\end{equation}
where $t$ represents the timestep of the trajectory $\{(s_t, a_t)\}^{\infty}_{t=0}$ induced by the policy $\pi_\theta$ with the initial state $s_0=s$. 
%where $t \in [0, \infty)$ is the timestep.
The goal of the learner is to search for a policy that maximizes the following objective function as
\begin{equation}
V^{\pi_\theta}(\rho):=\E_{s\sim\rho}[V^{\pi_{\mathbf{\theta}}}(s)].
%J_{\mu}(\theta) = \sum_{s \in \mathcal{S}} d^{\pi_b}(s) V^{\pi_{\mathbf{\theta}}, \gamma}(s)
\end{equation}
%where the state distribution $\mu$ satisfies that $\mu(s)> 0, \forall s \in \mathcal{S}$.
For ease of exposition, we use $\pi^*$ to denote an optimal policy and let $V^*(s)$ be a shorthand notation for $V^{\pi^*}(s)$.
Moreover, for any given policy $\pi_\theta$, we define the $Q$-function  $Q^{\pi_\theta}: \mathcal{S} \times \mathcal{A} \rightarrow \mathbb{R}$ as
\begin{equation}
Q^{\pi_\theta}(s, a):=\mathbb{E}\bigg[\sum_{t=0}^{\infty} \gamma^{t} r(s_{t}, a_{t})\bigg\vert \pi, s_{0}=s,a_0=a\bigg].
\end{equation}
We also define the advantage function $A^{\pi_\theta}: \mathcal{S} \times \mathcal{A} \rightarrow \mathbb{R}$ as
\begin{equation}
A^{\pi_{\theta}}(s, a):=Q^{\pi_\theta}(s, a)-V^{\pi_{\theta}}(s),
\end{equation}
which reflects the relative benefit of taking the action $a$ at state $s$ under policy $\pi_{\theta}$. Moreover, throughout this paper, we use $m$ as the index of the training iterations and use $\pi_m$ and $\pi_{\theta_m}$ interchangeably to denote the parameterized policy at iteration $m$.
%hinting the preference of the associate action at a given state.
%One useful property of the advantage function is $\sum_{a\in\cA}\pi_\theta(a \rvert s)A^{\pi_\theta}(s, a) = 0$.

%Notice the relationship between the value function and Q-value function
%\begin{equation}
%\sum_{a}\pi(a | s)Q(s, a) = V(s)
%\end{equation}

%by definition of advantage we also have
%\begin{equation}
%\label{eq:piadv0}
%\sum_{a}\pi(a | s)A(s, a) = 0
%\end{equation}
%This well-known equation is a crucial foundation for the proof of global convergence in section \ref{section:method}.
%\subsection{Policy Gradient}
\textbf{Policy Gradients.} 
The policy gradient is a popular policy optimization method that updates the parameterized policy $\pi_\theta$ by applying gradient ascent with respect to an objective function $V^{\pi_\theta}(\mu)$, where $\mu$ is some starting state distribution.
%i.e., in each iteration $m$,
%\begin{equation}
%\theta_{m+1} = \theta_{m} + \eta \nabla_{\theta} J_{\mu}(\theta).
%\end{equation}
The standard stochastic policy gradient theorem states that the policy gradient $\nabla_{\theta} V^{\pi_\theta}({\mu})$ takes the form as \citep{sutton1999policy}
\begin{align}
\label{eq:pg}
&\nabla_{\theta} V^{\pi_\theta}({\mu})\nonumber\\
&\hspace{-6pt}= \frac{1}{1-\gamma} \mathbb{E}_{s \sim d_{\mu}^{\pi_{\theta}},a \sim \pi_{\theta}(\cdot \mid s)}\big[\nabla_{\theta} \log \pi_{\theta}(a \rvert s) A^{\pi_{\theta}}(s, a)\big],
\end{align}
where the outer expectation is taken over the \textit{discounted state visitation distribution} under $\mu$ as
\begin{equation}
    d_{\mu}^{\pi_\theta}(s):=\mathbb{E}_{s_{0} \sim \mu}\bigg[(1-\gamma) \sum_{t=0}^{\infty} \gamma^{t} \bbP\big(s_{t}=s \rvert s_{0},{\pi_\theta}\big)\bigg].
\end{equation}
%\begin{equation}
%d_{\mu}^{\pi}(s):=\mathbb{E}_{s_{0} \sim \mu}\left[d_{s_{0}}^{\pi}(s)\right]    
%\end{equation}
%where $d_{s_{0}}^{\pi}(s)$ is the distribution starting at state $s_0$, defined as:
%\begin{equation}
%d_{s_{0}}^{\pi}(s):=(1-\gamma) \sum_{t=0}^{\infty} \gamma^{t} \bbP\left(s_{t}=s \rvert s_{0},{\pi}\right).    
%\end{equation}
Note that $d_{\mu}^{\pi_\theta}(s)$ reflects how frequently the learner would visit the state $s$ under $\pi_\theta$.
% [To Hsin-En]: We need to mention SPG here
%The exact PG requires the information about the state visitation distribution $d_{\mu}^{\pi_\theta}$, which is difficult to derive in the learning settings.

Regarding PG for off-policy learning, the learner's goal is to learn an optimal policy $\pi^*$ by following a behavior policy.
\citet{degris2012offpac} proposed to optimize the following surrogate objective defined as
\begin{equation}
    J^{\pi_\theta}(\beta):=\sum_{s\in\cS}\bar{d}^{\beta}(s)V^{\pi_\theta}(s),
    \label{eq:offpac_fixed}
\end{equation}
where $\beta:\cS\rightarrow \Delta(\cA)$ is a \textit{fixed} behavior policy and $\bar{d}^{\beta}(s)$ is the stationary state distribution under $\beta$ (which is assumed to exist in \citep{degris2012offpac}).
The resulting off-policy PG enjoys a closed-form expression as
\begin{align}
    \nabla_\theta J^{\pi_\theta}(\beta)=\E_{s\sim \bar{d}^{\beta}(s)}\Big[&\sum_{a\in\cA}\Big(\nabla_\theta \pi_\theta(a\rvert s)Q^{\pi_\theta}(a\rvert s)\nonumber\\
    &+\pi_{\theta}(a\rvert s)\nabla_{\theta}Q^{\pi_\theta}(s,a)\Big)\Big].\label{eq:OPPG}
\end{align}
Moreover, \citet{degris2012offpac} showed that one can ignore the term $\pi_{\theta}(a\rvert s)\nabla_{\theta}Q^{\pi_\theta}(s,a)$ in (\ref{eq:OPPG}) under tabular parameterization without introducing any bias and proposed the corresponding Off-Policy Actor-Critic algorithm (Off-PAC)
\begin{equation}
    \theta_{m+1}=\theta_{m}+\eta\cdot \omega_m(s,a)Q^{\pi_{m}}(s,a)\nabla_{\theta}\log \pi_{\theta_m}(s,a),
\end{equation}
where $s$ is drawn from $\bar{d}^{\beta}$, $a$ is sampled from $\beta(\cdot\rvert s)$, and $\omega_m(s,a):=\frac{\pi_m(a\rvert s)}{\beta(a\rvert s)}$ denotes the importance ratio.
Subsequently, the off-policy PG has been generalized by incorporating state-dependent emphatic weightings \citep{imani2018off} and introducing a counterfactual objective \citep{zhang2019general}.
%In practice, $\nabla_{\theta}J_{\mu}(\theta)$ is approximated by a stochastic estimate based on sampled trajectories under $\pi_{\theta}$, i.e.,
%\begin{equation}
%    \nabla_{\theta} J_{\mu}(\theta) \approx \hat{\nabla}_{\theta} J_{\mu}(\theta)=\frac{1}{(1-\gamma)}\cdot\frac{1}{\lvert \cD\rvert} \sum_{(s,a)\in \cD}\nabla_{\theta} \log \pi_{\theta}(a \rvert s) A^{\pi_{\theta}}(s, a).
%\end{equation}
%As suggested by (\ref{eq:pg}), we estimate such expectation over $d^\pi$ by sampling trajectories following $\pi$ in practice.  m  
%\pch{[TO BE REVISED] However we show in \Cref{section:exp:dpi} that to correctly approximate such distribution will require large number of samples, which is rarely satisfied in real applications.}

% \subsection{On-policy Stochastic Policy Gradient}
% In practice, policy gradient is estimated through sampling. For every iteration $m$, we will collect $|\mathcal{B}|$ transitions $(s_i, a_i, r_{i}, s_{i+1})$ by following the policy $\pi_m$.

% The wildly used unbiased policy gradient estimator $\hat{g}$ (which we refer to vanilla SPG) is known as:
% \begin{equation}
% \label{eq:vanilla_pg}
% %
% \hat{g} = \frac{1}{|\mathcal{B}|} \sum_{i} Q\left(s_{i}, a_{i}\right) \nabla_{\theta} \log \pi_{\theta_{m}}\left(a_{i} \mid s_{i}\right)
% %
% \end{equation}
% The policy $\pi_{\theta_m}$ is then updated as:
% \begin{equation}
% \theta_{m+1} = \theta_{m} + \alpha \cdot \hat{g}    
% \end{equation}

%\subsection{Coordinate Descent}
\textbf{Coordinate Ascent.} Coordinate ascent (CA) methods optimize a parameterized objective function $f(\theta):\mathbb{R}^{n}\rightarrow\mathbb{R}$ by iteratively updating the parameters along coordinate directions or coordinate hyperplanes.
Specifically, in the $m$-th iteration, the CA update along the $i_m$-th coordinate is
%In this work, we focus on the case where parameters are updated along a single random coordinate $i$:
\begin{equation}
    \theta_{m+1}= \theta_{m} +\eta\cdot [\nabla_\theta f(\theta)]_{i_m} e_{i_m},
\end{equation}
where $e_{i_m}$ denotes the one-hot vector of the $i_m$-th coordinate and $\eta$ denotes the step size. The main difference among the CA methods mainly lies in the selection of coordinates for updates. Popular variants of CA methods include:
\vspace{-2mm}

\begin{itemize}[leftmargin=*]
    \item {\textbf{Cyclic CA}: The choice of coordinate proceeds in a predetermined cyclic order \citep{saha2013nonasymptotic}. For example, one possible configuration is $i_m\leftarrow m \bmod n$.}
    %\vspace{-2mm}
    %\item \pch{\textbf{Almost-Cyclic CA}:}
    \vspace{-2mm}
    \item {\textbf{Randomized CA}: In each iteration, one coordinate is drawn randomly from some distribution with support $[n]$ \citep{nesterov2012efficiency}.}
\end{itemize}
\vspace{-2mm}

Moreover, the CA updates can be extended to the \textit{blockwise} scheme \citep{tseng2001convergence,beck2013convergence}, where multiple coordinates are selected in each iteration.
{Despite the simplicity, the CA methods have been widely used in variational inference \citep{jordan1999introduction} and large-scale machine learning \citep{nesterov2012efficiency} due to its parallelization capability.
To the best of our knowledge, CA has remained largely unexplored in the context of policy optimization.}
%Different algorithms has different ways to select which coordinate $i$ to update in each iteration $m$. For example, by choosing $i$ uniformly at random, we have:
%$$
%\mathbb{E}\left[g^m\right]=\frac{1}{n} \sum_{i=1}^{n}\left[\nabla f\left(\theta\right)\right]_{i} e_{i}=\frac{1}{n}\nabla f\left(\theta\right),
%$$
%which shows that $g^m$ indeed serves as an unbiased estimator of the gradient $\nabla f\left(\theta\right)$.

% Commenting below because we do not utilize parallel properties of CD in this work. 
%
% Such property makes coordinate descent a good alternative when facing large scale optimization problem, when the dimensionality is too large to calculate the entire gradient efficiently.
% For related literatures, please refer to \cref{app:related:coord}.

%% file: 4-method.tex
\section{Methodology}
\label{section:method}
In this section, we present the proposed CAPO algorithm, which improves the policy through coordinate ascent updates.
Throughout this section, we consider the class of tabular softmax policies.
Specifically, for each state-action pair $(s, a)$, let $\theta(s,a)$ denote the corresponding parameter. The probability of selecting action $a$ given state $s$ is given by $\pi_\theta(a \rvert s) =\frac{\exp({\theta(s,a)})}{\sum_{a' \in \mathcal{A}}\exp({\theta(s,a')})}$.

% \subsection{\tmp{The effect of stochasticity sampling in  NPG}}

% To overcome such limitation, we present our main algorithm, coordinate ascent policy optimization (CAPO), that guarantees global convergence in general MDPs even under the stochastic setting.
%
%
%
\subsection{Coordinate Ascent Policy Optimization}
\label{subsection:CAPO}

% Under the softmax policy, let $\mu$ be the initial state distribution. The coordinate ascent direction of the value function has the form (\citet{agarwal2020theory}, \textbf{Lemma C.1}):

% \begin{equation}
%     \frac{\partial V^{\pi_{\theta}}(\mu)}{\partial \theta_{s, a}}=\frac{1}{1-\gamma} d_{\mu}^{\pi_{\theta}}(s) \pi_{\theta}(a \mid s) A^{\pi_{\theta}}(s, a)
% \end{equation}

% As coordinate ascent updates along only a single coordinate, its direction can be decomposed into two part, the 
% \textbf{step size}, $\bm{d_{\mu}^{\pi_{\theta}}(s) \pi_{\theta}(a \mid s)|A^{\pi_{\theta}}(s, a)|}$, and the \textbf{direction}, $\bm{\sign(A^{\pi_{\theta}}(s, a))}$. 
% %
% Since the direction function $\sign(A^{\pi_{\theta}}(s, a))$ has magnitude of $1$, we can control the magnitude of the coordinate update by a learning rate function $\bm{\alpha(s, a)}$ that varies with inputs state and action. 
% %
% Thus it is nature to design a update form which we called general Coordinate Ascent Policy Optimization (CAPO) form as:

% Motivated by the argument presented in the previous section, to overcome the potential vicious cycle of sampling and updating of SNPG, we present our method CAPO in this section.
To begin with, we present the general policy update scheme of CAPO. The discussion about the specific instances of CAPO along with their convergence rates will be provided subsequently in \Cref{subsection:convergence_rate}. 
To motivate the policy improvement scheme of CAPO, we first state the following lemma \citep{agarwal2020theory,mei2020global}.
%To demonstrate the off-policy capability of CAPO, we will focus on the off-policy setting in this section. 
%
%We also assume Q-value function is given by oracle. 
\begin{lemma}
\label{lemma:PG}
Under tabular softmax policies, the standard policy gradient with respect to $\theta$ is given by
\begin{equation}
    \frac{\partial V^{\pi_\theta}(\mu)}{\partial \theta(s,a)}=\frac{1}{1-\gamma}d^{\pi_\theta}_{\mu}(s)\cdot \pi_{\theta}(a\rvert s)\cdot A^{\pi_\theta}(s,a).
\end{equation}
\end{lemma}

Based on Lemma $\ref{lemma:PG}$, we see that the update direction of each coordinate is completely determined by the \textit{sign} of the advantage function. Accordingly, the proposed general CAPO update scheme is as follows: In each update iteration $m$, let ${B}_m$ denote the mini-batch of state-action pairs sampled by the behavior policy. The batch ${B}_m$ determines the coordinates of the policy parameter to be updated. Specifically, the policy is updated by
\begin{align}
\label{eq:CAPO_form}
&\theta_{m+1}(s, a)\nonumber\\
&=\theta_m(s, a) +\alpha_m(s, a)\mathbb{I}\{(s,a) \in {B}_m\} \cdot \sign\left(A^{\pi_{\theta_m}}(s, a)\right),
\end{align}
where $\alpha_m: \cS\times \cA\rightarrow \bR_{+}$ is the function that controls the \textit{magnitude} of the update and plays the role of the learning rate, the term $\sgn(A^{\pi_{\theta_m}}(s,a))$ controls the update \textit{direction}, and ${B}_m$ is the sampled batch of state-action pairs in the $m$-th iteration and determines the \textit{coordinate selection}. Under CAPO, only those parameters associated with the sampled state-action pairs will be updated accordingly, as suggested by (\ref{eq:CAPO_form}). 
Based on this, we could reinterpret $B_m$ as produced by a \textit{coordinate generator}, which could be induced by the behavior policies.

\begin{remark}
\normalfont Note that under the general CAPO update, the learning rate $\alpha$ is state-action-dependent. This is one salient difference from the learning rates of conventional coordinate ascent methods in the optimization literature \citep{nesterov2012efficiency,saha2013nonasymptotic}. As will be shown momentarily in \Cref{subsection:analysis}, this design allows CAPO to attain global optimality without statistical assumptions about the samples (i.e., the selected coordinates). 
%The design choice is made even more clear in \Cref{subsection:connecting_npg}, where we compare a similar update using fixed value of learning rate, resulting in suboptimal convergence.
On the other hand, while it appears that the update rule in (\ref{eq:CAPO_form}) only involves the sign of the advantage function, the magnitude of the advantage $\lvert A(s,a)\rvert$ could also be taken into account if needed through $\alpha(s,a)$, which is also state-action-dependent. As a result, (\ref{eq:CAPO_form}) indeed provides a flexible expression that separates the effect of the sign and magnitude of the advantage. 
Interestingly, as will be shown in the next subsections, we establish that CAPO can achieve global convergence without the knowledge of the magnitude of the advantage.
\end{remark}

\begin{remark}
\normalfont Compared to the off-policy PG methods \citep{degris2012offpac,wang2016sample,imani2018off}, one salient property of CAPO is that it allows \textit{off-policy learning} through coordinate ascent on the original \textit{on-policy} total expected reward $\E_{s\sim\rho}[V^{\pi}(s)]$, instead of the \textit{off-policy} total expected reward over the discounted state visitation distribution induced by the behavior policy.
On the other hand, regarding the learning of a critic, similar to the off-policy PG methods, CAPO can be integrated with any off-policy policy evaluation algorithm, such as Retrace \citep{munos2016safe} or V-trace \citep{espeholt2018impala}.
\end{remark}
% \begin{remark} The reason behind the $\sign$ function in \eqref{eq:CAPO_form} is that since the learning rate $\alpha$ is state-action dependent, one can always find a equivalent learning rate function that matches the magnitude of $f$ in \eqref{eq:snpg_softmax_update_simp}, since both $\alpha$ and $f$ is a function of $s$ and $a$.
% \end{remark}

% \begin{remark}
% The advantage function can be estimated using any off-policy evaluation algorithms, either under one single behavior policy or multiple behavior policies \cite{chen2019infinite}.
% \end{remark}
% \begin{remark} $f$ is chosen to be a sign-preserving function such that $\sign(f(A(s, a)) = \sign(A(s, a))$
% \label{remark:f}

% Notice that since $\alpha$ controls the scale of the update, the magnitude of $f(A(s, a))$ can be merge into $\alpha$. e.g.\ $\alpha(s, a) f\left(A(s, a)\right) = |f\left(s, a\right)|\alpha(s, a) \sign\left(f\left(A(s, a)\right)\right) = \alpha^{\prime}(s, a) \sign\left(f\left(A(s, a)\right)\right)$. 
% %
% As a result, $f$ should be sign-preserving.

% In section \ref{subsection:tabular} we directly choose $f$ to be the sign function, this allows us to focus on characterizing $\alpha$, as one can show the equivalence for arbitrary sign-preserving function $f$. Note that the choice of $\alpha$ directly impacts the global convergence analysis, details can be found in \ref{subsection:analysis}.
% \end{remark}

\input{./method_sub/tabular}

\input{./method_sub/analysis}

\input{./method_sub/convergence_rate}

\input{./method_sub/onCAPO}

\input{./method_sub/neural}

%% file: method_sub/tabular.tex
\subsection{Asymptotic Global Convergence of CAPO With General Coordinate Selection}
\label{subsection:analysis}

% In this section, we will show that under tabular softmax policy, CAPO converges to global optimum in probability.
%
% For simplicity, we also assume that each $(s, a)$-pair is unique for each batch $\mathcal{B}$. 
%
% However, with small adjustment to the proofs, one can show that the statements is true when $(s,a)$-pairs is not unique.
In this section we discuss the convergence result of CAPO under softmax parameterization. 
In the subsequent analysis, we assume that the following \Cref{condition:sa} is satisfied.

\begin{condition}
\label{condition:sa}
$\lim_{M \rightarrow \infty} \sum^{M}_{m=1} \bbI\{(s,a) \in {B}_m\} \rightarrow \infty$
\end{condition}
Note that \Cref{condition:sa} is rather mild as it could be met by exploratory behavior policies (e.g., $\epsilon$-greedy policies) given the off-policy capability of CAPO.
Moreover, Condition \ref{condition:sa} is similar to the standard condition of infinite visitation required by various RL methods \citep{singh2000convergence,munos2016safe}. 
Notably, \Cref{condition:sa} indicates that under CAPO the coordinates are not required to be selected by following a specific policy, as long as infinite visitation to every state-action pair is satisfied.
This feature naturally enables flexible off-policy learning, justifies the use of a replay buffer, and enables the flexibility to decouple policy improvement from value estimation. 

%\Cref{condition:sa} impose very little restriction but infinite exploration on state-visitation distribution, as a result, we assume that data are given by a \textit{generator} $\mathcal{G}$ that output $|B_m|$ of state-action pairs by any predefined behavior that satisfy \Cref{condition:sa}.

%Algorithm \ref{algo:CAPO} provides the pseudocode for CAPO, where in each episode $m$ we select $|\mathcal{B}|$ states from some generator $\mathcal{G}$, execute action $a \sim \pi_{m}(\cdot | s)$ to obtain transition $(s, a, r, s^{\prime})$, and update the policy by (\ref{eq:CAPO_form}).

% One major property of CAPO is, unlike policy gradient methods, CAPO does not require to follow the state-visitation distribution $d^{\pi}(s)$. 
% %
% Instead, CAPO only requires the visited state-action pair satisfies condition \ref{condition:sa}:

We first show that CAPO guarantees strict improvement under tabular softmax parameterization.
\begin{lemma}[Strict Policy Improvement]
%Under general CAPO updating form (\eqref{eq:CAPO_form}), where $\alpha(s, a) > 0$, $V^{\pi_{m+1}}(s) \geq V^{\pi_{m}}(s), \forall s \in S, m > 0$.  
Under the CAPO update given by (\ref{eq:CAPO_form}), we have $V^{\pi_{m+1}}(s) \geq V^{\pi_{m}}(s)$, for all $s \in S$, for all $m\in \mathbb{N}$.  
\label{lemma:strict_improvement}
\end{lemma}
\begin{proof}
The proof can be found in Appendix \ref{app:proof_strict_improvement}.
\end{proof}

% As discussed in section \ref{section:method}, the magnitude of the update direction can be controlled by the learning rate $\alpha$, thus it is intuitive to choose the simplest sign-preserving function for $f$:
% \begin{equation}
% f(x) = \sign(x)
% \end{equation}

% Note that CAPO eliminates the need to estimate the complicated policy gradient $\nabla_{\theta} J(\theta) =\nabla_{\theta} \sum_{s \in \mathcal{S}} d^{\pi}(s) \sum_{a \in \mathcal{A}} A^{\pi}(s, a) \pi_{\theta}(a \mid s)$. It is not only expensive to obtain an accurate state-visitation distribution, $A^\pi(s,a)$ is infamous to have erroneous value. As shown in the toy experiment, erroneous $A^\pi(s,a)$ can lead to dramastically inaccurate gradient direction.

% On the other hand, instead of estimating the exact value of $A(s,a)$, CAPO is more forgiving as it only requires the sign of advantage. Later in section \ref{subsection:analysis}, we will see that by scaling the update direction with the properly designed step size that only depends on the internal parameters $\theta$, instead of the estimation of the environment model, CAPO can achieve global convergence.  

%\begin{remark}
%\normalfont Note that CAPO can be integrated with any off-policy policy evaluation algorithm, such as Retrace \citep{munos2016safe} or V-trace \citep{espeholt2018impala}.
%\end{remark}

%% file: method_sub/analysis.tex
% \subsection{Theoretical Analysis for the Tabular Case}
% \label{subsection:analysis}

We proceed to substantiate the benefit of the state-action-dependent learning rate used in the general CAPO update in (\ref{eq:CAPO_form}) by showing that CAPO can attain a globally optimal policy with a properly designed learning rate $\alpha(\cdot,\cdot)$.
%Recall that we consider tabular softmax parameterized policy $\pi_\theta$, and
%$\theta_m$ denotes the policy parameter of the $m$-th iteration.
%\ref{eq:CAPO_form}. 

%any $f$ satisfying $f(A(s, a)) * A(s, a) > 0$ guarantees strict improvement under the tabluar softmax setting. \eqref{eq:CAPO_update} is the general CAPO form with $f(x) = \sign(x)$, we now show that any $f$ satisfying $f(A(s, a)) * A(s, a) > 0$ guarantees strict improvement under the tabluar softmax setting. 

\begin{theorem}
\label{theorem:convergeoptimal}
Consider a tabular softmax parameterized policy $\pi_\theta$. Under (\ref{eq:CAPO_form}) with ${\alpha_m(s, a) \ge \log (\frac{1}{\pi_{\theta_m(a\rvert s)}})}$, if Condition \ref{condition:sa} is satisfied, then we have $V^{\pi_{m}}(s) \rightarrow V^{*}(s)$ as $m\rightarrow \infty$, for all $s\in\cS$.
%$\forall s \in \mathcal{S}, \theta' \in \mathcal{S \times A},\lim_{m \rightarrow \infty} V^{\pi_{m}}(s) \geq V^{\pi_{\theta^{\prime}}}(s)$  
\end{theorem}

%\begin{equation}
%\label{eq:CAPO_update_alpha}
%\theta_{m+1}(s, a) = \theta_{m}(s, a) + \log \left( \frac{1}{\pi\left(a \rvert s\right)} \right) \sign\left(A^{m}\left(s, a\right)\right)
%\end{equation}

\begin{proof}[Proof Sketch]
The detailed proof can be found in Appendix \ref{app:proof_convergeoptimal}. 
To highlight the main ideas of the analysis, we provide a sketch of the proof as follows: (i) Since the expected total reward is bounded above, with the strict policy improvement property of CAPO update (cf. Lemma \ref{lemma:strict_improvement}), the sequence of value functions is guaranteed to converge, i.e., the limit of $V^{\pi_m}(s)$ exists. (ii) The proof proceeds by contradiction. We suppose that CAPO converges to a sub-optimal policy, which implies that there exists at least one state-action pair $(s',a')$ such that $A^{(\infty)}(s',a')>0$ and $\pi_{\infty}(a'' \rvert s'') = 0$ for all state-action pair $(s'',a'')$ satisfying $A^{(\infty)}(s'',a'')>0$. 
As a result, this implies that for any $\epsilon>0$, there must exist a time $M^{\eps}$ such that $\pi_m(a'' \rvert s'') < \eps$, $\forall m > M^{\eps}_{}$. (iii) However, under CAPO update, we show that the policy weight of the state-action pair which has the greatest advantage value shall approach $1$, and this leads to a contradiction.
\end{proof}

\begin{remark}
\normalfont \pch{The proof of Theorem \ref{theorem:convergeoptimal} is inspired by \citep{agarwal2020theory}. Nevertheless, the analysis of CAPO presents its own salient challenge: Under true PG, the policy updates in all the iterations can be fully determined once the initial policy and the step size are specified. By contrast, under CAPO, the policy obtained in each iteration depends on the selected coordinates, which can be almost arbitrary under Condition \ref{condition:sa}. This makes it challenging to establish a contradiction under CAPO, compared to the argument of directly deriving the policy parameters in the limit in true PG \citep{agarwal2020theory}. Despite this, we address the challenge by using a novel induction argument based on the action ordering w.r.t. the $Q$ values in the limit.}
\end{remark}

\begin{remark}
\normalfont Notably, the condition of the learning rate $\alpha$ in Theorem \ref{theorem:convergeoptimal} does not depend on the advantage, but only on the action probability $\pi_{\theta}(a\rvert s)$. As a result, the CAPO update only requires the sign of the advantage function, without the knowledge of the magnitude of the advantage. Therefore, CAPO can still converge even under a low-fidelity critic that merely learns the sign of the advantage function.
%be less sensitive to the erroneous estimation of advantage, which is almost inevitable in practice due to the randomness in rewards and state transitions as well as the use of neural function approximation.
\end{remark}
%as CAPO only takes the sign of $A$ into account, which is invariant to small permutation in the advantage function.

% We will show in this section that $\alpha$ does not depend on $A$, but the internal parameters $\theta$. Since $f$ only takes the sign of $A$ into account, CAPO update is invariant to small permutation in the advantage function, as long as the sign is correct.

%\begin{remark}
%\normalfont It is somewhat surprising that simple coordinate ascent is already capable of achieving global convergence. As shown in the convergence analysis, the properly-designed step size in (\ref{eq:CAPO_form}) turns out to be crucial for the convergence result. On the flip side, we will show that CAPO with a conventional design of fixed step sizes could easily get stuck at a sub-optimal stationary point in \Cref{app:OnCAPO:onCAPO_fixed}. 
%\end{remark}

%% file: method_sub/convergence_rate.tex
\subsection{Convergence Rates of CAPO With Specific Coordinate Selection Rules}
\label{subsection:convergence_rate}

In this section, we proceed to characterize the convergence rates of CAPO under softmax parameterization and the three specific coordinate generators, namely, Cyclic, Batch, and Randomized CAPO.

\begin{itemize}[leftmargin=*]
    \item \textbf{Cyclic CAPO}: Under Cyclic CAPO, every state action pair $(s,a) \in \mathcal{S} \times \mathcal{A}$ will be chosen for policy update by the coordinate generator cyclically. Specifically, Cyclic CAPO sets $\lvert B_m\rvert=1$ and $\bigcup_{i=1}^{|\mathcal{S}||\mathcal{A}|} B_{m \cdot |\mathcal{S}||\mathcal{A}| + i} =  \mathcal{S} \times \mathcal{A}$.
    \item \textbf{Randomized CAPO}: {Under Randomized CAPO, in each iteration, one state-action pair $(s,a) \in \mathcal{S} \times \mathcal{A}$ is chosen randomly from some coordinate generator distribution $d_{\text{gen}}$ with support $\cS\times \cA$ for policy update, where $d_{\text{gen}}(s,a)>0$ for all $(s,a)$. For ease of exposition, we focus on the case of a fixed $d_{\text{gen}}$. Our convergence analysis can be readily extended to the case of time-varying $d_{\text{gen}}$.}
    %Specifically, Cyclic CAPO sets $\lvert B_m\rvert=1$ and $\bigcup_{i=1}^{|\mathcal{S}||\mathcal{A}|} B_{m \cdot |\mathcal{S}||\mathcal{A}| + i} =  \mathcal{S} \times \mathcal{A}$.}
    \item \textbf{Batch CAPO}: Under Batch CAPO, we let each batch contain all of the state-action pairs, i.e., $B_m = \left \{ (s,a) : (s,a) \in \mathcal{S} \times \mathcal{A} \right \}$, in each iteration. Despite that Batch CAPO may not be a very practical choice, we use this variant to further highlight the difference in convergence rate between CAPO and the true PG.
\end{itemize}

We proceed to state the convergence rates of the above three instances of CAPO as follows.
\begin{theorem}[Cyclic CAPO]
\label{theorem:cyclic_convergence_rate}
Consider a tabular softmax policy $\pi_\theta$. Under Cyclic CAPO with ${\alpha_m(s, a) \ge \log (\frac{1}{\pi_{\theta_m(a\rvert s)}})}$ and $|B_m| = 1$, $\bigcup_{i=1}^{|\mathcal{S}||\mathcal{A}|} B_{m \cdot |\mathcal{S}||\mathcal{A}| + i} =  \mathcal{S} \times \mathcal{A}$, we have:
\begin{align}
V^{*}(\rho) - V^{\pi_m}(\rho) \le \frac{|\mathcal{S}||\mathcal{A}|}{c} \cdot \frac{1}{m} , \quad\text{for all $m \ge 1$}
\end{align}
%\begin{align}
%&\sum_{m=1}^{M}  V^{*}(\rho) - V^{\pi_m}(\rho)\nonumber\\
%&\le |\mathcal{S}||\mathcal{A}| \cdot \min{ \left \{ \sqrt{\frac{M}{c (1-\gamma)}}, \frac{\log M +1}{c} \right \}  },\hspace{2pt}\text{for all $m \ge 1$}
%\end{align}
where $c = \frac{(1-\gamma)^4}{2} \cdot \left \| \frac{1}{\mu} \right \|_{\infty}^{-1} \cdot {\min} \left \{ \frac{\min_s{\mu(s)} }{2} , \frac{ (1-\gamma)}{|\mathcal{S}||\mathcal{A}|} \right \} > 0$.
\end{theorem}
\begin{proof}[Proof Sketch]
The detailed proof and the upper bound of the partial sum can be found in Appendix \ref{app:convergence_rate}. 
To highlight the main ideas of the analysis, we provide a sketch of the proof as follows: (i) We first write the one-step improvement of the performance $V^{\pi_{m+1}}(s) - V^{\pi_m}(s)$ in state visitation distribution, policy weight, and advantage value, and also construct the lower bound of it. (ii) We then construct the upper bound of the performance difference $ V^{*}(s) - V^{\pi_m}(s)$. (iii) Since the bound in (i) and (ii) both include advantage value, we can connect them and construct the upper bound of the performance difference using one-step improvement of the performance. (iv) Finally, we can get the desired convergence rate by induction.
\end{proof}
%\begin{proof}
%The proof can be found in Appendix \ref{app:cyclic_convergence_rate}.
%\end{proof}

Notably, it is somewhat surprising that Theorem \ref{theorem:cyclic_convergence_rate} holds under Cyclic CAPO \textit{without} any further requirement on the specific cyclic ordering. This indicates that Cyclic CAPO is rather flexible in the sense that it provably attains $\cO(\frac{1}{m})$ convergence rate under any cyclic ordering or even cyclic orderings that vary across cycles. 
{On the flip side, such a flexible coordinate selection rule also imposes significant challenges on the analysis: (i) While Lemma \ref{lemma:strict_improvement} ensures strict improvement in each iteration, it remains unclear how much improvement each Cyclic CAPO update can actually achieve, especially under an \textit{arbitrary cyclic ordering}. This is one salient difference compared to the analysis of the true PG \citep{mei2020global}. (ii) Moreover, an update along one coordinate can already significantly change the advantage value (and its sign as well) of other state-action pairs. Therefore, it appears possible that there might exist a well-crafted cyclic ordering that leads to only minimal improvement in each coordinate update within a cycle.}

{Despite the above, we tackle the challenges by arguing that in each cycle, under a properly-designed variable step size $\alpha$, there must exist at least one state-action pair such that the one-step improvement is sufficiently large, regardless of the cyclic ordering. Moreover, by the same proof technique, Theorem \ref{theorem:cyclic_convergence_rate} can be readily extended to CAPO with almost-cyclic coordinate selection, where the cycle length is greater than $\lvert S\rvert\lvert A\rvert$ and each coordinate appears at least once.}

We extend the proof technique of Theorem \ref{theorem:cyclic_convergence_rate} to establish the convergence rates of the Batch and Randomized CAPO.
%\begin{remark}
%\normalfont We finished the proof of Theorem \ref{theorem:cyclic_convergence_rate} by arguing that in each cycle, there must exist at least one state-action pair such that the amount of the one-step improvement will be guaranteed once we update it regardless of the updating order.
%\end{remark}

%\begin{remark}
%\normalfont Compared to the standard policy gradient methods, one surprising property of Cyclic CAPO is that we can still establish the convergence rate regardless of the updating order.
%\end{remark}

\begin{theorem}[Batch CAPO]
\label{theorem:batch_convergence_rate}
Consider a tabular softmax policy $\pi_\theta$. Under Batch CAPO with ${\alpha_m(s,a) = \log (\frac{1}{\pi_{\theta_m(a\rvert s)}})}$ and $B_m = \left \{ (s,a) : (s,a) \in \mathcal{S} \times \mathcal{A} \right \}$, we have :
\begin{equation}
V^{*}(\rho) - V^{\pi_m}(\rho) \le \frac{1}{c} \cdot \frac{1}{m}, \quad \text{for all $m \ge 1$}
\end{equation}
%\begin{align}
%&\sum_{m=1}^{M}  V^{*}(\rho) - %V^{\pi_m}(\rho)\nonumber\\
%&\le \min{ \left \{ \sqrt{\frac{M}{c \cdot (1-\gamma)}}, \frac{\log M +1}{c} \right \}  }, \quad \text{for all $m \ge 1$}
%\end{align}
where $c = \frac{(1-\gamma )^4}{|\mathcal{A}|} \cdot \left \| \frac{1}{\mu} \right \|_{\infty}^{-1} \cdot \underset{s}{\min} \left \{ \mu(s) \right \} > 0$.
\end{theorem}
\begin{proof}
The proof and the upper bound of the partial sum can be found in Appendix \ref{app:batch_convergence_rate}.
\end{proof}

%\pch{Note that in Batch CAPO, the learning rate $\alpha(s,a)$ is set to be exactly $\log (\frac{1}{\pi_\theta(a\rvert s)})$ since the extreme case of learning rate might has negative influences on convergence rate.}

%Lastly, by removing all of the restriction of the generator, we also provide the convergence rate in expectation. 

\begin{theorem}[Randomized CAPO]
\label{theorem:randomized_convergence_rate}
Consider a tabular softmax policy $\pi_\theta$. Under Randomized CAPO with ${\alpha_m(s, a) \ge \log (\frac{1}{\pi_{\theta_m(a\rvert s)}})}$, we have :
\begin{equation}
\underset{({s_m},{a_m}) \sim d_{\text{gen}}}{\mathbb{E}} \left [ V^{*}(\rho) - V^{\pi_m}(\rho) \right] \le \frac{1}{c} \cdot \frac{1}{m}, \quad \text{for all $m \ge 1$}
\end{equation}
%\begin{align}
%&\sum_{m=1}^{M} \underset{({s_m},{a_m}) \sim d_{\text{gen}}}{\mathbb{E}} \left [ V^{*}(\rho) - V^{\pi_m}(\rho) \right]\nonumber \\
%&\le \min{ \left \{ \sqrt{\frac{M}{c \cdot (1-\gamma)}}, \frac{\log M +1}{c} \right \}  }, \quad \text{for all $m \ge 1$}
%\end{align}
where $c = \frac{(1-\gamma )^4}{2} \cdot \left \| \frac{1}{\mu} \right \|_{\infty}^{-1} \cdot \underset{(s,a)}{\min} \left \{ d_{\text{gen}}(s,a)\cdot\mu(s) \right \} > 0$ and $d_{\text{gen}}:\mathcal{S} \times \mathcal{A} \rightarrow (0,1)$, $d_{\text{gen}}(s,a) = \mathbb{P}((s,a) \in B_m)$.
\end{theorem}
\begin{proof}
The proof and the upper bound of the partial sum can be found in Appendix \ref{app:randomized_convergence_rate}.
\end{proof}

\begin{remark}
\normalfont The above three specific instances of CAPO all converge to a globally optimal policy at a rate $\cO(\frac{1}{m})$ and attains a better pre-constant than the standard policy gradient \citep{mei2020global} under tabular softmax parameterization. Moreover, as the CAPO update can be combined with a variety of coordinate selection rules, one interesting future direction is to design  coordinate generators that improve over the convergence rates of the above three instances.
\end{remark}

\begin{table*}[!ht]
\caption{A summary of convergence rates under tabular softmax parameterization under different algorithms.} \label{tab:constant_conv_rate}
\begin{center}
\begin{tabular}{ll}
\textbf{Algorithm}  &\textbf{Convergence Rate} \\
\hline \\
Policy Gradient \citep{mei2020global}         &$V^{*}(\rho) - V^{\pi_m}(\rho) \le \frac{16 \cdot |\mathcal{S}|}{\inf_{m \ge 1} \pi_m(a^*|s)^2\cdot (1-\gamma)^6} \cdot \left \| \frac{d^{\pi^*}_{\mu}}{\mu} \right \|_\infty^2  \cdot\left \| \frac{1}{\mu}  \right \|_{\infty } \cdot \frac{1}{m} $ \\
%Natural Policy Gradient \citep{agarwal2020theory}             &$V^{*}(\rho) - V^{\pi_m}(\rho) \le \frac{2}{(1-\gamma)^2} \cdot \frac{1}{m}$ \\
Cyclic CAPO (Theorem \ref{theorem:cyclic_convergence_rate})             &$ V^{*}(\rho) - V^{\pi_m}(\rho) \le \frac{2 \cdot |\mathcal{S}||\mathcal{A}|}{(1-\gamma)^4} \cdot \left \| \frac{1}{\mu} \right \|_{\infty} \cdot {\max} \left \{ \frac{2}{\min_s{\mu(s)} } , \frac{|\mathcal{S}||\mathcal{A}|}{(1-\gamma)} \right \} \cdot \frac{1}{m}$ \\
Batch CAPO (Theorem \ref{theorem:batch_convergence_rate})             &$ V^{*}(\rho) - V^{\pi_m}(\rho) \le \frac{|\mathcal{A}|}{(1-\gamma )^4} \cdot \left \| \frac{1}{\mu} \right \|_{\infty} \cdot \frac{1}{\underset{s}{\min} \left \{ \mu(s) \right \}} \cdot \frac{1}{m}$ \\
Randomized CAPO (Theorem \ref{theorem:randomized_convergence_rate})            &$ \underset{({s_m},{a_m}) \sim d_{\text{gen}}}{\mathbb{E}} \left [ V^{*}(\rho) - V^{\pi_m}(\rho) \right] \le \frac{2}{(1-\gamma )^4} \cdot \left \| \frac{1}{\mu} \right \|_{\infty} \cdot \frac{1}{\underset{(s,a)}{\min} \left \{ d_{\text{gen}}(s,a)\cdot\mu(s) \right \}} \cdot \frac{1}{m}$\\
\end{tabular}
\end{center}
\end{table*}

%% file: method_sub/onCAPO.tex
\section{Discussions}
\label{section:method:onCAPO}
% While we originally design CAPO to be an off-policy algorithm to avoid the vicious cycle of on-policy update, it is somewhat surprising that with the ability to escape from the local optimum endowed with variable learning rate, we can show that on-policy CAPO also guarantees global convergence.
In this section, we describe the connection between CAPO and the existing policy optimization methods and present additional useful features of CAPO.
%While CAPO is designed to be an flexible algorithm that utilize the power of off-policy sampling, in this section we take a look at on-policy CAPO as it share some similarities with stochastic NPG (SNPG), which unlike On-CAPO, does not guarantee global convergence even when an unbiased estimator is used \cite{mei2020global}.

\input{method_sub/connecting_npg}

%% file: method_sub/connecting_npg.tex
\textbf{Batch CAPO and True PG.}
We use Batch CAPO to highlight the fundamental difference in convergence rate between CAPO and the true PG as they both take all the state-action pairs into account in one policy update. Compared to the rate of true PG \citep{mei2020global}, Batch CAPO removes the dependency on the size of state space $\lvert \cS\rvert$ and $\inf_{m \ge 1} \pi_m(a^*|s)^2$.
In true PG, these two terms arise in the construction of the Łojasiewicz inequality, which quantifies the amount of policy improvement with the help of an optimal policy $\pi^*$, which explains why $\inf_{m \ge 1} \pi_m(a^*|s)^2$ appears in the convergence rate. By contrast, in Batch CAPO, we quantify the amount of policy improvement based on the coordinate with the largest advantage, and this proof technique contributes to the improved rate of Batch CAPO compared to true PG. Moreover, we emphasize that this technique is feasible in Batch CAPO but not in true PG mainly due to the properly-designed learning rate of CAPO. 

\textbf{Connecting CAPO With Natural Policy Gradient.}
%\label{subsection:connecting_npg}
%\subsection{Natural Policy Gradient}
The natural policy gradient (NPG) \citep{kakade2001natural} exploits the landscape of the parameter space and updates the policy by:
\begin{equation}
\label{eq:natural_pg_update}
\theta_{m+1} = \theta_{m}+\eta\left(F_{\rho}^{\theta_{m}}\right)^{\dagger} \nabla_{\theta} J^{\pi_{\theta}}(\rho),
\end{equation}
where $\eta$ is the step size and $\left(F_{\rho}^{\theta_{m}}\right)^{\dagger}$ is the Moore-Penrose pseudo inverse of the Fisher information matrix
$F_{\rho}^{\theta_{m}}:={\mathbb{E}}_{s \sim d_{\rho}^{\pi_{\theta_{m}}}, a \sim \pi_{\theta_{m}}(\cdot \mid s)}[(\nabla_{\theta} \log \pi_{\theta_{m}}(a \rvert s))(\nabla_{\theta} \log \pi_{\theta_{m}}(a \rvert s))^{\top}]$.
Moreover, under softmax parameterization, the true NPG update takes the following form \citep{agarwal2020theory}:
\begin{equation}
\label{eq:npg_softmax_update}
\theta_{m+1}=\theta_{m}+\frac{\eta}{1-\gamma} A^{\pi_{\theta_m}},
\end{equation}
where $A^{\pi_{\theta_m}}$ denotes the $\lvert S\rvert \lvert A\rvert$-dimensional vector of all the advantage values of ${\pi_{\theta_m}}$.
It has been shown that the true NPG can attain linear convergence \citep{mei2021understanding,khodadadian2021linear}.
Given the expression in (\ref{eq:npg_softmax_update}), CAPO can be interpreted as adapting NPG to the mini-batch or stochastic settings. That said, compared to true NPG, CAPO only requires the sign of the advantage function, not the magnitude of the advantage.
On the other hand, it has recently been shown that some variants of on-policy stochastic NPG could exhibit committal behavior and thereby suffer from convergence to sub-optimal policies \citep{mei2021understanding}.
The analysis of CAPO could also provide useful insights into the design of stochastic NPG methods.
Interestingly, in the context of variational inference, a theoretical connection between coordinate ascent and the natural gradient has also been recently discovered \citep{ji2021marginalized}.

\textbf{CAPO for Low-Fidelity RL Tasks.}
One salient feature of CAPO is that it requires only the sign of the advantage function, instead of the exact advantage value. It has been shown that accurate estimation of the advantage value could be rather challenging under benchmark RL algorithms \citep{ilyas2019closer}. As a result, CAPO could serve as a promising candidate solution for RL tasks with low-fidelity or multi-fidelity value estimation \citep{cutler2014reinforcement,kandasamy2016multi,khairy2022multifidelity}. 

\textbf{CAPO for On-Policy Learning.}
The original motivation of CAPO is to achieve off-policy policy updates without the issues of distribution mismatch and fixed behavior policy.
Despite this, the CAPO scheme in (\ref{eq:CAPO_form}) can also be used in an \textit{on-policy} manner.
Notably, the design of on-policy CAPO is subject to a similar challenge of committal behavior in on-policy stochastic PG and stochastic NPG \citep{chung2021beyond,mei2021understanding}.
Specifically: (i) We show that on-policy CAPO with a fixed step size could converge to sub-optimal policies through a multi-armed bandit example similar to that in \citep{chung2021beyond}. (ii) We design a proper step size for on-policy CAPO and establish asymptotic global convergence. Through a simple bandit experiment, we show that this variant of on-policy CAPO can avoid the committal behavior. Due to space limitation, all the above results are provided in Appendix \ref{app:OnCAPO}.

%\textbf{\pch{CAPO With a Constant Step Size.}}

%% file: method_sub/neural.tex
\section{Practical Implementation of CAPO}
%As the number of state and action increase, eventually it will become unrealistic to maintain a tabular policy. 
To address the large state and action spaces of the practical RL problems, we proceed to parameterize the policy for CAPO by a neural network and make use of its powerful representation ability. 
As presented in \Cref{section:method:onCAPO}, the coordinate update and variable learning rate are two salient features of CAPO. 
These features are difficult to preserve if the policy is trained in a completely end-to-end manner. 
Instead, we take a two-step approach by first leveraging the tabular CAPO to derive target action distributions and then design a loss function that moves the output of the neural network towards the target distribution. 
%compute the policy by generalizing $\theta$ with a neural network $f_\theta: \mathcal{S \times A} \rightarrow \mathbb{R}$ such that:
%\begin{equation}
%\pi_{\theta}(a | s) = \frac{f_\theta(s, a)}{\sum_{a' \in \mathcal{A}} f_\theta(s, a')}
%\end{equation}
Specifically, we designed a neural version of CAPO, called Neural Coordinate Ascent Policy Optimization (NCAPO): Let $f_\theta(s,a)$ denote the output of the policy network parameterized by $\theta$, for each $(s,a)$. In NCAPO, we use neural softmax policies, i.e., ${\pi}_{\theta}(a\rvert s) = 
\frac{\exp({f_{\theta}\left(s, a\right)})}{\sum_{a^{\prime} \in \mathcal{A}} \exp({f_{\theta}\left(s, a^{\prime}\right)})}$.
\begin{itemize}[leftmargin=*]
    \item Inspired by the tabular CAPO, we compute a target softmax policy $\pi_{\hat{\theta}}(s, a)$ by following the CAPO update (\ref{eq:CAPO_form})
\begin{equation}
\label{eq:NCAPO_update}
\tilde{\theta}(s, a) = f_{\theta}(s,a) +\alpha(s, a) \mathbb{I}\{(s,a) \in {B}\}\cdot \sign\left(A^{\pi_{\theta}}\left(s, a\right)\right).
\end{equation}
The target action distribution is then computed w.r.t. $\tilde{\theta}$ as
$\tilde{\pi}(a\rvert s) = 
\frac{\exp({\tilde{\theta}\left(s, a\right)})}{\sum_{a^{\prime} \in \mathcal{A}} \exp({\tilde{\theta}\left(s, a^{\prime}\right)})}$.
\item Finally, we learn $f_\theta$ by minimizing the NCAPO loss, which is the KL-divergence loss between the current policy and the target policy:
\begin{equation}
\label{eq:NCAPO_loss}
\mathcal{L}(\theta)=\sum_{s \in B} D_{\text{KL}}\left(\pi_{\theta}(\cdot \rvert s) \| \tilde{\pi}(\cdot \rvert s)\right).
\end{equation}
\end{itemize}

%% file: 6-experiments.tex
\section{Experimental Results}
\label{section:exp}
In this section, we empirically evaluate the performance of CAPO on several benchmark RL tasks.
%\subsection{MinAtar}
%To demonstrate the effectiveness of NCAPO, we implement NCAPO with replay buffer and compare the result with various baselines.
We evaluate NCAPO in MinAtar \citep{young19minatar}, a simplified Arcade Learning Environment (ALE), and consider a variety of environments, including \textit{Seaquest, Breakout, Asterix, and Space Invaders}. Each environment is associated with $10 \times 10 \times n$ binary state representation, which corresponds to the $10 \times 10$ grid and $n$ channels (the value $n$ depends on the game).

\textbf{Benchmark Methods.}
We select several benchmark methods for comparison, including Rainbow \citep{hessel2018rainbow,obando2020revisiting}, PPO \citep{schulman2017ppo}, Off-PAC \citep{degris2012offpac}, and Advantage Actor-Critic (A2C) \citep{mnih2016asynchronous}, to demonstrate the effectiveness of NCAPO. 
For Rainbow, we use the code provided by \citep{obando2020revisiting} without any change. 
For the other methods, we use the open-source implementation provided by Stable Baselines3 \citep{stable-baselines3}.

\begin{figure*}[ht]
\centering
  \includegraphics[width=0.25\textwidth]{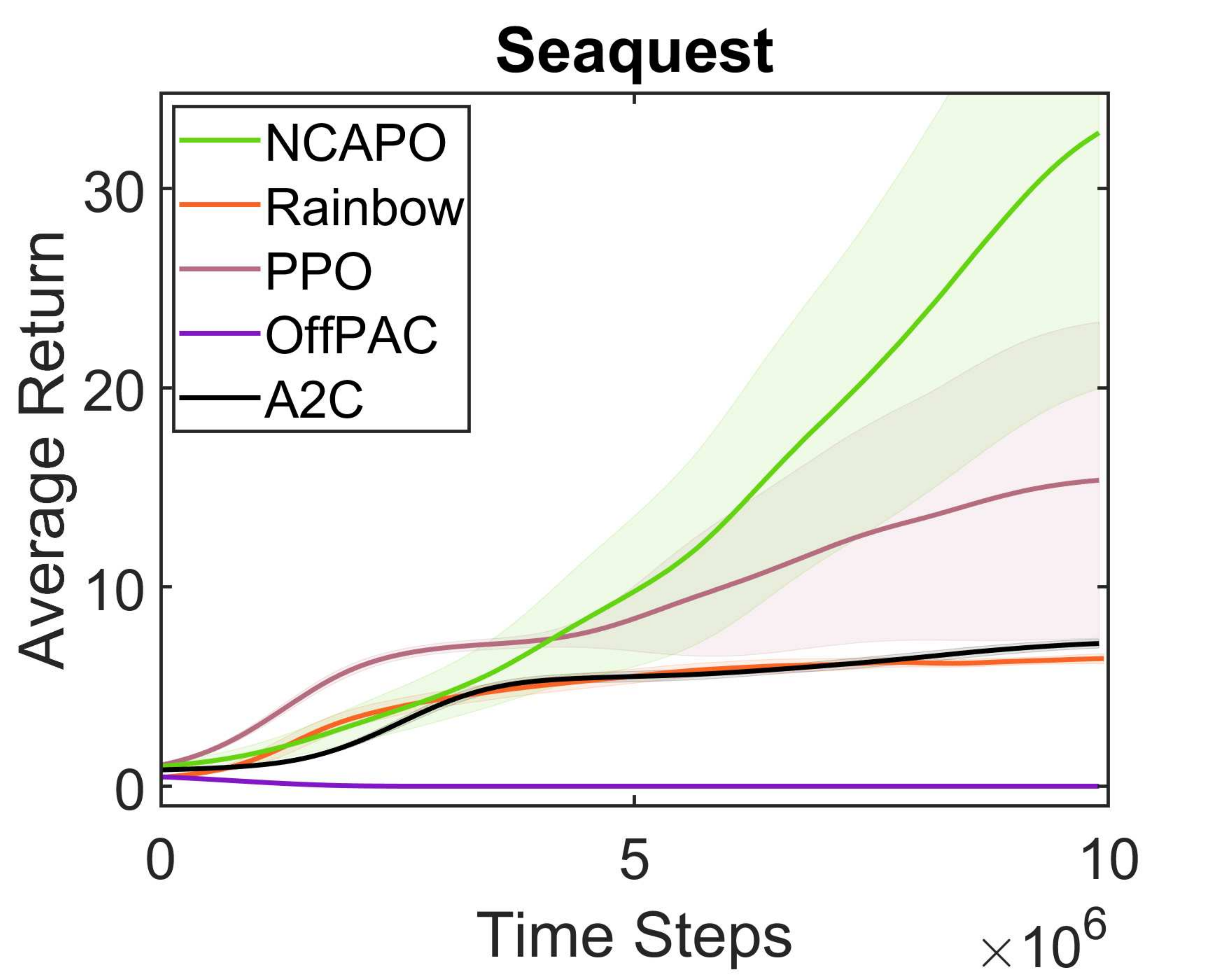}
  \hspace{-5mm}
  \includegraphics[width=0.245\textwidth]{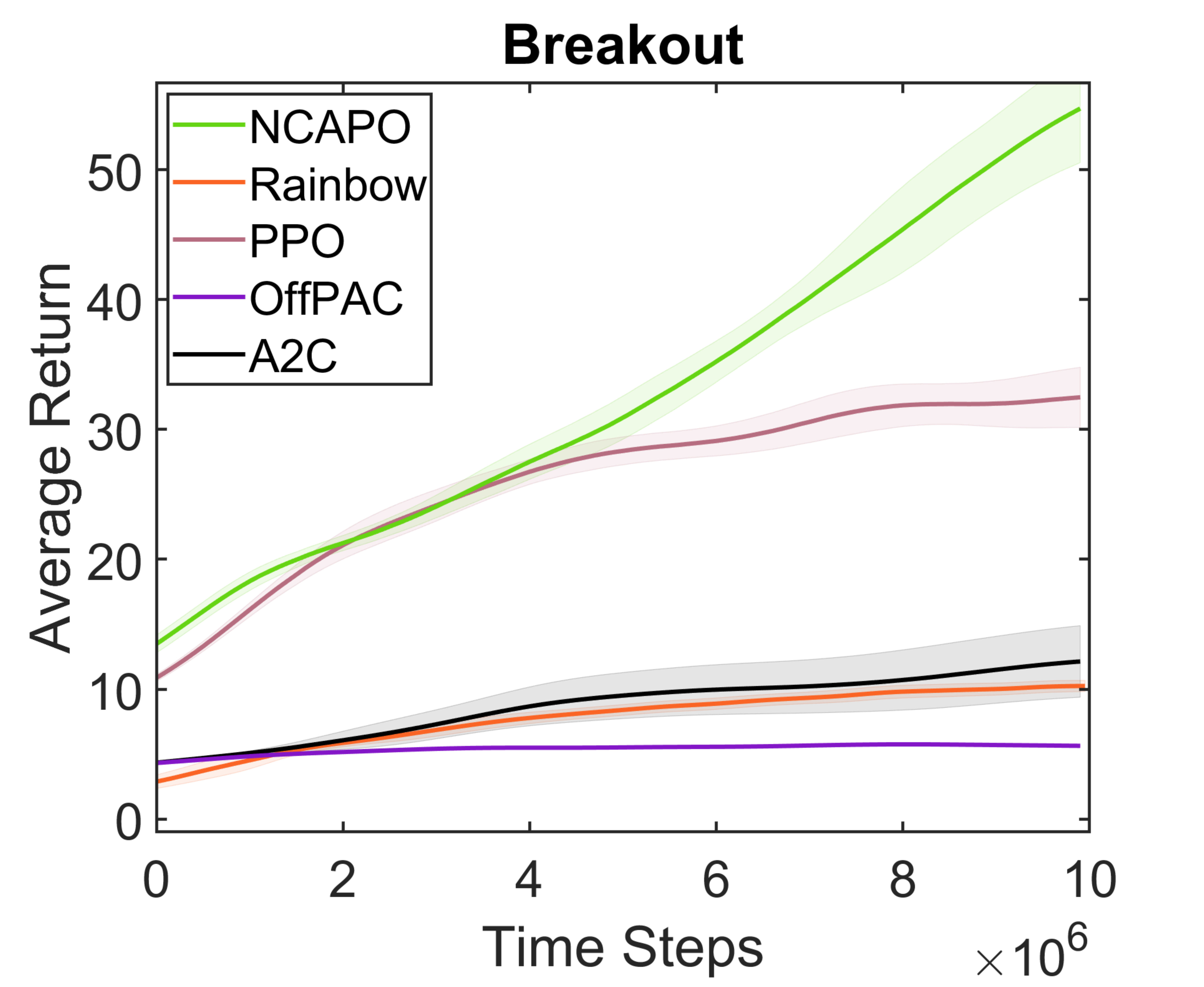}
  \hspace{-5mm}
  \includegraphics[width=0.25\textwidth]{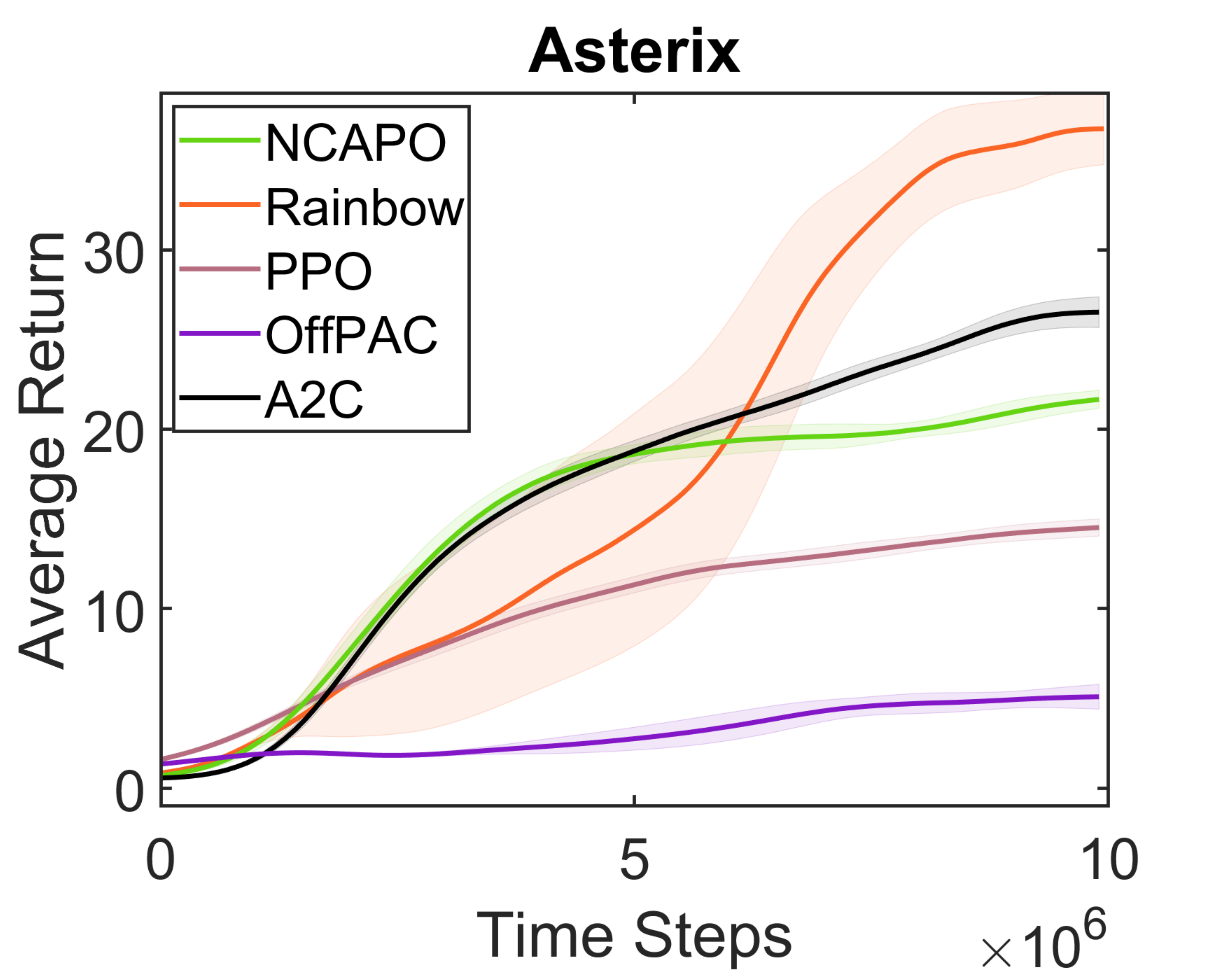}
  \hspace{-5mm}
  \includegraphics[width=0.25\textwidth]{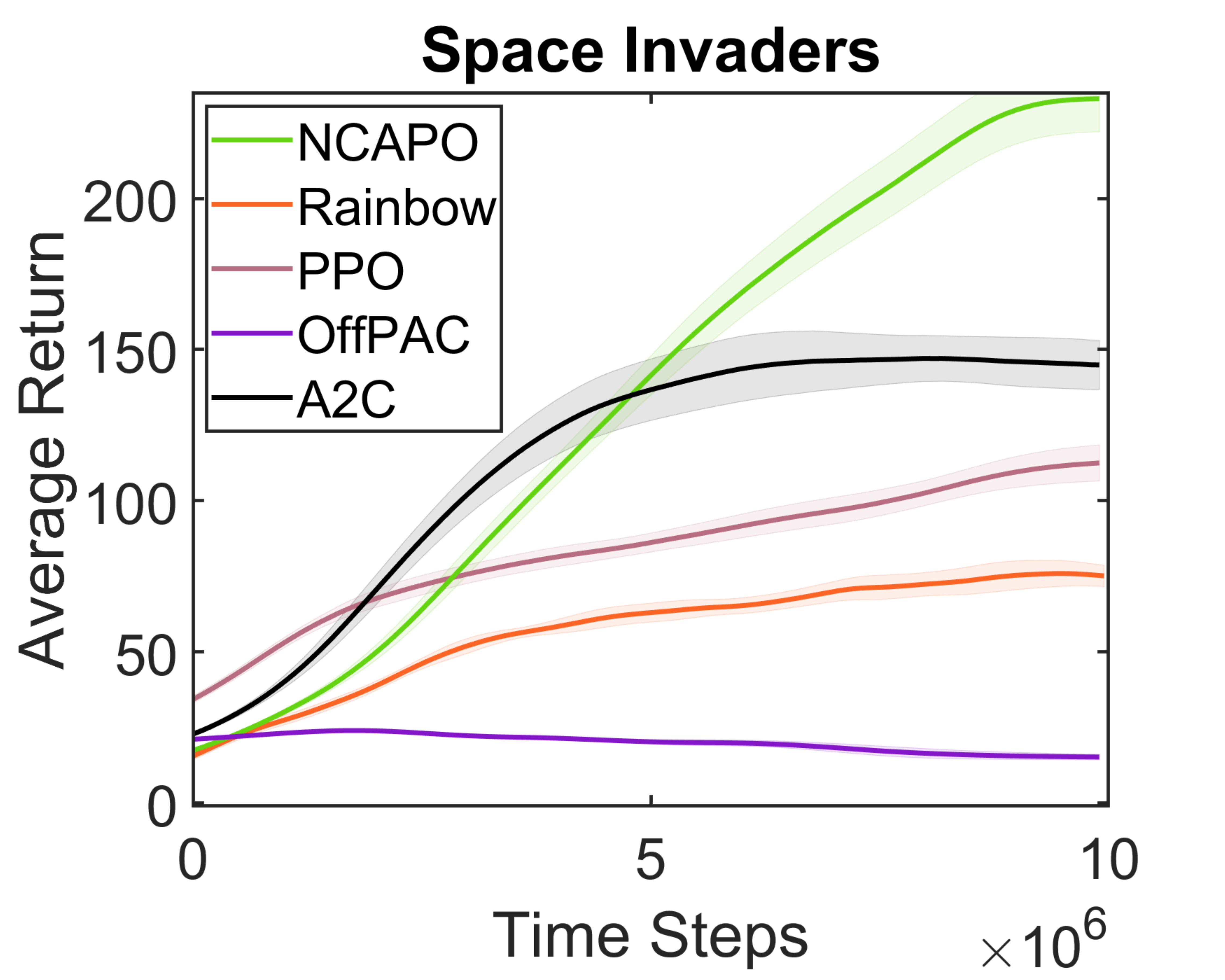}
  \caption{A comparison between the performance of NCAPO and other benchmark methods algorithms in MinAtar. All the results are averaged over 10 random seeds (with the shaded area showing the range of $\text{mean} \pm 0.5\cdot\text{std}$).}
  \label{fig:minatar}
\end{figure*}

\textbf{Empirical Evaluation.}
The detailed implementation of NCAPO is provided in Appendix \ref{app:exp}.
From \Cref{fig:minatar}, we can observe that NCAPO has the best performance in \textit{Seaquest, Breakout, Space Invaders}. 
%Despite its slightly slower convergence in \textit{Freeway}, NCAPO still ultimately converges to a relatively good policy.
We also see that NCAPO is more robust across tasks than PPO and Rainbow. For example, Rainbow performs especially well in \textit{Asterix}, while relatively poorly in \textit{Space Invaders}. 
PPO performs relatively strong in \textit{Breakout} and \textit{Seaquest}, but converges rather slowly in \textit{Asterix}.
Off-PAC with a uniform behavior policy has very little improvement throughout training in all the tasks due to the issue of fixed behavior policy, which could hardly find sufficiently long trajectories with high scores.
By contrast, NCAPO outperforms all the benchmark methods in three out of four environments, while being on par with other methods in the remaining environment.
\begin{comment}
\begin{remark}
\normalfont Both PG-based methods and DQN-based methods have been extensively studied and experimented. For example, Rainbow has evolved from the vanilla DQN through multiple enhancements, including double Q-learning, prioritized experience replay, dueling networks, multi-step learning, distributional RL and noisy networks. On the other hand, as the first attempt of coordinate ascent method for RL, NCAPO could surely benefit from future enhancements. The competitive empirical performance of NCAPO further manifests its potential .
\end{remark}
\end{comment}

% \subsection{Grid World}
% The grid world environment consist of a $H*W$ grid, where $H$ and $W$ is the height and width of the grid. The agent start at position $(1, 1)$, which is the upper-left of the grid, the goal of the agent is to reach the bottom-right of the grid at position $(H, W)$. However, there is a wall between the agent and the goal, either horizontal or vertical, with only one opening. The agent can perform 4 actions, by moving either Top, Left, Right or Down. It will remain at the starting position if it tries to move into the wall, and ends when the agent reach the goal.

% The reward for reaching the goal is $100$, $-1$ otherwise.

% We compare on-policy actor-critic with different batch sizes, agent is trained with true value function.
% \tmp{figure of grid world here}

% \subsubsection{Update Direction}
% We compare the empirical policy gradient direction and the true gradient direction, as shown in \tmp{Figure X}.

%% file: 7-related.tex
\section{Related Work}
\label{section:related}

%\textbf{Importance sampling.}
\begin{comment}
\textbf{Stochastic Policy Gradient}
While it has been shown that exact policy gradient converges to the global optimal policy \citep{agarwal2020theory, bhandari2019global}, it was also shown that not all stochastic settings of policy gradient will converge to the global optimum. \citep{chung2021beyond} showed that the choice of baseline not only affect the variance of SGD, it can directly impact the convergence of the algorithm due to the early commitment to a sub-optimal action. They showed that under three-arm bandit, natural policy gradient can stuck in local optimum. \citet{mei2021understanding} extends the result from three-arm bandit to a more general K-arm bandit case under the NPG setting. 
%Moreover, they also showed that with importance sampling, on-policy PG can converge to the global optimum. %However, we have shown in this work that the more widely used unbiased estimator for vanilla on-policy gradient will still stuck in local optimum due to early commitment. 
Despite the difficulties PG-based methods are facing, CAPO is invariant to the stochasticity of sampling and guarantees global convergence.
On the other hand, to accelerate the unlearning process of PG, \cite{laroche2022beyond} proposes a cross-entropy policy update scheme.
\end{comment}

\textbf{Off-Policy Policy Gradients}.
Off-policy learning via PG has been an on-going research topic. 
Built on the off-policy PG theorem \citep{degris2012offpac, silver2014deterministic,zhang2019general, imani2018off}, various off-policy actor-critic algorithms have been developed with an aim to achieve more sample-efficient RL \citep{wang2016sample,gu2017q,chung2021beyond,ciosek2018expected,espeholt2018impala,schmitt2020off}.
In the standard off-policy PG formulation, the main idea lies in the use of a {surrogate objective}, which is the expected total return with expectation taken over the stationary distribution induced by the behavior policy. 
While this design avoids the issue of an exponentially-growing importance sampling ratio, it has been shown that this surrogate objective can suffer from convergence to sub-optimal policies due to distribution mismatch, and distribution correction is therefore needed, either via a learned density correction ratio \citep{liu2020off} or emphatic weighting \citep{maei2018convergent,zhang2019general,zhang2020provably}.
On the other hand, off-policy actor-critic based on NPG has been recently shown to achieve provable sample complexity guarantees in both tabular \citep{khodadadian2021finite} and linear function approximation setting
\citep{chen2022sample,chen2022finite}.
Another line of research is on characterizing the convergence of off-policy actor-critic methods in the \textit{offline} setting, where the learner is given only a fixed dataset of samples \citep{xu2021doubly,huang2022convergence}.
Some recent attempts propose to enable off-policy learning beyond the use of policy gradient.
For example, \citep{romain2021JnH} extends the on-policy PG to an off-policy policy update by generalizing the role of the discounted state visitation distribution. 
\citep{laroche2022beyond} proposes to use the gradient of the cross-entropy loss with respect to the action with maximum Q.
Both approaches are shown to attain similar convergence rates as the on-policy true PG.
Different from all the above, CAPO serves as the first attempt to address off-policy policy optimization through the lens of coordinate ascent, without using the policy gradient.

\textbf{Exploiting the Sign of Advantage Function.}
As pointed out in Section \ref{section:prelim}, the sign of the advantage function (or temporal difference (TD) residual as a surrogate) can serve as an indicator of policy improvement.
For example, \citep{van2007reinforcement} proposed Actor Critic Learning Automaton (ACLA), which is designed to reinforce only those state-action pairs with positive TD residual and ignore those pairs with non-positive TD residual.  
The idea of ACLA is later extended by \citep{zimmer2016neural} to Neural Fitted Actor Critic (NFAC) , which learns neural policies for continuous control, and penalized version of NFAC for improved empirical performance \citep{zimmer2019exploiting}.
On the other hand, \citep{tessler2019distributional} proposes generative actor critic (GAC), a distributional policy optimization approach that leverages the actions with positive advantage to construct a target distribution.
%The theory of GAC relies on a three timescale assumption.
By contrast, CAPO takes the first step towards understanding the use of coordinate ascent with convergence guarantees for off-policy RL.

% \textbf{Approximate policy iteration.}
% Another line of related research is on the approximate policy iteration.

%% file: 8-conclusion.tex
\section{Conclusion}
\label{section:conclusion}
%With the extensive research on policy gradient based methods, the mismatch between theoretical assumptions and real world implementation is overlooked. 
We propose CAPO, which takes the first step towards addressing off-policy policy optimization by exploring the use of coordinate ascent in RL.
Through CAPO, we enable off-policy learning without the need for importance sampling or distribution correction.
We show that the general CAPO can attain asymptotic global convergence and establish the convergence rates of CAPO with several popular coordinate selection rules.
Moreover, through experiments, we show that the neural implementation of CAPO can serve as a competitive solution compared to the benchmark RL methods and thereby demonstrates the future potential of CAPO. 
%to address such problems and opens up the potential of off-policy sampling without importance sampling and any estimation of $d^\pi$. 
%It is surprising that simple coordinate update can provide great theoretical property such as global convergence. We also establish the importance of variable learning rate to escape from the early commitment, especially under the on-policy setting in \Cref{section:method:onCAPO}. 
%Moreover, we explored the potential of coordinate gradient ascent which is left unexplored in the literature of RL. As mention in \Cref{section:exp}, we compare the unrefined NCAPO with some of the most well-studied methods, the results demonstrates the future potential of CAPO.

%% file: 9-appendix.tex
\input{./appendices/appendix_method}
\input{./appendices/appendix_convergence_rate}

\input{./appendices/appendix_OnCAPO}
\input{./appendices/appendix_OnCAPO_fixed}
\input{./appendices/appendix_bandit_exp}
\input{./appendices/appendix_generator}
\input{./appendices/appendix_experiments}
\input{./appendices/appendix_pseudocode}

\input{./appendices/appendix_other}

%% file: appendices/appendix_method.tex
\section{Proofs of the Theoretical Results in Section \ref{subsection:analysis}}
\subsection{Proof of Lemma \ref{lemma:strict_improvement}}
\begin{lemma}[Performance Difference Lemma in \citep{Kakade2002approx}]
\label{lemma:perf_diff} 
For each state $s_0$, the difference in the value of $s_0$ between two policies $\pi$ and $\pi^{\prime}$ can be characterized as:
\begin{equation}
V^{\pi}\left(s_{0}\right)-V^{\pi^{\prime}}\left(s_{0}\right)=\frac{1}{1-\gamma} \mathbb{E}_{s \sim d_{s_{0}}^{\pi}} \mathbb{E}_{a \sim \pi(\cdot \rvert s)}\left[A^{\pi^{\prime}}(s, a)\right]
\end{equation}
\end{lemma}

Now we are ready to prove Lemma \ref{lemma:strict_improvement}. For ease of exposition, we restate Lemma \ref{lemma:strict_improvement} as follows.
\label{app:method}
\begin{lemma*}
Under the CAPO update given by (\ref{eq:CAPO_form}), we have $V^{\pi_{m+1}}(s) \geq V^{\pi_{m}}(s)$, for all $s \in S$, for all $m\in \mathbb{N}$.  
\label{app:proof_strict_improvement}
\end{lemma*}

\begin{proof}[Proof of \Cref{lemma:strict_improvement}]
Note that by the definition of $A(s, a)$, we have
\begin{equation}
\label{eq:sumAdvZero}
\sum_{a \in \mathcal{A}}\pi_m(a|s)A^{m}(s,a) = 0, \quad\forall s \in \mathcal{S}
\end{equation}

To simplify notation, let $Z_{m}(s) := \sum_{a \in \mathcal{A}}\exp({\theta_{m}(s,a)})$. Then, $\pi_{m} (a|s)$ and $\pi_{m+1} (a|s)$ can be simplified as:
\begin{equation}
\pi_{m} (a|s) = \frac{\exp({\theta_m(s,a)})}{Z_{m}(s)},\quad 
\pi_{m+1} (a|s) = \frac{\exp({\theta_{m+1}(s,a)})}{Z_{m+1}(s)}.\label{eq:sum pi and A}
\end{equation}
By Lemma \ref{lemma:perf_diff}, in order to show that 
$V^{\pi_{m+1}}(s) \geq V^{\pi_{m}}(s), \forall s \in S$,
it is sufficient to show that
\begin{equation}
\label{eq:posA}
\sum_{a\in \cA}\pi_{m+1}(a|s)A^{m}(s,a) > 0, \quad\forall s \in \mathcal{S}.
\end{equation}
%To begin with, we first introduce the following proposition, which would serve as a useful immediate step for the proof of Theorem \ref{theorem:convergeoptimal}.
%Proposition \ref{prop:theta_update_direction} indicates that the coordinate update direction is solely determined by the sign of $A(s,a)$.
%\begin{prop} 
%Under the CAPO update given by (\ref{eq:CAPO_form}), if $A(s,a) >0$ then $\theta_{m+1}(s,a) > \theta_{m}(s,a)$; otherwise $\theta_{m+1}(s,a) \leq \theta_{m}(s,a)$.
%\label{prop:theta_update_direction}
%\end{prop}
% Under the tabular parameterization, $\frac{\partial \pi_{\theta(s, a)}}{\partial \theta_{s', a}} = \frac{\partial \pi_{\theta(s, a)}}{\partial \theta_{s', a'}} = 0$ if $s' \neq s$, thus we can split $\mathcal{B}$ to $\cup_{s \in \mathcal{S}}\mathcal{B}(s)$ and discuss $\mathcal{B}(s) = \{(s,a) \rvert s \in \mathcal{B} \}$ separately. Where the set $\{a \in \mathcal{B}(s) \}= \{ a \rvert (s, a) \in \mathcal{B}(s) \}$ and set $\{a \not\in \mathcal{B}(s) \} = \mathcal{B}(s) - \{a \in \mathcal{B}(s)\}$
For ease of notation, we define ${B}_m(s) := \{a\rvert (s,a)\in {B}_m \}$.
%, and the set $\{a \in \mathcal{B}(s) \}= \{ a \rvert (s, a) \in \mathcal{B}(s) \}$ and $\{a \not\in \mathcal{B}(s) \} = \mathcal{A} - \{a \in \mathcal{B}(s)\}$.
To establish (\ref{eq:sum pi and A}), we have that for all $s \in \mathcal{S}$,
\begin{align}
\sum_{a \in \mathcal{A}}\pi_{m+1}(a|s)A^{m}(s,a)
&= \sum_{a \in \mathcal{A}}\frac{\exp({\theta_{m+1}(s,a)})}{Z_{m+1}(s)}A^m(s,a) &&\\
&= \frac{Z_{m}(s)}{Z_{m+1}(s)} \sum_{a \in \mathcal{A} }\frac{\exp({\theta_{m+1}(s,a)})}{Z_m(s)}A^m(s,a) &&\\
%\end{flalign*}
%\begin{flalign*}
&= \frac{Z_{m}(s)}{Z_{m+1}(s)} \left [ \sum_{a \in {B}_m(s)}
\frac{\exp({\theta_{m+1}(s,a)})}{Z_m(s)}A^m(s, a)+\sum_{a \not\in {B}_m(s)}\frac{\exp({\theta_{m}(s,a)})}{Z_m(s)}A^m(s, a) \right ]  &&\\
%\end{flalign*}
%\begin{flalign*}
&>\frac{Z_{m}(s)}{Z_{m+1}(s)} \left[ \sum_{a \in {B}_m(s)}\frac{\exp({\theta_{m}(s,a)})}{Z_m(s)}A^m(s, a)+ \sum_{a \not\in {B}_m(s)}\frac{\exp({\theta_{m}(s,a)})}{Z_m(s)} A^m(s, a)
\right] \label{eq:lemma2 ineq}&&\\
%\end{flalign*}
%\begin{flalign*}
&=\frac{Z_{m}(s)}{Z_{m+1}(s)} \sum_{a \in \mathcal{A}}\pi_m(a|s)A^{m}(s,a)&& \\
&= 0, &&
\end{align}
% Later in section \ref{subsection:analysis}, $(s, a) \in \mathcal{B}$ is all the state-action pair in the batch,
where (\ref{eq:lemma2 ineq}) holds by the CAPO update given by (\ref{eq:CAPO_form}).
\end{proof}

\subsection{Proof of Theorem \ref{theorem:convergeoptimal}}
Since $\{V^m\}$ is bounded above and enjoys strict improvement by \Cref{lemma:strict_improvement}. By the monotone convergence theorem, the limit of $\{V^m\}$ is guaranteed to exist. 
Similarly, we know that the limit of $\{Q^m\}$ also exists.
We use $V^{(\infty)}(s)$ and $Q^{(\infty)}(s,a)$ to denote the limits of $\{V^{(m)}(s)\}$ and $\{Q^{(m)}(s,a)\}$, respectively.
We also define $A^{(\infty)}(s, a):=Q^{(\infty)}(s, a)-V^{(\infty)}(s)$.
Our concern is whether the corresponding policy $\pi_{\infty}$ is optimal.
Inspired by \citep{agarwal2020theory}, we first define the following three sets as
\begin{align}
I_{0}^{s} &:=\left\{a \rvert Q^{(\infty)}(s, a)=V^{(\infty)}(s)\right\}, \\
I_{+}^{s} &:=\left\{a \rvert Q^{(\infty)}(s, a)>V^{(\infty)}(s)\right\}, \\
I_{-}^{s} &:=\left\{a \rvert Q^{(\infty)}(s, a)<V^{(\infty)}(s)\right\}.
\end{align}
By definition, $V^{(\infty)}$ is optimal if and only if $I_{+}^{s}$ is empty, for all $s$. 
We prove by contradiction that $V^{(\infty)}$ is optimal by showing $I_{+}^{s} = \emptyset$.

\noindent \textbf{Main steps of the proof}. The proof procedure can be summarized as follows:
\begin{itemize}[leftmargin=*]
  \item Step 1: We first assume $V^{(\infty)}$ is not optimal so that by definition $\exists s \in \mathcal{S}, I_{+}^{s} \neq \emptyset$.
  \item Step 2: We then show in \Cref{lemma:sumIzero}, $\forall s \in \mathcal{S}$, actions $a_{-} \in I_{-}^{s}$ have zero weights in policy (i.e.\ $\pi_{\infty}(a_{-} | s) = 0$, $\forall a_{-} \in I_{-}^{s}$).
  \item Step 3: Since the actions in $I_{-}^{s}$ have zero probability, by (\ref{eq:sumAdvZero}), this directly implies \Cref{lemma:sumI+zero}: $\forall I_{+}^{s} \neq \emptyset$, $a \in I_{+}^{s}$ must also have zero probability (i.e.\ $\pi_{\infty}(a_{+} | s) = 0$, $\forall a_{+} \in I_{+}^{s}$).
  \item Step 4: Moreover, under CAPO, in the sequel we can show \Cref{lemma:contradiction_of_a+}, which states that as long as \Cref{condition:sa} is satisfied, there must exist one action $a_+ \in I_{+}^{s}$ such that $\lim_{m \rightarrow \infty} \pi_m(a_+|s) = 1$. This contradicts the assumption that $\exists s \in \mathcal{S}, I_{+}^{s} \neq \emptyset$, proving that $I_{+}^{s} = \emptyset, \forall s \in \mathcal{S}$.
  %$\pi_{m}(a_+ | s) > 0$, $\forall a_+ \in I_{+}^{s}, m>M_1$ will always hold, thus contradicting with the assumption that $\exists s \in \mathcal{S}, I_{+}^{s} \neq \emptyset$, proving that $I_{+}^{s} = \emptyset, \forall s \in \mathcal{S}$.
\end{itemize}

\begin{lemma} Under CAPO, there exists $M_1$ such that for all $m > M_1, s \in \mathcal{S}, a \in \mathcal{A}$, we have :
\label{lemma:A_strict}
\begin{align}
A^{(m)}(s, a)<-\frac{\Delta}{4}, &\quad\text{ for } a \in I_{-}^{s},\\
A^{(m)}(s, a)>\frac{\Delta}{4}, &\quad\text{ for } a \in I_{+}^{s},
\end{align}
where $\Delta:=\min_{\left\{s, a \rvert A^{(\infty)}(s, a) \neq 0\right\}}\left|A^{(\infty)}(s, a)\right|$.

%This lemma comes directly from Lemma C.4 in \citet{agarwal2020theory}.
\end{lemma}
\begin{proof}[Proof of \Cref{lemma:A_strict}]
Given the strict policy improvement property of CAPO in Lemma \ref{lemma:strict_improvement}, this can be shown by applying Lemma C.4 in \citep{agarwal2020theory}. 
\end{proof}

% \begin{equation}
% \label{app:eq:CAPO_update_alpha}
% \theta_{m+1}(s, a) = \theta_{m}(s, a) + \log \left( \frac{1}{\pi\left(a \rvert s\right)} \right) \sign\left(A^{m}\left(s, a\right)\right)
% \end{equation}

%\begin{proof}  

%\begin{remark}
%\Cref{lemma:A_strict} shows that for $m > M_1$, the sign of $A^{(m)}(s, a)$ is fixed for both $I_{-}^{s}$ and $I_{+}^{s}$.
%\end{remark}

%
%\begin{assumption} $V^\infty$ is not optimal, by definition implying there exists some $s$ such that $I_{+}^{s} \neq \emptyset$.
%\end{assumption}
%

\begin{lemma}
\label{lemma:sumIzero} 
Under CAPO, $\pi_{\infty}(a_{-}|s) = 0, \forall s \in \mathcal{S}, a_{-} \in I^{s}_{-}$.
%
%This implies that for all $m > M_1$, sign of $A^{(m)}(s, a)$ is fixed. Since under Algorithm $\ref{algo:CAPO}$, $\theta_{m+1}(s, a) > \theta_{m}(s, a)$ if and only if $A^{(m)}(s, a) > 0$. Thus it is guaranteed that $\forall a \in I^{s}_{-}, m > M_1$, $\theta(s, a)$ is non-increasing.
\begin{proof}[Proof of \Cref{lemma:sumIzero}]
Lemma \ref{lemma:A_strict} shows that for all $m > M_1$, the sign of $A^{(m)}(s, a)$ is fixed. 
Moreover, we know that under CAPO update, $\theta_m(s, a_{-})$ is non-increasing, $\forall a_{-} \in I^{s}_{-}, \forall m> M_1$.  
Similarly, $\forall a_{+} \in I^{s}_{+}, m > M_1$, $\theta_m(s, a_{+})$ is non-decreasing.
By \Cref{condition:sa}, all the state-action pairs with negative advantage are guaranteed to be sampled for infinitely many times as $m \rightarrow \infty$.
Under the CAPO update in (\ref{eq:CAPO_form}), we have
\begin{equation}
\theta_{m+1}(s, a_{-}) - \theta_{m}(a_{-} \rvert s) \le -\log( \frac{1}{\pi_m(a_{-} \rvert s)} ) < 0.
\end{equation}
Given the infinite visitation, we know that $\lim_{m \rightarrow \infty}\theta_m(s, a_{-}) = -\infty$.
\end{proof}
\end{lemma}

We now show in \Cref{lemma:sumI+zero} that \Cref{lemma:sumIzero} implies $\sum_{a_{+} \in I_{+}^{s}} \pi_{\infty}(a_{+} \rvert s) = 0$.
\begin{lemma}
\label{lemma:sumI+zero}
If $I_{+}^{s} \neq \emptyset$ is true, then \Cref{lemma:sumIzero} implies $\sum_{a_{+} \in I_{+}^{s}} \pi_{\infty}(a_{+}|s) = 0$.
\begin{proof}[Proof of \Cref{lemma:sumI+zero}]
Recall from (\ref{eq:sumAdvZero}) that
$\sum_{a \in \mathcal{A}} \pi_{m}(a \rvert s) A^{m}(s, a)=0, \forall s \in \mathcal{S}, m>0$.
By definition, $\sum_{a_{0} \in I_{0}^{s}} \pi_{\infty}(a_{0}|s)A^{\infty}(s, a_0) = 0$, which directly implies that
\begin{align}
\sum_{a_{+} \in I_{+}^{s}} \pi_{\infty}(a_{+}|s)A^{\infty}(s, a_{+})
&= \sum_{a \in \mathcal{A}} \pi_{\infty}(a \rvert s) A^{\infty}(s, a) - \sum_{a_{0} \in I_{0}^{s}} \pi_{\infty}(a_{0} \rvert s)A^{\infty}(s, a)  -\sum_{a_{-} \in I_{-}^{s}} \pi_{\infty}(a_{-} \rvert s)A^{\infty}(s, a)  &&\\
&= 0 - 0 - 0 = 0, &&
\end{align}
where the second equality holds by \Cref{lemma:sumIzero}.
Since $A^{\infty}(s, a_{+}) > 0$ and $\pi_{\infty}(a_{+} \rvert s) \ge 0$, we have $\sum_{a_{+} \in I_{+}^{s}} \pi_{\infty}(a_{+} \rvert s) = 0$.
This completes the proof of Lemma \ref{lemma:sumI+zero}.
\end{proof}
\end{lemma}

In \Cref{lemma:sumI+zero}, we have that if $I_{+}^{s} \neq \emptyset$ is true, then $\pi_{m}(a_{+} \rvert s) \rightarrow 0$ as $m \rightarrow \infty$. 
To establish contradiction, we proceed to show in the following \Cref{lemma:contradiction_of_a+} that there must exist one action $a \in I_{+}^{s}$ such that $\lim_{m \rightarrow \infty} \pi_m(a|s) = 1$, which contradicts \Cref{lemma:sumI+zero} and hence implies the desired result that $I_{+}^{s} = \emptyset$.
%This is done since by the definition of limit, \Cref{lemma:sumI+zero} implies that $\forall \epsilon >0$, there must exist some $M^\epsilon_s > M_1$, such that for $m > M^\epsilon_s$, $\pi_m(s, a_{+}) < \epsilon, \forall a_{+} \in I_{+}^{s}$. 

If $I_{+}^{s} \neq \emptyset$ is true, then there exist $K$ such that $\forall m > K, s \in \mathcal{S}$, we have:
\begin{align}
Q^{m}(s,a^{+}) > Q^{m}(s,a^{0}) > Q^{m}(s,a^{-}) , \quad \text{for all $a^{+} \in I^{s}_{+}$, $a^{0} \in I^{s}_{0}$, $a^{-} \in I^{s}_{-}$}.
\end{align}
Without loss of generality, assume that the order of $Q^{m}$, $\forall m > K$, can be written as
\begin{align}
Q^{m}(s,\tilde {a}^{+}) > Q^{m}(s,a_1) > Q^{m}(s,a_2) > \dots  > Q^{m}(s,a_{|\mathcal{A}|-1}), \quad \text{provided that $I_{+}^{s} \neq \emptyset$},
\end{align}
where $\tilde {a}^{+} := \argmax_{a^{+} \in I^{s}_{+}} Q^{(\infty)}(s, a^{+})$. Note that we simplify the case above by considering "strictly greater than" instead of "greater than or equal to", but the simplification can be relaxed with a little extra work.

\begin{claim} 
\label{lemma:contradiction_of_a+}
If $I_{+}^{s} \neq \emptyset$ is true, then there must exist one action $a_+ \in I_{+}^{s}$ such that $\lim_{m \rightarrow \infty} \pi_m(a_+|s) = 1$ under (\ref{eq:CAPO_form}) with  $\alpha_m(s, a) \ge \log \frac{1}{\pi_m(a|s)}$.
%with ${\alpha_m(s, a) := f(\frac{1}{\pi_m(a|s)})}$ where $f(x) \in \Omega(\ln z).$
\end{claim}

%\begin{proof}  
%Since $V^{m}$ and $Q^{m}$ will converge to a stationary point, $V^{(\infty)}$ and $Q^{(\infty)}$. 

To establish \Cref{lemma:contradiction_of_a+}, we show that if $I_{+}^{s} \neq \emptyset$, then $\lim_{m \rightarrow \infty} \pi_m(a|s) = 0$ for all $a \ne \tilde {a}^{+}$ by induction.
For ease of exposition, we first present the following propositions.

\begin{prop} For any $m \ge 1, s \in \mathcal{S}, a \in \mathcal{A}$, if $A^{m}(s,a) \le 0$ and $\exists \enspace a' \ne a$, $a' \in {B}_{m}(s)$, satisfying $A^{m}(s,a') > 0$, then $\pi_{m+1}(a|s) \le \frac{1}{2}$, regardless of whether $a \in {B}_{m}(s)$ or not.
\label{prop:upperbdd_of_pi}
\end{prop}
\begin{proof}[Proof of \Cref{prop:upperbdd_of_pi}]
\phantom{}

Since $A^{m}(s,a) \le 0$, we have $\sign(A^m(s,a)) \cdot \alpha_{m}(s,a) \le 0$. As a result, we have:
\begin{align}
\begin{cases}
\pi_{m+1}(a|s) = \frac{ \exp({\theta_{m}(s, a) + \sign(A^m(s,a)) \cdot \alpha_{m}(s,a)}) }{ Z_{m+1}(s) } \le \frac{ \exp({\theta_{m}(s, a)})}{ Z_{m+1}(s) } \le \frac{Z_{m}(s)}{Z_{m+1}(s)}  \\
\pi_{m+1}(a^{'}|s) = \frac{ \exp({\theta_{m}(s, a') + \alpha_{m}(s,a')}) }{ Z_{m+1}(s) } \ge \frac{ \exp({\theta_{m}(s, a') + \log(\frac{1}{ \pi_m(a'|s)})})}{ Z_{m+1}(s) } = \frac{Z_{m}(s)}{Z_{m+1}(s)}
\end{cases}
\end{align}
Hence, we have $\pi_{m+1}(a^{'}|s) \ge \pi_{m+1}(a|s)$. Since $\pi_{m+1}(a^{'}|s) + \pi_{m+1}(a|s) \le 1$, we get $\pi_{m+1}(a|s) \le \frac{1}{2}$.

\end{proof}

\begin{prop} 
\label{prop:adv_all_neg}
For any $s \in \mathcal{S}, a \in \mathcal{A}\setminus \left \{ \tilde {a}^{+} \right \} $, if $\exists \enspace T \in \mathbb{N}$ such that $\forall m > T$, $A^m(s,a) \le 0$, then $\exists \enspace n \in \mathbb{N}$, $\bar{K} \in \mathbb{N}$ such that $A^{m+n+1}(s,a) < 0 \text{, } \forall \enspace m > \bar{K} $.
\end{prop}
\begin{proof}[Proof of \Cref{prop:adv_all_neg}]
\phantom{}

By Condition \ref{condition:sa} and $I_{+}^{s} \neq \emptyset$, there exist some finite $n \in \mathbb{N}$ such that $\exists \enspace a' \ne a$, $a' \in {B}_{m+n}(s)$, satisfying $A^{m+n+1}(s,a') > 0$.
Then, by \Cref{prop:upperbdd_of_pi}, we have
\begin{align}
\pi_{m+n+1}(a|s) \le \frac{1}{2}, \quad \text{$\forall m \ge T$}.
\end{align}
Hence, we have
\begin{align}
V^{m+n+1}(s) = \sum_{a \in \mathcal{A}} \pi_{m+n+1}(a|s) \cdot Q^{m+n+1}(s,a) \ge \frac{1}{2} \cdot \left ( Q^{m+n+1}(s,a) + Q^{m+n+1}(s,a') \right ) \\
\text{where $a' = \underset{\substack{a'' \in \mathcal{A} \\ Q^{\infty}(s,a'')>Q^{\infty}(s,a)}}{\mathrm{argmin}}  Q^{\infty}(s,a'')$}
\end{align}
Moreover, by the ordering of $Q^{m}$ and that $\lim_{m \rightarrow \infty}Q^m(s, a) = Q^{\infty}(s, a)$, for $\epsilon = \frac{1}{4} \cdot \left ( Q^{\infty}(s, a')-Q^{\infty}(s, a) \right ) > 0 $, $\exists \enspace \bar{T}$ such that for all $m > \bar{T}$:
\begin{align}
\begin{cases}
Q^{m}(s, a) \in \left ( Q^{\infty}(s, a) - \epsilon, Q^{\infty}(s, a) + \epsilon \right )\\
Q^{m}(s, a') \in \left ( Q^{\infty}(s, a') - \epsilon, Q^{\infty}(s, a') + \epsilon \right )
\end{cases}
\end{align}
Finally, we have that for all $m > \max \left \{ T,\bar{T} \right \}$:
\begin{align}
V^{m+n+1}(s) &\ge \frac{1}{2} \cdot \left ( Q^{m+n+1}(s,a) + Q^{m+n+1}(s,a') \right ) \\
& > \frac{1}{2} \cdot \left ( (Q^{\infty}(s, a)) + (Q^{\infty}(s, a')) \right ) \\
& > \frac{1}{2} \cdot \left ( Q^{\infty}(s,a) + Q^{\infty}(s,a') \right ) - \epsilon \\
& = Q^{\infty}(s,a) + \epsilon > Q^{m+n+1}(s,a).
\end{align}
The above is equivalent to $A^{m+n+1}(s,a) < 0 \enspace \forall \enspace m > \bar{K}$, where $\bar{K} = \max \left \{ T,\bar{T} \right \}$.
\end{proof}

\begin{prop} 
\label{prop:ratio}
If $V^m(s) \in \left ( Q^m(s,a_{|\mathcal{A}|-k}), Q^m(s,a_{|\mathcal{A}|-(k+2)}) \right ) $, then $\exists \enspace T' \in \mathbb{N}$ such that for all $m > T'$:
\begin{align}
\frac{\sum_{\substack{a \in \mathcal{A} \\ Q^{\infty}(s,a)>Q^{\infty}(s,a_{|\mathcal{A}|-(k+1)})}} \pi_{m} (a|s)}{\sum_{\substack{a \in \mathcal{A} \\ Q^{\infty}(s,a)<Q^{\infty}(s,a_{|\mathcal{A}|-(k+1)})}} \pi_{m} (a|s)} \ge \frac{Q^m(s,a_{|\mathcal{A}|-(k+1)}) - Q^m(s,a_{|\mathcal{A}|-1})}{Q^m(s,a_{|\mathcal{A}|-(k+2)}) - Q^m(s,a_{|\mathcal{A}|-(k+1)})}\\
\text{provided $V^m(s) \in \left ( Q^m(s,a_{|\mathcal{A}|-k}), Q^m(s,a_{|\mathcal{A}|-(k+2)}) \right ) $}
\end{align}
\end{prop}
\begin{proof}[Proof of \Cref{prop:ratio}]
\phantom{}

Since $V^m(s) \in \left ( Q^m(s,a_{|\mathcal{A}|-k}), Q^m(s,a_{|\mathcal{A}|-(k+2)}) \right ) $, we have $A^m(s,a_{|\mathcal{A}|-j}) < 0, \enspace \forall \enspace j = 1, 2, \dots , k$. 
By Condition \ref{condition:sa}, there exists some finite $n \in \mathbb{N}$ such that $\bar{a}^{+} \in \mathcal{B}_{m+n}(s)$ for some $\bar{a}^{+} \in \left \{ \tilde{a}^{+}, a_1, a_2, \dots, a_{|\mathcal{A}|-(k+2)} \right \} $. \\
Hence, we have that for all $\bar{a}^{-} \in \left \{ a_{|\mathcal{A}|-k}, a_{|\mathcal{A}|-(k-1)}, \dots, a_{|\mathcal{A}|-1} \right \} $,
\begin{align}
\frac{\pi_{m+n+1}(\bar{a}^{+}|s)}{\pi_{m+n+1}(\bar{a}^{-}|s)} \ge \frac{ \frac{ e^{\theta_{m+n}(s, \bar{a}^+ ) + \log{\frac{1}{\pi_{m+n}(\bar{a}^+|s)} } }}{ Z_{m+n+1}(s) }}{ \frac{ e^{\theta_{m+n}(s, \bar{a}^- ) }}{ Z_{m+n+1}(s) }} = \frac{ Z_{m+n}(s) }{ e^{\theta_{m+n}(s, \bar{a}^-)} } = \frac{1}{\pi_{m+n}(\bar{a}^{-}|s)}.
\end{align}

Since $\lim_{m \rightarrow \infty}\pi_m(s, a) = 0$, we have $\forall z \in \mathbb{Z}$, $\exists \enspace T \in \mathbb{N}$ such that $\frac{\pi_{m+n+1}(\bar{a}^{+}|s)}{\pi_{m+n+1}(\bar{a}^{-}|s)} \ge z, \enspace \forall m > T$.
For $m > K$, we have $Q^m(s,a_{|\mathcal{A}|-(k+2)}) - Q^m(s,a_{|\mathcal{A}|-(k+1)}) > 0$.
Hence, by simply choosing $z = \frac{1}{\mathcal{A}} \cdot \frac{Q^m(s,a_{|\mathcal{A}|-(k+1)}) - Q^m(s,a_{|\mathcal{A}|-1})}{Q^m(s,a_{|\mathcal{A}|-(k+2)}) - Q^m(s,a_{|\mathcal{A}|-(k+1)})}$ and taking the summation of the ratio over $\bar{a}^{+}$ and $\bar{a}^{-}$, we can reach the desired result with $T'=\max \left \{ K, T \right \} $.
\end{proof}

Now, we are ready to prove \Cref{lemma:contradiction_of_a+} by an induction argument.

\textit{Proof of \Cref{lemma:contradiction_of_a+}.}
\begin{itemize}[leftmargin=*]
    \item Show that if $I_{+}^{s} \neq \emptyset$, then $\lim_{m \rightarrow \infty} \pi_m(a_{|\mathcal{A}|-1}|s) = 0$:\\
    By the ordering of $Q^{m}$, we have:
    \begin{align}
    %\min_{\pi_m(\cdot|s)} \left \{ V^m(s) \right \}  = \min_{\pi_m(\cdot|s)} \left \{ \sum_{a \in \mathcal{A}} \pi_m(a|s) \cdot Q^m(s,a) \right \} = 1 \cdot Q^{m}(s,a_{|\mathcal{A}|-1}) \text{, \quad $\forall m > K$ }.
    V^m(s) = \sum_{a \in \mathcal{A}} \pi_m(a|s) \cdot Q^m(s,a) \ge 1 \cdot Q^{m}(s,a_{|\mathcal{A}|-1}) \text{, \quad $\forall m > K$ }.
    \end{align}
    Hence, for all $m > K$, we have:
    \begin{align}
    %A^m(s,a_{|\mathcal{A}|-1}) = Q^m(s,a_{|\mathcal{A}|-1}) - V^m(s) \le Q^m(s,a_{|\mathcal{A}|-1}) - \min_{\pi_m(\cdot|s)} \left \{ V^m(s) \right \} = 0.
    A^m(s,a_{|\mathcal{A}|-1}) = Q^m(s,a_{|\mathcal{A}|-1}) - V^m(s) \le Q^m(s,a_{|\mathcal{A}|-1}) - Q^m(s,a_{|\mathcal{A}|-1}) = 0.
    \end{align}
    Therefore, by \Cref{prop:adv_all_neg}, we have $\exists \enspace n_{|\mathcal{A}|-1} \in \mathbb{N}$, $K_{|\mathcal{A}|-1} \in \mathbb{N}$ such that:
    \begin{align}
    A^{m+n_{|\mathcal{A}|-1}+1}(s,a_{|\mathcal{A}|-1}) < 0 \text{, \quad $\forall m > K_{|\mathcal{A}|-1}$ }.
    \end{align}
    Moreover,
    \begin{align}
    \sign(A^m(s,a_{|\mathcal{A}|-1})) \cdot \alpha_m\left ( s,a_{|\mathcal{A}|-1} \right ) < 0 \text{, \quad $\forall m > K_{|\mathcal{A}|-1}$. }
    \end{align}
    With the monotone-decreasing property and the infinite visitation condition, it is guaranteed that $\lim_{m \rightarrow \infty}\theta_m(s, a_{|\mathcal{A}|-1}) = -\infty$. Hence, we have $\lim_{m \rightarrow \infty}\pi_m(s, a_{|\mathcal{A}|-1}) = 0$.

    \item Suppose that $\lim_{m \rightarrow \infty}\pi_m(a_{|\mathcal{A}|-1}|s) = \lim_{m \rightarrow \infty}\pi_m(a_{|\mathcal{A}|-2}|s) = \dots = \lim_{m \rightarrow \infty}\pi_m(a_{|\mathcal{A}|-k}|s) = 0$, where $k \in \left [ 1, (|\mathcal{A}|-2) \right ] $. Then we would like to derive $\lim_{m \rightarrow \infty}\pi_m(a_{|\mathcal{A}|-(k+1)}|s)$:\\
    By the above assumption, we have:
    \begin{align}
    \lim_{m \rightarrow \infty} \sum_{\substack{a \in \mathcal{A} \\ Q^{\infty}(s,a)<Q^{\infty}(s,a_{|\mathcal{A}|-(k+1)})}} \pi_{m} (a|s) = 0
    \end{align}
    By \Cref{prop:ratio}, $\exists \enspace K'_{|\mathcal{A}|-(k+1)} \in \mathbb{N}$ such that $\forall \enspace m >  K'_{|\mathcal{A}|-(k+1)} $, we can establish the ratio between the summation of policy weight of the policy worse than $a_{|\mathcal{A}|-(k+1)}$ and the policy better than $a_{|\mathcal{A}|-(k+1)}$:
    %m > \max_{j \in [1,k]} \left \{ K, K_{|\mathcal{A}|-j}
    \begin{align}
    \frac{\sum_{\substack{a \in \mathcal{A} \\ Q^{\infty}(s,a)>Q^{\infty}(s,a_{|\mathcal{A}|-(k+1)})}} \pi_{m} (a|s)}{\sum_{\substack{a \in \mathcal{A} \\ Q^{\infty}(s,a)<Q^{\infty}(s,a_{|\mathcal{A}|-(k+1)})}} \pi_{m} (a|s)} \ge \frac{Q^m(s,a_{|\mathcal{A}|-(k+1)}) - Q^m(s,a_{|\mathcal{A}|-1})}{Q^m(s,a_{|\mathcal{A}|-(k+2)}) - Q^m(s,a_{|\mathcal{A}|-(k+1)})}\\
    \text{provided $V^m(s) \in \left ( Q^m(s,a_{|\mathcal{A}|-k}), Q^m(s,a_{|\mathcal{A}|-(k+2)}) \right ) $}
    \end{align}
    And by the ordering of $Q^{m}$, we have:
    \begin{align}
    V^m(s) &= \sum_{a \in \mathcal{A}} \pi_m(a|s) \cdot Q^m(s,a) &&\\
    &= \left [ \sum_{\substack{a \in \mathcal{A} \\ Q^{\infty}(s,a)>Q^{\infty}(s,a_{|\mathcal{A}|-(k+1)})}} \pi_{m} (a|s) \cdot Q^{m}(s,a) + Q^{m}(s,a_{|\mathcal{A}|-(k+1)}) \cdot \pi_m(a_{|\mathcal{A}|-(k+1)}|s)  \right. &&\\
    &\left. \quad + \sum_{\substack{a \in \mathcal{A} \\ Q^{\infty}(s,a)<Q^{\infty}(s,a_{|\mathcal{A}|-(k+1)})}} \pi_{m} (a|s) \cdot Q^{m}(s,a) \right ]  &&\\
    &\ge \left [  Q^{m}(s,a_{|\mathcal{A}|-(k+2)}) \cdot \sum_{\substack{a \in \mathcal{A} \\ Q^{\infty}(s,a)>Q^{\infty}(s,a_{|\mathcal{A}|-(k+1)})}} \pi_{m} (a|s)+ Q^{m}(s,a_{|\mathcal{A}|-(k+1)}) \cdot \pi_m(a_{|\mathcal{A}|-(k+1)}|s) \right. &&\\
    &\left. \quad + Q^{m}(s,a_{|\mathcal{A}|-1}) \cdot \sum_{\substack{a \in \mathcal{A} \\ Q^{\infty}(s,a)<Q^{\infty}(s,a_{|\mathcal{A}|-(k+1)})}} \pi_{m} (a|s) \right ]  \text{, \quad $\forall m > K$ }.
    \end{align}
    Hence, for all $m > K'_{|\mathcal{A}|-(k+1)}$, we have:
    \begingroup
    \allowdisplaybreaks
    \begin{align}
    A^m(s,a_{|\mathcal{A}|-(k+1)})
    & = Q^m(s,a_{|\mathcal{A}|-(k+1)}) - V^m(s) &&\\
    %& \le Q^m(s,a_{|\mathcal{A}|-(k+1)}) - \min_{\pi_m(\cdot|s)} \left \{ V^m(s) \right \} &&\\
    & \le Q^m(s,a_{|\mathcal{A}|-(k+1)}) - \left [  Q^{m}(s,a_{|\mathcal{A}|-(k+2)}) \cdot \sum_{\substack{a \in \mathcal{A} \\ Q^{\infty}(s,a)>Q^{\infty}(s,a_{|\mathcal{A}|-(k+1)})}} \pi_{m} (a|s) \right. &&\\
    &\left. \qquad \qquad \qquad \qquad \qquad \quad + Q^{m}(s,a_{|\mathcal{A}|-(k+1)}) \cdot \pi_m(a_{|\mathcal{A}|-(k+1)}|s) \right. &&\\
    &\left. \qquad \qquad \qquad \qquad \qquad \quad + Q^{m}(s,a_{|\mathcal{A}|-1}) \cdot \sum_{\substack{a \in \mathcal{A} \\ Q^{\infty}(s,a)<Q^{\infty}(s,a_{|\mathcal{A}|-(k+1)})}} \pi_{m} (a|s) \right ]  &&\\\\
    &= \left ( Q^{m}(s,a_{|\mathcal{A}|-(k+1)}) - Q^{m}(s,a_{|\mathcal{A}|-(k+2)}) \right ) \cdot \sum_{\substack{a \in \mathcal{A} \\ Q^{\infty}(s,a)>Q^{\infty}(s,a_{|\mathcal{A}|-(k+1)})}} \pi_{m} (a|s) &&\\
    & \quad + \left ( Q^{m}(s,a_{|\mathcal{A}|-(k+1)}) - Q^{m}(s,a_{|\mathcal{A}|-1}) \right ) \cdot \sum_{\substack{a \in \mathcal{A} \\ Q^{\infty}(s,a)<Q^{\infty}(s,a_{|\mathcal{A}|-(k+1)})}} \pi_{m} (a|s) &&\\
    &\le \left ( Q^{m}(s,a_{|\mathcal{A}|-(k+1)}) - Q^{m}(s,a_{|\mathcal{A}|-(k+2)}) \right ) \cdot \sum_{\substack{a \in \mathcal{A} \\ Q^{\infty}(s,a)>Q^{\infty}(s,a_{|\mathcal{A}|-(k+1)})}} \pi_{m} (a|s) &&\\
    & \quad + \left ( Q^{m}(s,a_{|\mathcal{A}|-(k+1)}) - Q^{m}(s,a_{|\mathcal{A}|-1}) \right ) \cdot \frac{Q^m(s,a_{|\mathcal{A}|-(k+2)}) - Q^m(s,a_{|\mathcal{A}|-(k+1)})}{Q^m(s,a_{|\mathcal{A}|-(k+1)}) - Q^m(s,a_{|\mathcal{A}|-1})} &&\\
    & \qquad \cdot \sum_{\substack{a \in \mathcal{A} \\ Q^{\infty}(s,a)>Q^{\infty}(s,a_{|\mathcal{A}|-(k+1)})}} \pi_{m} (a|s) &&\\
    & = 0
    \end{align}
    By \Cref{prop:adv_all_neg}, we have $\exists \enspace n_{|\mathcal{A}|-(k+1)} \in \mathbb{N}$, $K_{|\mathcal{A}|-(k+1)} \in \mathbb{N}$ such that:
    \begin{align}
    A^{m+n_{|\mathcal{A}|-(k+1)}+1}(s,a_{|\mathcal{A}|-(k+1)}) < 0 \text{, \quad $\forall m > K_{|\mathcal{A}|-(k+1)}$ }
    \end{align}
    Moreover,
    \begin{align}
    \sign(A^m(s,a_{|\mathcal{A}|-1})) \cdot \alpha_m\left ( s,a_{|\mathcal{A}|-(k+1)} \right ) < 0 \text{, \quad $\forall m > K_{|\mathcal{A}|-(k+1)}$ }
    \end{align}
    \endgroup
    With the monotone-decreasing property and the infinite visitation, it is guaranteed that $\lim_{m \rightarrow \infty}\theta_m(s, a_{|\mathcal{A}|-(k+1)}) = -\infty$. Hence we have $\lim_{m \rightarrow \infty}\pi_m(s, a_{|\mathcal{A}|-(k+1)}) = 0$.\\
    
    Finally we complete the induction and so we conclude that $\forall a \ne \tilde{a^+}$, $\lim_{m \rightarrow \infty}\pi_m(s, a) = 0$, which is equivalent to $\lim_{m \rightarrow \infty}\pi_m(s, \tilde{a^+}) = 1$. This completes the proof of \Cref{lemma:contradiction_of_a+}.
    
\end{itemize}
\hfill $\qed$

Now we are ready to put everything together and prove Theorem \ref{theorem:convergeoptimal}.
For ease of exposition, we restate Theorem \ref{theorem:convergeoptimal} as follows.
\begin{theorem*}
Consider a tabular softmax parameterized policy $\pi_\theta$, under (\ref{eq:CAPO_form}) with ${\alpha_{m}(s, a) \ge \log (\frac{1}{\pi_{\theta_{m}}(a\rvert s)})}$, if Condition \ref{condition:sa} is satisfied, then we have $V^{\pi_{m}}(s) \rightarrow V^{*}(s)$ as $m\rightarrow \infty$, for all $s \in \mathcal{S}$.
\end{theorem*}
\label{app:proof_convergeoptimal}

\begin{proof}[Proof of \Cref{theorem:convergeoptimal}]
In \Cref{lemma:contradiction_of_a+}, we have that if $I_{+}^{s} \neq \emptyset$ is true, then there must exist one action $a \in I_{+}^{s}$ such that $\lim_{m \rightarrow \infty} \pi_m(a|s) = 1$. This leads to the contradiction with \Cref{lemma:sumI+zero}, and finally we get the desired result that $I_{+}^{s} = \emptyset$, implying that $V^{(\infty)}$ is optimal.
\end{proof}

%% file: appendices/appendix_convergence_rate.tex
\section{Proofs of the Convergence Rates of CAPO in Section \ref{subsection:convergence_rate}}
\label{app:convergence_rate}

\begin{lemma} $|A^m(s,a)| \le \frac{1}{1-\gamma} \cdot (1-\pi_m(a|s))$, \quad \text{for all $(s,a) \in \mathcal{S} \times \mathcal{A}$.}
\label{lemma:upper_bound_of_advantage}

\begin{proof}[Proof of \Cref{lemma:upper_bound_of_advantage}]  
\phantom{}

If $A^m(s,a) > 0$ :
\begin{align}
|A^m(s,a)|
&=Q^{\pi_m}(s,a) - V^{\pi_m}(s) &&\\
&=Q^{\pi_m}(s,a) - \sum_{a' \in \mathcal{A}}\pi_m(a'|s) \cdot Q^{\pi_m}(s,a') &&\\
&\le Q^{\pi_m}(s,a) - \pi_m(a|s) \cdot Q^{\pi_m}(s,a) &&\\
&= Q^{\pi_m}(s,a) \cdot (1 - \pi_m(a|s)) &&\\
&\le \frac{1}{1-\gamma} \cdot (1 - \pi_m(a|s)) &&
\end{align}

If $A^m(s,a) \le 0$ :
\begin{align}                            
|A^m(s,a)|
&= V^{\pi_m}(s) - Q^{\pi_m}(s,a)  &&\\
&= \sum_{a' \in \mathcal{A}}\pi_m(a'|s) \cdot Q^{\pi_m}(s,a') - Q^{\pi_m}(s,a)  &&\\
&= \sum_{a' \ne a}\pi_m(a'|s) \cdot Q^{\pi_m}(s,a') - (1 - \pi_m(a|s)) \cdot Q^{\pi_m}(s,a)  &&\\
&\le \sum_{a' \ne a}\pi_m(a'|s) \cdot Q^{\pi_m}(s,a') &&\\
&\le \frac{1}{1-\gamma} \cdot \sum_{a' \ne a}\pi_m(a'|s) &&\\
&= \frac{1}{1-\gamma} \cdot (1 - \pi_m(a|s)) &&
\end{align}
\end{proof}
\end{lemma}

\begin{lemma} $\left ( V^{*}(s) - V^{\pi_m}(s) \right )^2 \le \left ( \frac{1}{1-\gamma} \cdot A^{m}(\tilde{s_m},\tilde{a_m}) \right )^2$, \text{for all $m \ge 1$} where $(\tilde{s_m},\tilde{a_m}) = \underset{(s,a) \in \mathcal{S} \times \mathcal{A}}{\argmax} A^{m}(s,a)$.
\label{lemma:lower_bound_of_performance_difference}

\begin{proof}[Proof of \Cref{lemma:lower_bound_of_performance_difference}]  
\phantom{}

\begin{align}
\left ( V^{*}(s) - V^{\pi_m}(s) \right )^2 
&= \left ( \frac{1}{1-\gamma} \cdot \sum_{s' \in \mathcal{S}} d^{\pi^*}_{s} (s') \sum_{a' \in \mathcal{A}} \pi^*(a'|s') \cdot A^{m}(s',a') \right )^2&&\\
&\le  \left ( \frac{1}{1-\gamma} \cdot \sum_{s' \in \mathcal{S}} d^{\pi^*}_{s} (s') \cdot \underset{a' \in \mathcal{A}}{\max}A^{m}(s',a') \right )^2&&\\
&\le  \left ( \frac{1}{1-\gamma} \cdot \underset{(s',a') \in \mathcal{S} \times \mathcal{A}}{\max} A^{m}(s',a') \right )^2&&\\
&= \left ( \frac{1}{1-\gamma} \cdot A^{m}(\tilde{s_m},\tilde{a_m}) \right )^2&&
\end{align}

The first equation holds by \Cref{lemma:perf_diff}.\\
The first and the second inequality hold since the value inside the quadratic term is non-negative.

\end{proof}
\end{lemma}

\begin{lemma} $V^{*}(\rho) - V^{\pi_m}(\rho) \le \frac{1}{1-\gamma} \cdot \left \| \frac{1}{\mu} \right \|_{\infty} \cdot \left ( V^{*}(\mu) - V^{\pi_m}(\mu) \right )$
\label{lemma:performance_difference_in_rho}

\begin{proof}[Proof of \Cref{lemma:performance_difference_in_rho}]  
\phantom{}

\begingroup
\allowdisplaybreaks
\begin{align}
V^{*}(\rho) - V^{\pi_m}(\rho)
&= \frac{1}{1-\gamma} \cdot \sum_{s \in \mathcal{S}} d^{\pi^*}_{\rho} (s) \sum_{a \in \mathcal{A}} \pi^*(a|s) \cdot A^{m}(s,a)&&\\
&= \frac{1}{1-\gamma} \cdot \sum_{s \in \mathcal{S}} d^{\pi^*}_{\mu}(s) \cdot \frac{d^{\pi^*}_{\rho}(s)}{d^{\pi^*}_{\mu}(s)}  \sum_{a' \in \mathcal{A}} \pi^*(a|s) \cdot A^{m}(s,a) &&\\
&\le \frac{1}{1-\gamma} \cdot \left \| \frac{1}{d^{\pi^*}_{\mu}} \right \|_{\infty} \cdot \sum_{s \in \mathcal{S}} d^{\pi^*}_{\mu}(s)  \sum_{a' \in \mathcal{A}} \pi^*(a|s) \cdot A^{m}(s,a)&&\\
&\le \frac{1}{(1-\gamma)^2} \cdot \left \| \frac{1}{\mu} \right \|_{\infty} \cdot \sum_{s \in \mathcal{S}} d^{\pi^*}_{\mu}(s)  \sum_{a' \in \mathcal{A}} \pi^*(a|s) \cdot A^{m}(s,a)&&\\
&= \frac{1}{1-\gamma} \cdot \left \| \frac{1}{\mu} \right \|_{\infty} \cdot \left ( V^{*}(\mu) - V^{\pi_m}(\mu) \right ) &&
\end{align}
\endgroup

The first and the last equation holds by the performance difference lemma in \Cref{lemma:perf_diff}.\\
The first and second inequality holds since the value inside the summation is non-negative.

\end{proof}
\end{lemma}

\begin{lemma} $d^{\pi}_{\mu}(s) \ge (1-\gamma) \cdot \mu(s)$, \quad \text{for any $\pi, s \in \mathcal{S} $} where $\mu(s)$ is some starting state distribution of the MDP.
\label{lemma:lower_bound_of_state_visitation_distribution}

\begin{proof}[Proof of \Cref{lemma:lower_bound_of_state_visitation_distribution}]  
\phantom{}

\begin{align}
d^{\pi}_{\mu}(s)
&= \underset{s_0 \sim \mu}{\mathbb{E}} \left [ d^{\pi}_{\mu}(s) \right]&&\\
&= \underset{s_0 \sim \mu}{\mathbb{E}} \left [ (1-\gamma) \cdot \sum_{t=0}^{\infty}\gamma^t \cdot  \mathbb{P}(s_t=s \: | \: s_0, \pi) \right]&&\\
&\ge \underset{s_0 \sim \mu}{\mathbb{E}} \left [ (1-\gamma) \cdot \mathbb{P}(s_0=s \: | \: s_0, \pi) \right] &&\\
&= (1-\gamma) \cdot \mu(s)&&
\end{align}

The first equation holds by the performance difference lemma in \Cref{lemma:perf_diff}.\\
The second and the third equation hold since the value inside quadratic term is non-negative.

\end{proof}
\end{lemma}

\begin{lemma} Given $\delta_{m+1} \le \delta_{m} - c \cdot \delta_{m}^2$ where $\delta_{m} \le \frac{1}{1-\gamma}$ for all $m \ge 1$ and $c \le \frac{1-\gamma}{2}$, then $\delta_{m} \le \frac{1}{c} \cdot \frac{1}{m}$ and $\sum_{m=1}^{M} \delta_m \le \min{ \left \{ \sqrt{\frac{M}{c \cdot (1-\gamma)}}, \frac{\log M +1}{c} \right \}  }$ for all $m \ge 1$.
\label{lemma:induction}

\begin{proof}[Proof of \Cref{lemma:induction}]  
\phantom{}

We prove this lemma by induction. For $m \le 2$, $\delta_{m} \le \frac{1}{c} \cdot \frac{1}{m}$ directly holds since $c \le \frac{1-\gamma}{2}$ and $\delta_{m} \le \frac{1}{1-\gamma}$.\\
Let $f_t(x) = x - c \cdot x^2 = -c(x-\frac{1}{2c})^2 + \frac{1}{4c}$. Then $f_t(x)$ is monotonically increasing in $[0,\frac{1}{2c}]$. And so we have :
\begin{align}
\delta_{m+1}
&\le f_t(\delta_{m})&&\\
&\le f_t(\frac{1}{c} \cdot \frac{1}{m})&&\\
&= \frac{1}{c} \cdot (\frac{1}{m}-\frac{1}{m^2})&&\\
&\le \frac{1}{c} \cdot \frac{1}{m+1}
\end{align}

and by summing up $\delta_{m}$, we have :
\begin{align}
\sum_{m=1}^{M} \delta_{m}
&\le \sum_{m=1}^{M} \frac{1}{c} \cdot \frac{1}{m}&&\\
&= \frac{1}{c} \cdot \sum_{m=1}^{M} \frac{1}{m}&&\\
&\le \frac{1}{c} \cdot (\ln{M}+1)&&
\end{align}

On the other hand, we also have :
\begin{align}
\sum_{m=1}^{M} \delta_{m}^2
&\le \frac{1}{c} \cdot \sum_{m=1}^{M} (\delta_{m}-\delta_{m+1})&&\\
&\le \frac{1}{c} \cdot (\delta_{1}-\delta_{M+1})&&\\
&\le \frac{1}{c} \cdot \frac{1}{1-\gamma}&&
\end{align}

Therefore, by Cauchy-Schwarz,
\begin{align}
\sum_{m=1}^{M} \delta_{m}
&\le \sqrt{M} \cdot \sqrt{\sum_{m=1}^{M} \delta_{m}^2}&&\\
&\le \sqrt{M} \cdot \sqrt{\frac{1}{c} \cdot \frac{1}{1-\gamma}}&&\\
&=\sqrt{\frac{M}{c \cdot (1-\gamma)}}&&
\end{align}

\end{proof}
\end{lemma}

\begin{comment}
\begin{prop} $\sum\limits_{a \in \mathcal{A}}\pi(a|s) \cdot A^{\pi}(s,a) = 0, \forall  s \in \mathcal{S}$ .
\label{prop:pi_s_Adv}
\end{prop}
\end{comment}

\begin{lemma} Under the CAPO update (\ref{eq:CAPO_form}) with ${\alpha_m(s, a) = \log (\frac{1}{\pi_{\theta_m}(a\rvert s)})}$, if $B_m = \left \{  (s_m, a_m) \right \} $ and $A^{m}(s_m, a_m) > 0$, then the policy weight difference $\pi_{m+1}(a|s) - \pi_{m}(a|s)$ can be written as :
\label{lemma:change_of_policy_weight_1}
\begin{align}
\pi_{m+1}(a|s) - \pi_{m}(a|s) =
\begin{cases}
\frac{(1-\pi_{m}(a_m|s_m))^2}{2-\pi_{m}(a_m|s_m)}  & \text{, if } s = s_m, a = a_m \\
- \frac{1-\pi_{m}(a_m|s_m)}{2-\pi_{m}(a_m|s_m)} \cdot \pi_{m}(a|s)  & \text{, if } s = s_m, a \ne a_m\\
0 & \text{, else } 
\end{cases}
\end{align}

\begin{proof}[Proof of \Cref{lemma:change_of_policy_weight_1}]  
\phantom{}

For $s = s_m, a = a_m$:
\begingroup
\allowdisplaybreaks
\begin{align}
\pi_{m+1}(a_m|s_m) - \pi_{m}(a_m|s_m) 
&= \frac{e^{\theta_{m+1}(s_m, a_m)}}{ \sum\limits_{a \in \mathcal{A} } e^{\theta_{m+1}(s_m, a)} } - \pi_m(a_m|s_m) &&\\
&= \frac{e^{\theta_{m}(s_m, a_m)+ln(\frac{1}{\pi_m(a_m|s_m)}) \cdot sign(A^{m}(s_m, a_m)}}{ e^{\theta_{m}(s_m, a_m)+ln(\frac{1}{\pi_m(a_m|s_m)}) \cdot sign(A^{m}(s_m, a_m)} + \sum\limits_{a \ne a_m } e^{\theta_{m}(s_m, a)} } - \pi_m(a_m|s_m) &&\\
&= \frac{e^{\theta_{m}(s_m, a_m)+ln(\frac{\sum_a e^{\theta_m}(a)}{e^{\theta_m}(a_m)})}}{ e^{\theta_{m}(s_m, a_m)+ln(\frac{\sum_a e^{\theta_m}(a)}{e^{\theta_m}(s_m, a_m)})} + \sum\limits_{a \ne a_m } e^{\theta_{m}(s_m, a)} } - \pi_m(a_m|s_m)&&\\
&= \frac{ \frac{e^{\theta_m}(s_m, a_m)}{\pi_m(a_m|s_m)} }{ \frac{e^{\theta_m}(s_m, a_m)}{\pi_m(a_m|s_m)} + \sum\limits_{a \ne a_m } e^{\theta_{m}(s_m, a)} } - \pi_m(a_m|s_m)&&\\
&= \frac{ \frac{e^{\theta_m}(s_m, a_m)}{\pi_m(a_m|s_m)} }{ \frac{e^{\theta_m}(s_m, a_m)}{\pi_m(a_m|s_m)} + (\frac{1}{\pi_m(a_m|s_m)}-1)  \cdot e^{\theta_m}(s_m, a_m) } - \pi_m(a_m|s_m)&&\\
&= \frac{ \frac{1}{\pi_m(a_m|s_m)} }{ \frac{2}{\pi_m(a_m|s_m)}-1} - \pi_m(a_m|s_m)&&\\
&= \frac{(1-\pi_{m}(a_m|s_m))^2}{2-\pi_{m}(a_m|s_m)}&&
\end{align}
\endgroup

For $s = s_m, a \ne a_m$:
\begin{align}
\pi_{m+1}(a|s_m) - \pi_{m}(a|s_m)
&= \frac{e^{\theta_{m+1}(s_m, a)}}{ \sum\limits_{a \in \mathcal{A} } e^{\theta_{m+1}(s_m, a)} } - \pi_m(a|s_m) &&\\
&= \frac{e^{\theta_{m}(s_m, a)}}{ e^{\theta_{m}(s_m, a_m)+ln(\frac{1}{\pi_m(a_m|s_m)}) \cdot sign(A^{m}(s_m, a_m)} + \sum\limits_{a \ne a_m } e^{\theta_{m}(s_m, a)} } - \pi_m(a|s_m) &&\\
&= \frac{e^{\theta_{m}(s_m, a)}}{ e^{\theta_{m}(s_m, a_m)+ln(\frac{\sum_a e^{\theta_m}(a)}{e^{\theta_m}(s_m, a_m)})} + \sum\limits_{a \ne a_m } e^{\theta_{m}(s_m, a)} } - \pi_m(a|s_m)&&\\
&= \frac{e^{\theta_{m}(s_m, a)} }{ \frac{e^{\theta_m}(s_m, a_m)}{\pi_m(a_m|s_m)} + \sum\limits_{a \ne a_m } e^{\theta_{m}(s_m, a)} } - \pi_m(a|s_m)&&\\
&= \frac{e^{\theta_{m}(s_m, a)} }{ \frac{e^{\theta_m}(s_m, a_m)}{\pi_m(a_m|s_m)} + (\frac{1}{\pi_m(a_m|s_m)}-1)  \cdot e^{\theta_m}(s_m, a_m) } - \pi_m(a|s_m)&&\\
&= \left ( \frac{e^{\theta_{m}(s_m, a)} }{ (\frac{2}{\pi_m(a_m|s_m)}-1)  \cdot e^{\theta_m}(s_m, a_m) } {\div}  \pi_m(a|s_m) - 1  \right ) \cdot \pi_m(a|s_m) &&\\
&= \left ( \frac{ \frac{1}{\pi_m(a_m|s_m)} }{ (\frac{2}{\pi_m(a_m|s_m)}-1) } -1 \right ) \cdot \pi_m(a|s_m) &&\\
&= - \frac{1-\pi_{m}(a_m|s_m)}{2-\pi_{m}(a_m|s_m)} \cdot \pi_{m}(a|s)
\end{align}
\end{proof}
\end{lemma}

\begin{lemma} Under the CAPO update (\ref{eq:CAPO_form}) with ${\alpha_m(s, a) = \log (\frac{1}{\pi_{\theta_{m}}(a\rvert s)})}$, if $B_m = \left \{  (s_m, a_m) \right \} $ and $A^{m}(s_m, a_m) < 0$, then the policy weight difference $\pi_{m+1}(a|s) - \pi_{m}(a|s)$ can be written as :
\label{lemma:change_of_policy_weight_2}
\begin{align}
\pi_{m+1}(a|s) - \pi_{m}(a|s) =
\begin{cases}
\frac{-\pi_{m}(a_m|s_m) \cdot (1-\pi_{m}(a_m|s_m))^2}{\pi_{m}(a_m|s_m)^2 - \pi_{m}(a_m|s_m) + 1}  & \text{, if } s = s_m, a = a_m \\
\frac{\pi_{m}(a_m|s_m) \cdot (1-\pi_{m}(a_m|s_m))}{\pi_{m}(a_m|s_m)^2 - \pi_{m}(a_m|s_m) + 1} \cdot \pi_{m}(a|s)  & \text{, if } s = s_m, a \ne a_m\\
0 & \text{, else } 
\end{cases}
\end{align}

\begin{proof}[Proof of \Cref{lemma:change_of_policy_weight_2}]  
\phantom{}

For $s = s_m, a = a_m$ :
\begingroup
\allowdisplaybreaks
\begin{align}
\pi_{m+1}(a_m|s_m) - \pi_{m}(a_m|s_m)
&= \frac{e^{\theta_{m+1}(s_m, a_m)}}{ \sum\limits_{a \in \mathcal{A} } e^{\theta_{m+1}(s_m, a)} } - \pi_m(a_m|s_m) &&\\
&= \frac{e^{\theta_{m}(s_m, a_m)+ln(\frac{1}{\pi_m(a_m|s_m)}) \cdot sign(A^{m}(s_m, a_m)}}{ e^{\theta_{m}(s_m, a_m)+ln(\frac{1}{\pi_m(a_m|s_m)}) \cdot sign(A^{m}(s_m, a_m)} + \sum\limits_{a \ne a_m } e^{\theta_{m}(s_m, a)} } - \pi_m(a_m|s_m) &&\\
&= \frac{e^{\theta_{m}(s_m, a_m)-ln(\frac{\sum_a e^{\theta_m}(a)}{e^{\theta_m}(a_m)})}}{ e^{\theta_{m}(s_m, a_m)-ln(\frac{\sum_a e^{\theta_m}(a)}{e^{\theta_m}(s_m, a_m)})} + \sum\limits_{a \ne a_m } e^{\theta_{m}(s_m, a)} } - \pi_m(a_m|s_m)&&\\
&= \frac{ e^{\theta_m}(s_m, a_m) \cdot \pi_m(a_m|s_m) }{ e^{\theta_m}(s_m, a_m) \cdot \pi_m(a_m|s_m) + \sum\limits_{a \ne a_m } e^{\theta_{m}(s_m, a)} } - \pi_m(a_m|s_m)&&\\
&= \frac{ e^{\theta_m}(s_m, a_m) \cdot \pi_m(a_m|s_m) }{ e^{\theta_m}(s_m, a_m) \cdot \pi_m(a_m|s_m) + (\frac{1}{\pi_m(a_m|s_m)}-1)  \cdot e^{\theta_m}(s_m, a_m) } - \pi_m(a_m|s_m)&&\\
&= \frac{ \pi_m(a_m|s_m) }{ \pi_m(a_m|s_m) - 1 +  \frac{1}{\pi_m(a_m|s_m)}} - \pi_m(a_m|s_m)&&\\
&= \frac{-\pi_{m}(a_m|s_m) \cdot (1-\pi_{m}(a_m|s_m))^2}{\pi_{m}(a_m|s_m)^2 - \pi_{m}(a_m|s_m) + 1}&&
\end{align}
\endgroup

For $s = s_m, a \ne a_m$ :
\begin{align}
\pi_{m+1}(a|s_m) - \pi_{m}(a|s_m)
&= \frac{e^{\theta_{m+1}(s_m, a)}}{ \sum\limits_{a \in \mathcal{A} } e^{\theta_{m+1}(s_m, a)} } - \pi_m(a|s_m) &&\\
&= \frac{e^{\theta_{m}(s_m, a)}}{ e^{\theta_{m}(s_m, a_m)+ln(\frac{1}{\pi_m(a_m|s_m)}) \cdot sign(A^{m}(s_m, a_m)} + \sum\limits_{a \ne a_m } e^{\theta_{m}(s_m, a)} } - \pi_m(a|s_m) &&\\
&= \frac{e^{\theta_{m}(s_m, a)}}{ e^{\theta_{m}(s_m, a_m)-ln(\frac{\sum_a e^{\theta_m}(a)}{e^{\theta_m}(s_m, a_m)})} + \sum\limits_{a \ne a_m } e^{\theta_{m}(s_m, a)} } - \pi_m(a|s_m)&&\\
&= \frac{e^{\theta_{m}(s_m, a)} }{ \pi_m(a_m|s_m) \cdot e^{\theta_m}(s_m, a_m) + \sum\limits_{a \ne a_m } e^{\theta_{m}(s_m, a)} } - \pi_m(a|s_m)&&\\
&= \frac{e^{\theta_{m}(s_m, a)} }{ \pi_m(a_m|s_m) \cdot e^{\theta_m}(s_m, a_m) + (\frac{1}{\pi_m(a_m|s_m)}-1)  \cdot e^{\theta_m}(s_m, a_m) } - \pi_m(a|s_m)&&\\
&= \left ( \frac{e^{\theta_{m}(s_m, a)} }{ ( \pi_m(a_m|s_m) - 1 + \frac{1}{\pi_m(a_m|s_m)})  \cdot e^{\theta_m}(s_m, a_m) } {\div}  \pi_m(a|s_m) - 1  \right ) \cdot \pi_m(a|s_m) &&\\
&= \left ( \frac{ \frac{1}{\pi_m(a_m|s_m)} }{ ( \pi_m(a_m|s_m) - 1 + \frac{1}{\pi_m(a_m|s_m)}) } -1 \right ) \cdot \pi_m(a|s_m) &&\\
&= \frac{\pi_{m}(a_m|s_m) \cdot (1-\pi_{m}(a_m|s_m))}{\pi_{m}(a_m|s_m)^2 - \pi_{m}(a_m|s_m) + 1} \cdot \pi_{m}(a|s)&&
\end{align}
\end{proof}
\end{lemma}

\begin{lemma} Under the CAPO update (\ref{eq:CAPO_form}) with ${\alpha_m(s, a) \ge \log (\frac{1}{\pi_{\theta_m}(a\rvert s)})}$, if $B_m = \left \{  (s_m, a_m) \right \} $ and $A^{m}(s_m, a_m) > 0$, then the policy weight difference $\pi_{m+1}(a|s) - \pi_{m}(a|s)$ can be written as :
\label{lemma:policy_weight_distribution}
\begin{align}
\pi_{m+1}(a|s) - \pi_{m}(a|s) =
\begin{cases}
W^+  & \text{, if } s = s_m, a = a_m \\
- W^+ \cdot \frac{\pi_{m}(a|s)}{1-\pi_{m}(a_m|s_m)} & \text{, if } s = s_m, a \ne a_m\\
0 & \text{, else } 
\end{cases}
\\ \text{where $(1-\pi_{m}(a_m|s_m)) \ge W^+ \ge \frac{(1-\pi_{m}(a_m|s_m))^2}{2-\pi_{m}(a_m|s_m)}$}
\end{align}

\begin{proof}[Proof of \Cref{lemma:policy_weight_distribution}]  
\phantom{}

By \Cref{lemma:change_of_policy_weight_1}, we have $W^+ = \frac{(1-\pi_{m}(a_m|s_m))^2}{2-\pi_{m}(a_m|s_m)}$ under ${\alpha_m(s, a) = \log (\frac{1}{\pi_{\theta_m}(a\rvert s)})}$. 
Since $\pi_{m+1}(a|s)$ is proportional to the learning rate $\alpha_m(s,a)$, we establish the lower bound of $W^+$ directly. 
The upper bound of $W^+$ is constructed by the maximum value of improvement.\\
Also, for $s = s_m, a \ne a_m$, we have:
\begin{align}
\frac{\pi_{m+1}(a|s)}{\pi_{m}(a|s)} = \frac{ \frac{e^{\theta_{m}(s,a)}}{Z_{m}(s)} }{\frac{e^{\theta_{m+1}(s,a)}}{Z_{m+1}(s)}} = \frac{ \frac{e^{\theta_{m}(s,a)}}{Z_{m}(s)} }{\frac{e^{\theta_{m}(s,a)}}{Z_{m+1}(s)}} = \frac{Z_{m}(s)}{Z_{m+1}(s)}
\end{align}
Since $\sum_{a \ne a_m} \left ( \pi_{m+1}(a|s) - \pi_{m}(a|s) \right ) = -W^+$, we have:
\begin{align}
\sum_{a \ne a_m} \left ( \pi_{m+1}(a|s) - \pi_{m}(a|s) \right ) = \sum_{a \ne a_m} \left ( \frac{Z_{m}(s)}{Z_{m+1}(s)}-1 \right )  \cdot \pi_{m}(a|s) = \left ( \frac{Z_{m}(s)}{Z_{m+1}(s)}-1 \right ) \cdot (1-\pi_{m}(a_m|s)) = -W^+
\end{align}
Hence, for $s = s_m, a \ne a_m$, we get:
\begin{align}
\pi_{m+1}(a|s) - \pi_{m}(a|s) = \frac{Z_{m}(s) \cdot \pi_m{(a|s)}}{Z_{m+1}(s)}  - \pi_{m}(a|s) = \left ( \frac{Z_{m}(s)}{Z_{m+1}(s)}-1 \right ) \cdot \pi_{m}(a|s) = \frac{-W^+}{1-\pi_{m}(a_m|s)} \cdot \pi_{m}(a|s)
\end{align}
\end{proof}
\end{lemma}

\begin{lemma} Under the CAPO update (\ref{eq:CAPO_form}) with ${\alpha_m(s, a) \ge \log (\frac{1}{\pi_{\theta_m}(a\rvert s)})}$, if $B_m = \left \{  (s_m, a_m) \right \} $ and $A^{m}(s_m, a_m) < 0$, then the policy weight difference $\pi_{m+1}(a|s) - \pi_{m}(a|s)$ can be written as :
\label{lemma:policy_weight_distribution_2}
\begin{align*}
\pi_{m+1}(a|s) - \pi_{m}(a|s) =
\begin{cases}
-W^-  & \text{, if } s = s_m, a = a_m \\
W^- \cdot \frac{\pi_{m}(a|s)}{1-\pi_{m}(a_m|s_m)} & \text{, if } s = s_m, a \ne a_m\\
0 & \text{, else } 
\end{cases}
\\ \text{where $\pi_m(a_m|s_m) \ge W^- \ge \frac{\pi_{m}(a_m|s_m) \cdot (1-\pi_{m}(a_m|s_m))^2}{\pi_{m}(a_m|s_m)^2 - \pi_{m}(a_m|s_m) + 1} $}
\end{align*}

\begin{proof}[Proof of \Cref{lemma:policy_weight_distribution_2}]  
\phantom{}

By \Cref{lemma:change_of_policy_weight_2}, we have $W^- = \frac{\pi_{m}(a_m|s_m) \cdot (1-\pi_{m}(a_m|s_m))^2}{\pi_{m}(a_m|s_m)^2 - \pi_{m}(a_m|s_m) + 1}$ under ${\alpha_m(s, a) = \log (\frac{1}{\pi_{\theta_m}(a\rvert s)})}$. 
Since $\pi_{m+1}(a|s)$ is proportional to the learning rate $\alpha_m(s,a)$, we establish the lower bound of $W^-$ directly. 
The upper bound of $W^-$ is constructed by the maximum value of improvement.\\
Also, for $s = s_m, a \ne a_m$, we have:
\begin{align}
\frac{\pi_{m+1}(a|s)}{\pi_{m}(a|s)} = \frac{ \frac{e^{\theta_{m}(s,a)}}{Z_{m}(s)} }{\frac{e^{\theta_{m+1}(s,a)}}{Z_{m+1}(s)}} = \frac{ \frac{e^{\theta_{m}(s,a)}}{Z_{m}(s)} }{\frac{e^{\theta_{m}(s,a)}}{Z_{m+1}(s)}} = \frac{Z_{m}(s)}{Z_{m+1}(s)}
\end{align}
Moreover, since $\sum_{a \ne a_m} \left ( \pi_{m+1}(a|s) - \pi_{m}(a|s) \right ) = W^-$, we have:
\begin{align}
\sum_{a \ne a_m} \left ( \pi_{m+1}(a|s) - \pi_{m}(a|s) \right ) = \sum_{a \ne a_m} \left ( \frac{Z_{m}(s)}{Z_{m+1}(s)}-1 \right )  \cdot \pi_{m}(a|s) = \left ( \frac{Z_{m}(s)}{Z_{m+1}(s)}-1 \right ) \cdot (1-\pi_{m}(a_m|s)) = W^-
\end{align}
Hence, for $s = s_m, a \ne a_m$, we get:
\begin{align}
\pi_{m+1}(a|s) - \pi_{m}(a|s) = \frac{Z_{m}(s) \cdot \pi_m{(a|s)}}{Z_{m+1}(s)}  - \pi_{m}(a|s) = \left ( \frac{Z_{m}(s)}{Z_{m+1}(s)}-1 \right ) \cdot \pi_{m}(a|s) = \frac{W^-}{1-\pi_{m}(a_m|s)} \cdot \pi_{m}(a|s)
\end{align}
\end{proof}
\end{lemma}

%\begin{prop} Updating $\theta_m$ with (\ref{eq:CAPO_form}), if $B_m = \left \{  (s_m, a_m) \right \} $ then $\pi_{m+1}(s,a) - \pi_{m}(s,a) = 0, \forall s \ne s_m, a \in \mathcal{A}$.
%\label{prop:pi_s_independent}
%\end{prop}

\begin{lemma} Under the CAPO update (\ref{eq:CAPO_form}) with ${\alpha_m(s, a) \ge \log (\frac{1}{\pi_{\theta_m}(a\rvert s)})}$, if $B_m = \left \{  (s_m, a_m) \right \} $ then the improvement of the performance $V^{\pi_{m+1}}(s) - V^{\pi_m}(s)$ can be written as :
\label{lemma:lower_bdd}
\begin{align}
V^{\pi_{m+1}}(s) - V^{\pi_{m}}(s) =
\begin{cases}
\frac{d^{\pi_{m+1}}_{s}(s_m)}{1-\gamma } \cdot \frac{W^+}{1-\pi_{m}(a_m|s_m)} \cdot A^{m}(s_m, a_m)  & \text{, if } A^{m}(s_m, a_m) > 0 \\
\frac{d^{\pi_{m+1}}_{s}(s_m)}{1-\gamma } \cdot \frac{W^-}{1-\pi_{m}(a_m|s_m)} \cdot (-A^{m}(s_m, a_m))  & \text{, if } A^{m}(s_m, a_m) <  0\\
\end{cases}\\
\text{where }
\begin{cases}
(1-\pi_{m}(a_m|s_m)) \ge W^+ \ge \frac{(1-\pi_{m}(a_m|s_m))^2}{2-\pi_{m}(a_m|s_m)}\\
\pi_m(a_m|s_m) \ge W^- \ge \frac{\pi_{m}(a_m|s_m) \cdot (1-\pi_{m}(a_m|s_m))^2}{\pi_{m}(a_m|s_m)^2 - \pi_{m}(a_m|s_m) + 1}
\end{cases}
\end{align}
and it can also be lower bounded by :
\begin{align}
V^{\pi_{m+1}}(s) - V^{\pi_{m}}(s) \ge
\begin{cases}
\frac{d^{\pi_{m+1}}_{s}(s_m)}{2} \cdot A^{m}(s_m, a_m)^2  & \text{, if } A^{m}(s_m, a_m) > 0 \\
d^{\pi_{m+1}}_{s}(s_m) \cdot \pi_m(a_m|s_m) \cdot A^{m}(s_m, a_m)^2  & \text{, if } A^{m}(s_m, a_m) <  0\\
\end{cases}
\end{align}
\end{lemma}

\begin{proof}[Proof of \Cref{lemma:lower_bdd}]  
\phantom{}

\begin{comment}
\begin{prop} Updating $\theta_m$ with (\ref{eq:CAPO_form}), if $B_m = \left \{  (s_m, a_m) \right \} $ then $\pi_{m+1}(s,a) - \pi_{m}(s,a) = 0, \forall s \ne s_m, a \in \mathcal{A}$.
\label{prop:pi_s_independent}
\end{prop}
\end{comment}

If $A^{m}(s, a) > 0$, then :
\begin{align}
V^{\pi_{m+1}}(s) - V^{\pi_{m}}(s)
&= \frac{1}{1-\gamma} \cdot \sum_{s \in \mathcal{S}} d^{\pi_{m+1}}_{s}(s) \sum_{a \in \mathcal{A}} \pi_{m+1}(a|s) \cdot A^{m}(s, a) &&\\
&= \frac{1}{1-\gamma} \sum_{s \in \mathcal{S}} d^{\pi_{m+1}}_{s}(s) \sum_{a \in \mathcal{A}} \left ( \pi_{m+1}(a|s) - \pi_{m}(a|s) \right ) \cdot A^{m}(s, a) &&\\
&= \frac{d^{\pi_{m+1}}_{s}(s_m)}{1-\gamma}  \cdot \sum_{a \in \mathcal{A}} \left ( \pi_{m+1}(a|s_m) - \pi_{m}(a|s_m) \right ) \cdot A^{m}(s_m, a) &&\\
&= \frac{d^{\pi_{m+1}}_{s}(s_m)}{1-\gamma} \cdot \left [ W^+ \cdot A^{m}(s_m, a_m) - \sum_{a \ne a_m} \frac{W^+}{1-\pi_{m}(a_m|s_m)} \cdot \pi_{m}(a|s_m) \cdot A^{m}(s_m, a) \right ]  &&\\
&= \frac{d^{\pi_{m+1}}_{s}(s_m)}{1-\gamma} \cdot \left [ W^+ \cdot A^{m}(s_m, a_m) - \frac{W^+}{1-\pi_{m}(a_m|s_m)} \cdot  \sum_{a \ne a_m} \pi_{m}(a|s_m) \cdot A^{m}(s_m, a) \right ]  &&\\
&= \frac{d^{\pi_{m+1}}_{s}(s_m)}{1-\gamma} \cdot \left [ W^+ \cdot A^{m}(s_m, a_m) + \frac{W^+}{1-\pi_{m}(a_m|s_m)} \cdot  \pi_{m}(a_m|s_m) \cdot A^{m}(s_m, a_m) \right ]  &&\\
&= \frac{d^{\pi_{m+1}}_{s}(s_m)}{1-\gamma} \cdot \frac{W^+}{1-\pi_{m}(a_m|s_m)} \cdot A^{m}(s_m, a_m) &&\\
&\ge \frac{d^{\pi_{m+1}}_{s}(s_m)}{2} \cdot A^{m}(s_m, a_m)^2  &&
\end{align}

The first equation holds by the performance difference lemma in \Cref{lemma:perf_diff}.\\
The second equation holds by the definition of $A(s,a)$.\\ %\ref{prop:pi_s_Adv}
The third equation holds since $\pi_{m+1}(a|s) = \pi_{m}(a|s)$, $\quad \forall s \ne s_m$.\\
The fourth equation holds by the difference of the updated policy weight that we have shown in \Cref{lemma:change_of_policy_weight_1} and \Cref{lemma:policy_weight_distribution}.\\
The last inequality holds by the bound of $A(s,a)$ in \Cref{lemma:upper_bound_of_advantage}.

If $A^{m}(s, a) < 0$, then :
\begingroup
\allowdisplaybreaks
\begin{align}
V^{\pi_{m+1}}(s) - V^{\pi_{m}}(s)
&= \frac{1}{1-\gamma} \cdot \sum_{s \in \mathcal{S}} d^{\pi_{m+1}}_{s}(s) \sum_{a \in \mathcal{A}} \pi_{m+1}(a|s) \cdot A^{m}(s, a) &&\\
&= \frac{1}{1-\gamma} \sum_{s \in \mathcal{S}} d^{\pi_{m+1}}_{s}(s) \sum_{a \in \mathcal{A}} \left ( \pi_{m+1}(a|s) - \pi_{m}(a|s) \right ) \cdot A^{m}(s, a) &&\\
&= \frac{d^{\pi_{m+1}}_{s}(s_m)}{1-\gamma}  \cdot \sum_{a \in \mathcal{A}} \left ( \pi_{m+1}(a|s_m) - \pi_{m}(a|s_m) \right ) \cdot A^{m}(s_m, a) &&\\
&= \frac{d^{\pi_{m+1}}_{s}(s_m)}{1-\gamma} \cdot \left [ -W^- \cdot A^{m}(s_m, a_m) + \sum_{a \ne a_m} \frac{W^-}{1-\pi_{m}(a_m|s_m)} \cdot \pi_{m}(a|s_m) \cdot A^{m}(s_m, a) \right ]  &&\\
&= \frac{d^{\pi_{m+1}}_{s}(s_m)}{1-\gamma} \cdot \left [ -W^- \cdot A^{m}(s_m, a_m) +  \frac{W^-}{1-\pi_{m}(a_m|s_m)} \cdot \sum_{a \ne a_m} \pi_{m}(a|s_m) \cdot A^{m}(s_m, a) \right ]  &&\\
&= \frac{d^{\pi_{m+1}}_{s}(s_m)}{1-\gamma} \cdot \left [ -W^- \cdot A^{m}(s_m, a_m) -  \frac{W^-}{1-\pi_{m}(a_m|s_m)} \cdot \pi_{m}(a_m|s_m) \cdot A^{m}(s_m, a_m) \right ]  &&\\
&= \frac{d^{\pi_{m+1}}_{s}(s_m)}{1-\gamma} \cdot \frac{W^-}{1-\pi_{m}(a_m|s_m)} \cdot (-A^{m}(s_m, a_m)) &&\\
&\ge d^{\pi_{m+1}}_{s}(s_m) \cdot \pi_{m}(a_m|s_m) \cdot A^{m}(s_m, a_m)^2  &&
\end{align}
\endgroup

The first equation holds by the performance difference lemma in \Cref{lemma:perf_diff}.\\
The second equation holds by the definition of $A(s,a)$.\\ %\ref{prop:pi_s_Adv}
The third equation holds since $\pi_{m+1}(a|s) = \pi_{m}(a|s)$, $\quad \forall s \ne s_m$.\\
The fourth equation holds by the difference of the updated policy weight that we have shown in \Cref{lemma:change_of_policy_weight_2} and \Cref{lemma:policy_weight_distribution_2}.\\
The last inequality holds by the bound of $A(s,a)$ in \Cref{lemma:upper_bound_of_advantage}.

\end{proof}

\subsection{Convergence Rate of Cyclic CAPO}
\label{subsection:cyclic_CAPO}
For ease of exposition, we restate \Cref{theorem:cyclic_convergence_rate} as follows.
\begin{theorem*}
Consider a tabular softmax parameterized policy $\pi_\theta$.
Under Cyclic CAPO with ${\alpha_m(s, a) \ge \log (\frac{1}{\pi_{\theta_m}(a\rvert s)})}$ and $|B_m| = 1$, $\bigcup_{i=1}^{|\mathcal{S}||\mathcal{A}|} B_{m \cdot |\mathcal{S}||\mathcal{A}| + i} =  \mathcal{S} \times \mathcal{A}$, we have :
\begin{align}
V^{*}(\rho) - V^{\pi_m}(\rho) \le \frac{|\mathcal{S}||\mathcal{A}|}{c} \cdot \frac{1}{m} , \quad \text{for all $m \ge 1$}
\end{align}
\begin{align}
\sum_{m=1}^{M}  V^{*}(\rho) - V^{\pi_m}(\rho) \le |\mathcal{S}||\mathcal{A}| \cdot \min{ \left \{ \sqrt{\frac{M}{c \cdot (1-\gamma)}}, \frac{\log M +1}{c} \right \}  }, \quad \text{for all $m \ge 1$}
\end{align}
where $c = \frac{(1-\gamma)^4}{2} \cdot \left \| \frac{1}{\mu} \right \|_{\infty}^{-1} \cdot {\min} \left \{ \frac{\min_s{\mu(s)} }{2} , \frac{ (1-\gamma)}{|\mathcal{S}||\mathcal{A}|} \right \} > 0$.
\label{app:cyclic_convergence_rate}
\end{theorem*}

\begin{proof}[Proof of \Cref{theorem:cyclic_convergence_rate}]  
\phantom{}

The proof can be summarized as:
\begin{enumerate}
\itemsep0em
  \item  We first write the improvement of the performance $V^{\pi_{m+1}}(s) - V^{\pi_m}(s)$ in state visitation distribution, policy weight, and advantage value in \Cref{lemma:lower_bdd}, and also construct the lower bound of it.
  \item We then construct the upper bound of the performance difference $ V^{*}(s) - V^{\pi_m}(s)$ using $V^{\pi_{m+ |\mathcal{S}||\mathcal{A}|}}(s) - V^{\pi_{m}}(s)$.
  \item Finally, we can show the desired result inductively by \Cref{lemma:induction}.
\end{enumerate}

By \Cref{lemma:lower_bdd}, we have for all $m \ge 1$:
\begin{align}
V^{\pi_{m+1}}(s) - V^{\pi_{m}}(s) =
\begin{cases}
\frac{d^{\pi_{m+1}}_{s}(s_m)}{1-\gamma } \cdot \frac{W^+}{1-\pi_{m}(a_m|s_m)} \cdot A^{m}(s_m, a_m)  & \text{, if } A^{m}(s_m, a_m) > 0 \\
\frac{d^{\pi_{m+1}}_{s}(s_m)}{1-\gamma } \cdot \frac{W^-}{1-\pi_{m}(a_m|s_m)} \cdot (-A^{m}(s_m, a_m))  & \text{, if } A^{m}(s_m, a_m) <  0\\
\end{cases}\\
\text{where }
\begin{cases}
(1-\pi_{m}(a_m|s_m)) \ge W^+ \ge \frac{(1-\pi_{m}(a_m|s_m))^2}{2-\pi_{m}(a_m|s_m)}\\
\pi_m(a_m|s_m) \ge W^- \ge \frac{\pi_{m}(a_m|s_m) \cdot (1-\pi_{m}(a_m|s_m))^2}{\pi_{m}(a_m|s_m)^2 - \pi_{m}(a_m|s_m) + 1}
\end{cases}
\end{align}

and it can also be lower bounded by:
\begin{align}
V^{\pi_{m+1}}(s) - V^{\pi_{m}}(s) \ge
\begin{cases}
\frac{d^{\pi_{m+1}}_{s}(s_m)}{2} \cdot A^{m}(s_m, a_m)^2  & \text{, if } A^{m}(s_m, a_m) > 0 \\
d^{\pi_{m+1}}_{s}(s_m) \cdot \pi_m(a_m|s_m) \cdot A^{m}(s_m, a_m)^2  & \text{, if } A^{m}(s_m, a_m) <  0
\end{cases}
\end{align}

Now, we're going to construct the upper bound of the performance difference $ V^{*}(s) - V^{\pi_m}(s)$ using $V^{\pi_{m+|\mathcal{S}||\mathcal{A}|}}(s) - V^{\pi_{m}}(s)$. Note that by \Cref{lemma:lower_bound_of_performance_difference}, there exists $(\tilde{s_m},\tilde{a_m})$ such that $\left ( V^{*}(s) - V^{\pi_m}(s) \right )^2 \le \left ( \frac{1}{1-\gamma} \cdot A^{m}(\tilde{s_m},\tilde{a_m}) \right )^2$ for all $m \ge 1$.\\

Hence, if we construct the upper bound of $A^{m}(\tilde{s_m},\tilde{a_m})^2$ using $V^{\pi_{m+|\mathcal{S}||\mathcal{A}|}}(s) - V^{\pi_{m}}(s)$, which is the improvement of the performance during the whole cycle, then we can get the the upper bound of the performance difference $ V^{*}(s) - V^{\pi_m}(s)$ using $V^{\pi_{m+|\mathcal{S}||\mathcal{A}|}}(s) - V^{\pi_{m}}(s)$ for all $m \equiv 0 \pmod{|\mathcal{S}||\mathcal{A}|} $.\\

Without loss of generality, Assume we update $(\tilde{s_m},\tilde{a_m})$ at episode $(m+T)$, where $T \in \left [ 0,|\mathcal{S}||\mathcal{A}| \right ) \bigcap \mathbb{N}$, $m \equiv 0 \pmod{|\mathcal{S}||\mathcal{A}|} $. We discuss two possible cases as follows: 

\begin{itemize}[leftmargin=*]
%\exists \: T < \tilde{T} \ni
    \item\label{item:case1} Case 1: $ V^{\pi_{m+T}}(s) - V^{\pi_{m}}(s) \ge A^{m}(\tilde{s_m},\tilde{a_m})$:\\
    \begingroup
    \allowdisplaybreaks
    \begin{align}
    A^{m}(\tilde{s_m},\tilde{a_m})^2
    &\le \left ( V^{\pi_{m+T}}(s) - V^{\pi_{m}}(s) \right ) ^2 &&\\
    &= \left ( \sum_{k=m}^{m+T-1} \left ( V^{\pi_{k+1}}(s) - V^{\pi_{k}}(s) \right )  \right ) ^2 &&\\
    &= \left ( \sum\limits_{\substack{k \in [m,m+T-1] \\ A^k(s_k,a_k)>0}} (V^{\pi_{k+1}}(s) - V^{\pi_{k}}(s)) + \sum\limits_{\substack{k \in [m,m+T-1] \\ A^k(s_k,a_k) < 0}} (V^{\pi_{k+1}}(s) - V^{\pi_{k}}(s)) \right ) ^2 &&\\
    &= \left( \sum\limits_{\substack{k \in [m,m+T-1] \\ A^k(s_k,a_k)>0}} \frac{d^{\pi_{k+1}}_{s}(s_k)}{1-\gamma } \cdot  \frac{W^{k+}}{1-\pi_k(a_k|s_k)} \cdot A^{k}(s_k, a_k) \right. \\
    &\left. \quad \quad +  \sum\limits_{\substack{k \in [m,m+T-1] \\ A^k(s_k,a_k) < 0}} \frac{d^{\pi_{k+1}}_{s}(s_k)}{1-\gamma } \cdot \frac{W^{k-}}{1-\pi_k(a_k|s_k)} \cdot (-A^{k}(s_k, a_k))  \right)^2 &&\\
    &\le T \cdot \left( \sum\limits_{\substack{k \in [m,m+T-1] \\ A^k(s_k,a_k)>0}} \left( \frac{d^{\pi_{k+1}}_{s}(s_k)}{1-\gamma } \cdot  \frac{W^{k+}}{1-\pi_k(a_k|s_k)} \cdot A^{k}(s_k, a_k) \right)^2 \right. \\
    &\left. \quad \quad \enspace + \sum\limits_{\substack{k \in [m,m+T-1] \\ A^k(s_k,a_k)<0}} \left( \frac{d^{\pi_{k+1}}_{s}(s_k)}{1-\gamma } \cdot \frac{W^{k-}}{1-\pi_k(a_k|s_k)} \cdot A^{k}(s_k, a_k) \right)^2  \right) &&\\
    &= T \cdot \left( \sum\limits_{\substack{k \in [m,m+T-1] \\ A^k(s_k,a_k)>0}} \left( \frac{d^{\pi_{k+1}}_{s}(s_k)}{1-\gamma } \cdot  \frac{W^{k+}}{1-\pi_k(a_k|s_k)} \right)^2 \cdot A^{k}(s_k, a_k)^2 \right. \\
    &\left. \quad \quad \enspace + \sum\limits_{\substack{k \in [m,m+T-1] \\ A^k(s_k,a_k)<0}} \left( \frac{d^{\pi_{k+1}}_{s}(s_k)}{1-\gamma } \cdot \frac{W^{k-}}{1-\pi_k(a_k|s_k)} \right)^2 \cdot A^{k}(s_k, a_k)^2  \right) &&\\
    &= T \cdot \left( \sum\limits_{\substack{k \in [m,m+T-1] \\ A^k(s_k,a_k)>0}} \left( \frac{d^{\pi_{k+1}}_{s}(s_k) \cdot W^{k+}}{1-\gamma } \right)^2 \cdot \frac{|A^{k}(s_k, a_k)|}{1-\pi_k(a_k|s_k)}  \cdot \frac{|A^{k}(s_k, a_k)|}{1-\pi_k(a_k|s_k)} \right.&&\\
    &\left. \quad\quad\quad + \sum\limits_{\substack{k \in [m,m+T-1] \\ A^k(s_k,a_k)<0}} \left( \frac{d^{\pi_{k+1}}_{s}(s_k) \cdot W^{k-}}{1-\gamma } \right)^2 \cdot \frac{|A^{k}(s_k, a_k)|}{1-\pi_k(a_k|s_k)}  \cdot \frac{|A^{k}(s_k, a_k)|}{1-\pi_k(a_k|s_k)} \right) &&\\
    &\le T \cdot \left( \sum\limits_{\substack{k \in [m,m+T-1] \\ A^k(s_k,a_k)>0}} \left( \frac{d^{\pi_{k+1}}_{s}(s_k) \cdot W^{k+}}{1-\gamma } \right)^2 \cdot \frac{1}{1-\gamma} \cdot \frac{1-\gamma }{d^{\pi_{k+1}}_{s}(s_k) \cdot W^{k+}} \cdot \left( V^{\pi_{k+1}}(s) - V^{\pi_{k}}(s) \right) \right.&&\\
    &\left. \quad\quad\quad + \sum\limits_{\substack{k \in [m,m+T-1] \\ A^k(s_k,a_k)<0}} \left( \frac{d^{\pi_{k+1}}_{s}(s_k) \cdot W^{k-}}{1-\gamma } \right)^2 \cdot \frac{1}{1-\gamma} \cdot \frac{1-\gamma }{d^{\pi_{k+1}}_{s}(s_k) \cdot W^{k-}} \cdot \left( V^{\pi_{k+1}}(s) - V^{\pi_{k}}(s) \right) \right) &&\\
    &= \frac{T}{(1-\gamma)^2} \cdot \left( \sum\limits_{\substack{k \in [m,m+T-1] \\ A^k(s_k,a_k)>0}} d^{\pi_{k+1}}_{s}(s_k) \cdot W^{k+} \cdot \left( V^{\pi_{k+1}}(s) - V^{\pi_{k}}(s) \right) \right. \\
    &\left. \qquad \qquad \qquad + \sum\limits_{\substack{k \in [m,m+T-1] \\ A^k(s_k,a_k)<0}} d^{\pi_{k+1}}_{s}(s_k) \cdot W^{k-} \cdot \left( V^{\pi_{k+1}}(s) - V^{\pi_{k}}(s) \right) \right) &&\\
    &\le \frac{T}{(1-\gamma)^2} \cdot \underset{k \in [m,m+T-1]}{\max} \left \{ \mathbbm{1} \left \{ A^k(s_k,a_k)>0 \right \} \cdot d^{\pi_{k+1}}_{s}(s_k) \cdot W^{k+} + \mathbbm{1} \left \{ A^k(s_k,a_k) < 0 \right \} \cdot d^{\pi_{k+1}}_{s}(s_k) \cdot W^{k-} \right \} \\
    & \quad \cdot \left ( \sum_{k=m}^{m+T-1} V^{\pi_{k+1}}(s) - V^{\pi_{k}}(s) \right ) &&\\
    &\le c_m \cdot \frac{T}{(1-\gamma)^2} \cdot \left ( V^{\pi_{m+T}}(s) - V^{\pi_{m}}(s) \right )&&\\
    &\le c_m \cdot \frac{T}{(1-\gamma)^2} \cdot \left ( V^{\pi_{m+T+1}}(s) - V^{\pi_{m}}(s) \right )&&\\
    &\le 2 \cdot {\max} \left \{ \frac{2}{d^{\pi_{m+T+1}}_{s}(s_{m+T})} , \frac{c_m \cdot T}{(1-\gamma)^2} \right \} \cdot \left (   V^{\pi_{m+T+1}}(s) - V^{\pi_{m}}(s) \right ) &&
    \end{align}
    \endgroup
    
    where $c_m = \underset{k \in [m,m+T-1]}{\max} \left \{ c_{k1}, c_{k2} \right \} \in [0,1]$ \\
    and $c_{k1} = \mathbbm{1} \left \{ A^k(s_k,a_k)>0 \right \} \cdot d^{\pi_{k+1}}_{s}(s_k) \cdot W^{k+}$, $c_{k2} = \mathbbm{1} \left \{ A^k(s_k,a_k) < 0 \right \} \cdot d^{\pi_{k+1}}_{s}(s_k) \cdot W^{k-}$.\\
    
    The third equation holds by \Cref{lemma:lower_bdd}.\\
    The second inequality holds by Cauchy-Schwarz.\\
    The third inequality holds by  \Cref{lemma:upper_bound_of_advantage} and \Cref{lemma:lower_bdd}.\\
    
    \item\label{item:case2} Case 2: $V^{\pi_{m+T}}(s) - V^{\pi_{m}}(s) < A^{m}(\tilde{s_m},\tilde{a_m})$:\\
    
    \begin{align}
    A^{m}(\tilde{s_m},\tilde{a_m})^2
    &= \left ( \left ( Q^{\pi_m}(\tilde{s_m},\tilde{a_m}) - V^{\pi_{m+T}}(s) \right ) +  \left ( V^{\pi_{m+T}}(s) - V^{\pi_{m}}(s) \right ) \right ) ^2 &&\\
    &\le \left ( \left ( Q^{\pi_{m+T}}(\tilde{s_m},\tilde{a_m}) - V^{\pi_{m+T}}(s) \right ) +  \left ( V^{\pi_{m+T}}(s) - V^{\pi_{m}}(s) \right ) \right ) ^2 &&\\
    &= \left (  A^{m+T}(\tilde{s_m},\tilde{a_m})  +  \left ( V^{\pi_{m+T}}(s) - V^{\pi_{m}}(s) \right ) \right ) ^2 &&\\
    &\le \left (  A^{m+T}(\tilde{s_m},\tilde{a_m} )^2  +  \left ( V^{\pi_{m+T}}(s) - V^{\pi_{m}}(s) \right )^2 \right ) \cdot (1^2 + 1^2) &&\\
    &\le 2 \cdot \left (  \frac{2}{d^{\pi_{m+T+1}}_{s}(s_{m+T})} \cdot \left ( V^{\pi_{m+T+1}}(s) - V^{\pi_{m+T}}(s) \right ) \right. \\
    &\left. \quad \quad \enspace +  c_m \cdot \frac{T}{(1-\gamma)^2} \cdot \left ( V^{\pi_{m+T}}(s) - V^{\pi_{m}}(s) \right ) \right )  &&\\
    &\le 2 \cdot {\max} \left \{ \frac{2}{d^{\pi_{m+T+1}}_{s}(s_{m+T})} , \frac{c_m \cdot T}{(1-\gamma)^2} \right \} \cdot \left (   V^{\pi_{m+T+1}}(s) - V^{\pi_{m}}(s) \right ) &&
    \end{align}
    
    The first inequality holds by the strict improvement of $V^{\pi}(s)$ \ref{app:proof_strict_improvement}, leading to the strict improvement of $Q^{\pi}(s,a)$.\\
    The second inequality holds by Cauchy-Schwarz.\\
    The third inequality holds by the result of Case 1 and \Cref{lemma:lower_bdd}
\end{itemize}

Hence, in both case we get: 
\begin{align}
&V^{\pi_{m+|\mathcal{S}||\mathcal{A}|}}(s) - V^{\pi_{m}}(s) \ge V^{\pi_{m+T+1}}(s) - V^{\pi_{m}}(s) \ge \frac{1}{2} \cdot \frac{1}{{\max} \left \{ \frac{2}{d^{\pi_{m+T+1}}_{s}(s_{m+T})} , \frac{c_m \cdot T}{(1-\gamma)^2} \right \}} \cdot A^{m}(\tilde{s_m},\tilde{a_m})^2
\end{align}
for all $m \equiv 0 \pmod{|\mathcal{S}||\mathcal{A}|} $.\\

Combining \Cref{lemma:lower_bound_of_performance_difference}, we can construct the upper bound of the performance difference $ V^{*}(s) - V^{\pi_m}(s)$ using $V^{\pi_{m+|\mathcal{S}||\mathcal{A}|}}(s) - V^{\pi_m}(s)$:
\begin{align}
V^{\pi_{m+|\mathcal{S}||\mathcal{A}|}}(s) - V^{\pi_{m}}(s)
&\ge \frac{(1-\gamma)^2}{2} \cdot \frac{1}{{\max} \left \{ \frac{2}{d^{\pi_{m+T+1}}_{s}(s_{m+T})} , \frac{c_m \cdot T}{(1-\gamma)^2} \right \}}  \cdot \left ( V^{*}(s) - V^{\pi_m}(s) \right )^2&&\\
&=\frac{(1-\gamma)^2}{2} \cdot {\min} \left \{ \frac{d^{\pi_{m+T+1}}_{s}(s_{m+T})}{2} , \frac{(1-\gamma)^2}{c_m \cdot T} \right \}  \cdot \left ( V^{*}(s) - V^{\pi_m}(s) \right )^2 &&
\end{align}

and if we consider the whole initial state distribution, $\mu$, we have :
\begin{align}
V^{\pi_{m+|\mathcal{S}||\mathcal{A}|}}(\mu) - V^{\pi_{m}}(\mu)
&\ge \frac{(1-\gamma)^2}{2} \cdot \frac{1}{{\max} \left \{ \frac{2}{d^{\pi_{m+T+1}}_{\mu}(s_{m+T})} , \frac{c_m \cdot T}{(1-\gamma)^2} \right \}}  \cdot \left ( V^{*}(\mu) - V^{\pi_m}(\mu) \right )^2&&\\
&=\frac{(1-\gamma)^2}{2} \cdot {\min} \left \{ \frac{d^{\pi_{m+T+1}}_{\mu}(s_{m+T})}{2} , \frac{(1-\gamma)^2}{c_m \cdot T} \right \}  \cdot \left ( V^{*}(s) - V^{\pi_m}(s) \right )^2 &&\\
&\ge\frac{(1-\gamma)^2}{2} \cdot {\min} \left \{ \frac{(1-\gamma) \cdot \min_s{\mu(s)} }{2} , \frac{ (1-\gamma)^2}{|\mathcal{S}||\mathcal{A}|} \right \}  \cdot \left ( V^{*}(\mu) - V^{\pi_m}(\mu) \right )^2&&\\
&\ge \underbrace{ \frac{(1-\gamma)^3}{2} \cdot {\min} \left \{ \frac{\min_s{\mu(s)} }{2} , \frac{ (1-\gamma)}{|\mathcal{S}||\mathcal{A}|} \right \} }_{:= c' > 0} \cdot \left ( V^{*}(\mu) - V^{\pi_m}(\mu) \right )^2 &&
\end{align}
The second inequality holds since $d^{\pi}_{\mu}(s) \ge (1- \gamma) \cdot \mu(s)$ \ref{lemma:lower_bound_of_state_visitation_distribution}.\\

And since $V^{\pi_{m+|\mathcal{S}||\mathcal{A}|}}(\mu) - V^{\pi_{m}}(\mu) = (V^{\pi^{*}}(\mu) - V^{\pi_{m}}(\mu)) - (V^{\pi^{*}}(\mu) - V^{\pi_{m+|\mathcal{S}||\mathcal{A}|}}(\mu))$, by rearranging the inequality above, we have :
\begin{align}
&\delta_{m+|\mathcal{S}||\mathcal{A}|} \le \delta_{m} - c' \cdot \delta_{m}^2 \quad \text{where $\delta_{m} = V^{\pi^*}(\mu) - V^{\pi_m}(\mu) $ for all $m \equiv 0 \pmod{|\mathcal{S}||\mathcal{A}|} $}
\end{align}

Then, we can get the following result by induction \ref{lemma:induction} :
\begin{align}
V^{*}(\mu) - V^{\pi_m}(\mu) \le \frac{1}{c'} \cdot \frac{1}{\max \left \{ \left \lfloor \frac{m}{|\mathcal{S}||\mathcal{A}|}  \right \rfloor, 1 \right \}   }   \le \frac{1}{c'} \cdot \min \left \{ \frac{|\mathcal{S}||\mathcal{A}|}{m} , 1 \right \} \le \frac{|\mathcal{S}||\mathcal{A}|}{c'} \cdot \frac{1}{m} , \quad \text{for all $m \ge 1$}
\end{align}
\begin{align}
\sum_{m=1}^{M}  V^{*}(\mu) - V^{\pi_m}(\mu) \le |\mathcal{S}||\mathcal{A}| \cdot \min{ \left \{ \sqrt{\frac{M}{c' \cdot (1-\gamma)}}, \frac{\log M +1}{c'} \right \}  }, \quad \text{for all $m \ge 1$}
\end{align}
where $c' = \frac{(1-\gamma)^3}{2} \cdot {\min} \left \{ \frac{\min_s{\mu(s)} }{2} , \frac{ (1-\gamma)}{|\mathcal{S}||\mathcal{A}|} \right \} > 0$.

Finally, we get the desired result by \Cref{lemma:performance_difference_in_rho}:
\begin{align}
V^{*}(\rho) - V^{\pi_m}(\rho) \le \frac{1}{1-\gamma} \cdot \left \| \frac{1}{\mu} \right \|_{\infty} \cdot \left ( V^{*}(\mu) - V^{\pi_m}(\mu) \right ) \le \frac{|\mathcal{S}||\mathcal{A}|}{c} \cdot \frac{1}{m} , \quad \text{for all $m \ge 1$}
\end{align}
\begin{align}
\sum_{m=1}^{M}  V^{*}(\rho) - V^{\pi_m}(\rho) \le |\mathcal{S}||\mathcal{A}| \cdot \min{ \left \{ \sqrt{\frac{M}{c \cdot (1-\gamma)}}, \frac{\log M +1}{c} \right \}  }, \quad \text{for all $m \ge 1$}
\end{align}
where $c = \frac{(1-\gamma)^4}{2} \cdot \left \| \frac{1}{\mu} \right \|_{\infty}^{-1} \cdot {\min} \left \{ \frac{\min_s{\mu(s)} }{2} , \frac{ (1-\gamma)}{|\mathcal{S}||\mathcal{A}|} \right \} > 0$.
\end{proof}

\subsection{Convergence Rate of Batch CAPO}
For ease of exposition, we restate \Cref{theorem:batch_convergence_rate} as follows.
\begin{theorem*}
Consider a tabular softmax parameterized policy $\pi_\theta$. Under Batch CAPO with ${\alpha_m(s, a) = \log (\frac{1}{\pi_{\theta_m}(a\rvert s)})}$ and $B_m = \left \{ (s,a) : (s,a) \in \mathcal{S} \times \mathcal{A} \right \}$, we have :
\begin{align}
V^{*}(\rho) - V^{\pi_m}(\rho) \le \frac{1}{c} \cdot \frac{1}{m}, \quad \text{for all $m \ge 1$}
\end{align}
\begin{align}
\sum_{m=1}^{M}  V^{*}(\rho) - V^{\pi_m}(\rho) \le \min{ \left \{ \sqrt{\frac{M}{c \cdot (1-\gamma)}}, \frac{\log M +1}{c} \right \}  }, \quad \text{for all $m \ge 1$}\\
\end{align}
where $c = \frac{(1-\gamma )^4}{|\mathcal{A}|} \cdot \left \| \frac{1}{\mu} \right \|_{\infty}^{-1} \cdot \underset{s}{min} \left \{ \mu(s) \right \} > 0$.
\label{app:batch_convergence_rate}
\end{theorem*}

\begin{proof}[Proof of \Cref{theorem:batch_convergence_rate}]  
\phantom{}

The proof can be summarized as follows:
\begin{enumerate}
\itemsep0em
  \item  We first construct the lower bound of the improvement of the performance $V^{\pi_{m+1}}(s) - V^{\pi_m}(s)$ in state visitation distribution, number of actions, and advantage value in \Cref{lemma:lower_bdd}.
  \item We then construct the upper bound of the performance difference $ V^{*}(s) - V^{\pi_m}(s)$ using $V^{\pi_{m+1}}(s) - V^{\pi_m}(s)$.
  \item Finally, we can show the desired result inductively \ref{lemma:induction}.
\end{enumerate}

\begin{lemma} Under (\ref{eq:CAPO_form}) with ${\alpha_m(s, a) = \log (\frac{1}{\pi_{\theta_m}(a\rvert s)})}$, if $B_m = \left \{ (s,a) : (s,a) \in \mathcal{S} \times \mathcal{A} \right \}$, then the updated policy weight $\pi_{m+1}(a|s)$ can be written as :
\label{lemma:updated_policy_weight}
\begin{align*}
\pi_{m+1}(a|s) =
\begin{cases}
\frac{1}{|s_{m}^+| + \underset{A^{m}(s,a) = 0}{\sum} \pi_{m}(a|s) + \underset{A^{m}(s,a) < 0}{\sum} \pi_{m}(a|s)^2}  & \text{, if } A^{m}(s,a) > 0 \\
\frac{\pi_m(a)}{|s_{m}^+| + \underset{A^{m}(s,a) = 0}{\sum} \pi_{m}(a|s) + \underset{A^{m}(s,a) < 0}{\sum} \pi_{m}(a|s)^2}  & \text{, if } A^{m}(s,a) = 0 \\
\frac{\pi_m(a)^2}{|s_{m}^+| + \underset{A^{m}(s,a) = 0}{\sum} \pi_{m}(a|s) + \underset{A^{m}(s,a) < 0}{\sum} \pi_{m}(a|s)^2}  & \text{, if } A^{m}(s,a) < 0
\end{cases}
\end{align*}
where $s_{m}^{+} := \left \{ a \in \mathcal{S} \: | \: A^{m}(s,a) > 0 \right \}$

\begin{proof}[Proof of \Cref{lemma:updated_policy_weight}]  
\phantom{}

For $A^{m}(s,a) > 0$ :
\begin{align}
&\pi_{m+1}(a|s) = \frac{e^{\theta_m(s,a)+\ln \frac{1}{\pi_m(a|s)}}}{\underset{a \in \mathcal{A}}{\sum} e^{\theta_{m+1}(s,a)} } = \frac{e^{\theta_m(s,a)+\ln \frac{\underset{a \in \mathcal{A}}{\sum} e^{\theta_{m}(s,a)}}{e^{\theta_m(s,a)}}}}{\underset{a \in \mathcal{A}}{\sum} e^{\theta_{m+1}(s,a)} } = \frac{\underset{a \in \mathcal{A}}{\sum} e^{\theta_{m}(s,a)}}{\underset{a \in \mathcal{A}}{\sum} e^{\theta_{m+1}(s,a)}} &&
\end{align}

For $A^{m}(s,a) = 0$ :
\begin{align}
&\pi_{m+1}(a|s) = \frac{e^{\theta_m(s,a)}}{\underset{a \in \mathcal{A}}{\sum} e^{\theta_{m+1}(s,a)} } = \frac{e^{ \theta_m(s,a)} \cdot \underset{a \in \mathcal{A}}{\sum} e^{\theta_{m}(s,a)} }{\underset{a \in \mathcal{A}}{\sum} e^{\theta_{m}(s,a)}\underset{a \in \mathcal{A}}{\sum} e^{\theta_{m+1}(s,a)} } = \pi_m(a|s) \cdot \frac{\underset{a \in \mathcal{A}}{\sum} e^{\theta_{m}(s,a)}}{\underset{a \in \mathcal{A}}{\sum} e^{\theta_{m+1}(s,a)}}&&
\end{align}

For $A^{m}(s,a) < 0$ :
\begin{align}
&\pi_{m+1}(a|s) = \frac{e^{\theta_m(s,a)-\ln \frac{1}{\pi_m(a|s)}}}{\underset{a \in \mathcal{A}}{\sum} e^{\theta_{m+1}(s,a)} } = \frac{e^{2 \cdot \theta_m(s,a)}}{\underset{a \in \mathcal{A}}{\sum} e^{\theta_{m}(s,a)}\underset{a \in \mathcal{A}}{\sum} e^{\theta_{m+1}(s,a)} } = \pi_m(a|s)^2 \cdot \frac{\underset{a \in \mathcal{A}}{\sum} e^{\theta_{m}(s,a)}}{\underset{a \in \mathcal{A}}{\sum} e^{\theta_{m+1}(s,a)}}&&
\end{align}

Moreover, since $\underset{a \in \mathcal{A}}{\sum} \pi_{m+1}(a|s) = 1$, we have:
\begin{align}
&\underset{A^{m}(s,a) > 0}{\sum} \pi_{m+1}(a|s) + \underset{A^{m}(s,a) = 0}{\sum} \pi_{m+1}(a|s) + \underset{A^{m}(s,a) < 0}{\sum} \pi_{m+1}(a|s) &&\\
&= |s_{m}^+| \cdot \frac{\underset{a \in \mathcal{A}}{\sum} e^{\theta_{m}(s,a)}}{\underset{a \in \mathcal{A}}{\sum} e^{\theta_{m+1}(s,a)}} + \underset{A^{m}(s,a) = 0}{\sum} \pi_{m}(a|s) \cdot \frac{\underset{a \in \mathcal{A}}{\sum} e^{\theta_{m}(s,a)}}{\underset{a \in \mathcal{A}}{\sum} e^{\theta_{m+1}(s,a)}} + \underset{A^{m}(s,a) < 0}{\sum} \pi_{m}(a|s)^2 \cdot \frac{\underset{a \in \mathcal{A}}{\sum} e^{\theta_{m}(s,a)}}{\underset{a \in \mathcal{A}}{\sum} e^{\theta_{m+1}(s,a)}} &&\\
&= \left ( |s_{m}^+| + \underset{A^{m}(s,a) = 0}{\sum} \pi_{m}(a|s) + \underset{A^{m}(s,a) < 0}{\sum} \pi_{m}(a|s)^2 \right ) \cdot  \frac{\underset{a \in \mathcal{A}}{\sum} e^{\theta_{m}(s,a)}}{\underset{a \in \mathcal{A}}{\sum} e^{\theta_{m+1}(s,a)}} = 1,&&
\end{align}
where $s_{m}^{+} := \left \{ a \in \mathcal{S} \: | \: A^{m}(s,a) > 0 \right \}$\\

Hence, we get:
\begin{align}
& \frac{\underset{a \in \mathcal{A}}{\sum} e^{\theta_{m}(s,a)}}{\underset{a \in \mathcal{A}}{\sum} e^{\theta_{m+1}(s,a)}} = \frac{1}{|s_{m}^+| + \underset{A^{m}(s,a) = 0}{\sum} \pi_{m}(a|s) + \underset{A^{m}(s,a) < 0}{\sum} \pi_{m}(a|s)^2}.&&
\end{align}

Finally, we get the desired result by substitution.

\end{proof}
\end{lemma}

\begin{lemma} Under (\ref{eq:CAPO_form}) with ${\alpha_m(s, a) = \log (\frac{1}{\pi_{\theta_m}(a\rvert s)})}$, if $B_m = \left \{ (s,a) : (s,a) \in \mathcal{S} \times \mathcal{A} \right \}$ then the improvement of the performance $V^{\pi_{m+1}}(s) - V^{\pi_m}(s)$ can be bounded by :
\label{lemma:lower_bdd_2}
\begin{align}
&V^{\pi_{m+1}}(s) - V^{\pi_{m}}(s) \ge \frac{1}{|\mathcal{A}|} \cdot \sum_{s' \in \mathcal{S}} d^{\pi_m}_{s}(s') \sum_{a \in s_m^{'+}} A^{m}(s', a)^2 &&
\end{align}
where $s_{m}^{'+} := \left \{ a \in \mathcal{S} \: | \: A^{m}(s',a) > 0 \right \}$
\end{lemma}

\begin{proof}[Proof of \Cref{lemma:lower_bdd_2}]  
\phantom{}
\begin{align}
V^{\pi_{m+1}}(s) - V^{\pi_{m}}(s)
&= \frac{1}{1-\gamma} \cdot \sum_{s' \in \mathcal{S}} d^{\pi_{m+1}}_{s}(s') \sum_{a \in \mathcal{A}} \pi_{m+1}(a|s') \cdot A^{m}(s', a) &&\\
&= \frac{1}{1-\gamma} \cdot \sum_{s' \in \mathcal{S}} d^{\pi_{m+1}}_{s}(s') \cdot \frac{1}{|s_{m}^+| + \underset{A^{m}(s,a) = 0}{\sum} \pi_{m}(a|s) + \underset{A^{m}(s,a) < 0}{\sum} \pi_{m}(a|s)^2}&&\\
& \quad \cdot \left ( \sum_{a \in s_{m}^{'+}} A^{m}(s', a) + \sum_{a \notin s_{m}^{'+}} \pi_{m}(a|s')^2 \cdot A^{m}(s', a) \right )  &&\\
&\ge \frac{1}{1-\gamma} \cdot \sum_{s' \in \mathcal{S}} d^{\pi_{m+1}}_{s}(s') \cdot \frac{1}{|s_{m}^+| + \underset{A^{m}(s,a) = 0}{\sum} \pi_{m}(a|s) + \underset{A^{m}(s,a) < 0}{\sum} \pi_{m}(a|s)^2} &&\\
& \quad \cdot \left ( \sum_{a \in s_{m}^{'+}} A^{m}(s', a) + \sum_{a \notin s_{m}^{'+}} \pi_{m}(a|s') \cdot A^{m}(s', a) \right )  &&\\
&= \frac{1}{1-\gamma} \cdot \sum_{s' \in \mathcal{S}} d^{\pi_{m+1}}_{s}(s') \cdot \frac{1}{|s_{m}^+| + \underset{A^{m}(s,a) = 0}{\sum} \pi_{m}(a|s) + \underset{A^{m}(s,a) < 0}{\sum} \pi_{m}(a|s)^2} &&\\
& \quad \cdot \left ( \sum_{a \in s_{m}^{'+}} (1-\pi_{m}(a|s')) \cdot A^{m}(s', a) \right )  &&\\
&\ge \frac{1}{1-\gamma} \cdot \frac{1}{|\mathcal{A}|} \cdot \sum_{s' \in \mathcal{S}} d^{\pi_{m+1}}_{s}(s') \cdot \left ( \sum_{a \in s_{m}^{'+}} (1-\pi_{m}(a|s')) \cdot A^{m}(s', a) \right )  &&\\
&\ge \frac{1}{|\mathcal{A}|} \cdot \sum_{s' \in \mathcal{S}} d^{\pi_{m+1}}_{s}(s') \cdot  \sum_{a \in s_{m}^{'+}} A^{m}(s', a)^2 
\end{align}

The first equation holds by the performance difference lemma in \Cref{lemma:perf_diff}.\\
The second equation holds by \Cref{lemma:change_of_policy_weight_2}.\\
The third equation holds by the definition of $A(s,a)$.\\ %\ref{prop:pi_s_Adv}
The last inequality holds by the bound of $A(s,a)$ in \Cref{lemma:upper_bound_of_advantage}.\\
\end{proof}

Hence, combining \Cref{lemma:lower_bdd_2} and \Cref{lemma:lower_bound_of_performance_difference}, we can construct the upper bound of the performance difference $ V^{*}(s) - V^{\pi_m}(s)$ using $V^{\pi_{m+1}}(s) - V^{\pi_m}(s)$ :
\begin{align}
V^{\pi_{m+1}}(s) - V^{\pi_{m}}(s)
&\ge \frac{1}{|\mathcal{A}|} \cdot \sum_{s' \in \mathcal{S}} d^{\pi_{m+1}}_{s}(s') \cdot  \sum_{a \in s_{m}^{'+}} A^{m}(s', a)^2  &&\\
&\ge \frac{1}{|\mathcal{A}|} \cdot d^{\pi_{m+1}}_{s}(\tilde{s_m}) \cdot A^{m}(\tilde{s_m}, \tilde{a_m})^2 &&\\
&= \frac{1}{|\mathcal{A}|} \cdot d^{\pi_{m+1}}_{s}(\tilde{s_m}) \cdot (1-\gamma)^2 \cdot (\frac{1}{1-\gamma})^2 \cdot A^{m}(\tilde{s_m}, \tilde{a_m})^2&&\\
&\ge \frac{1}{|\mathcal{A}|} \cdot d^{\pi_{m+1}}_{s}(\tilde{s_m}) \cdot (1-\gamma)^2 \cdot \left ( V^{*}(s) - V^{\pi_m}(s) \right )^2&&
\end{align}

Moreover, if we consider the whole starting state distribution $\mu$, we have :
\begin{align}
V^{\pi_{m+1}}(\mu) - V^{\pi_{m}}(\mu)
&\ge \frac{1}{|\mathcal{A}|} \cdot d^{\pi_{m+1}}_{\mu}(\tilde{s_m}) \cdot (1-\gamma)^2 \cdot \left ( V^{*}(\mu) - V^{\pi_m}(\mu) \right )^2&&\\
&\ge \frac{1}{|\mathcal{A}|} \cdot \mu(\tilde{s_m}) \cdot (1-\gamma)^3 \cdot \left ( V^{*}(\mu) - V^{\pi_m}(\mu) \right )^2&&\\
& \ge \underbrace{\frac{(1-\gamma)^3}{|\mathcal{A}|} \cdot \underset{s' \in \mathcal{S}}{\min}\left \{ \mu(s') \right \}}_{:= c' > 0}  \cdot \left ( V^{*}(\mu) - V^{\pi_m}(\mu) \right )^2
\end{align}

The second inequality holds since $d^{\pi}_{\mu}(s) \ge (1- \gamma) \cdot \mu(s)$ in \Cref{lemma:lower_bound_of_state_visitation_distribution}.\\

Since $V^{\pi_{m+1}}(\mu) - V^{\pi_{m}}(\mu) = (V^{\pi^{*}}(\mu) - V^{\pi_{m}}(\mu)) - (V^{\pi^{*}}(\mu) - V^{\pi_{m+1}}(\mu))$, by rearranging the inequality above, we have :
\begin{align}
&\delta_{m+1} \le \delta_{m} - c' \cdot \delta_{m}^2 \quad \text{where $\delta_{m} = V^{\pi^*}(\mu) - V^{\pi_m}(\mu) $}
\end{align}

Then, we can get the following result by induction based on \Cref{lemma:induction} :
\begin{align}
V^{*}(\mu) - V^{\pi_m}(\mu) \le \frac{1}{c'} \cdot \frac{1}{m}, \quad \text{for all $m \ge 1$}
\end{align}
\begin{align}
\sum_{m=1}^{M}  V^{*}(\mu) - V^{\pi_m}(\mu) \le \min{ \left \{ \sqrt{\frac{M}{c \cdot (1-\gamma)}}, \frac{\log M +1}{c'} \right \}  }, \quad \text{for all $m \ge 1$}\\
\end{align}
where $c' = \frac{(1-\gamma )^3}{|\mathcal{A}|} \cdot \underset{s}{min} \left \{ \mu(s) \right \} > 0$.

Finally, we get the desired result by \Cref{lemma:performance_difference_in_rho}:
\begin{align}
V^{*}(\rho) - V^{\pi_m}(\rho) \le \frac{1}{1-\gamma} \cdot \left \| \frac{1}{\mu} \right \|_{\infty} \cdot \left ( V^{*}(\mu) - V^{\pi_m}(\mu) \right ) \le \frac{1}{c} \cdot \frac{1}{m} , \quad \text{for all $m \ge 1$}
\end{align}
\begin{align}
\sum_{m=1}^{M}  V^{*}(\rho) - V^{\pi_m}(\rho) \le \min{ \left \{ \sqrt{\frac{M}{c \cdot (1-\gamma)}}, \frac{\log M +1}{c} \right \}  }, \quad \text{for all $m \ge 1$}
\end{align}
where $c = \frac{(1-\gamma )^4}{|\mathcal{A}|} \cdot \left \| \frac{1}{\mu} \right \|_{\infty}^{-1} \cdot \underset{s}{min} \left \{ \mu(s) \right \} > 0$.\\

\begin{remark}
\normalfont In \Cref{theorem:batch_convergence_rate}, we choose the learning rate ${\alpha_m(s, a)}$ to be exactly $\log (\frac{1}{\pi_{\theta_m}(a\rvert s)})$ instead of greater than or equal to $\log (\frac{1}{\pi_{\theta_m}(a\rvert s)})$. The reason is that under ${\alpha_m(s, a)} = \log (\frac{1}{\pi_{\theta_m}(a\rvert s)})$, we can guarantee that all state-action pair with positive advantage value can get the same amount of the policy weight with each other actions in the same state after every update \ref{lemma:updated_policy_weight}. This property directly leads to the result of \Cref{lemma:lower_bdd_2} that the one-step improvement $V^{\pi_{m+1}}(s) - V^{\pi_{m}}(s)$ can be quantified using the summation of all positive advantage value $\sum_{a \in s_{m}^{'+}} A^{m}(s', a)^2$, and hence it guarantees that one of the $A^{m}(s', a)^2$ will connect the one-step improvement with the performance difference $V^{\pi^{*}}(s) - V^{\pi_{m}}(s)$. This property also prevents some extreme cases where one of the learning rates of the state-action pairs with extremely tiny but positive advantage value dominates the updated policy weight, i.e., $\pi_{m+1}(a_m|s_m) \rightarrow 1$, leading to tiny one-step improvement.
\end{remark}

\end{proof}

\subsection{Convergence Rate of Randomized CAPO}
For ease of exposition, we restate \Cref{theorem:randomized_convergence_rate} as follows.
\begin{theorem*}
Consider a tabular softmax parameterized policy $\pi_\theta$, under (\ref{eq:CAPO_form}) with ${\alpha_m(s, a) \ge \log (\frac{1}{\pi_{\theta_m}(a\rvert s)})}$ and $|B_m| = 1$, if Condition \ref{condition:sa} is satisfied, then we have :
\begin{align}
\underset{({s_m},{a_m}) \sim d_{gen}}{\mathbb{E}} \left [ V^{*}(\rho) - V^{\pi_m}(\rho) \right] \le \frac{1}{c} \cdot \frac{1}{m}, \quad \text{for all $m \ge 1$}
\end{align}
\begin{align}
\sum_{m=1}^{M} \underset{({s_m},{a_m}) \sim d_{gen}}{\mathbb{E}} \left [ V^{*}(\rho) - V^{\pi_m}(\rho) \right] \le \min{ \left \{ \sqrt{\frac{M}{c \cdot (1-\gamma)}}, \frac{\log M +1}{c} \right \}  }, \quad \text{for all $m \ge 1$}\\
\end{align}
where $c = \frac{(1-\gamma )^4}{2} \cdot \left \| \frac{1}{\mu} \right \|_{\infty}^{-1} \cdot \underset{(s,a)}{min} \left \{ d_{gen}(s,a)\cdot\mu(s) \right \} > 0$ and $d_{gen}:\mathcal{S} \times \mathcal{A} \rightarrow (0,1)$, $d_{gen}(s,a) = \mathbb{P}((s,a) \in B_m)$.
\label{app:randomized_convergence_rate}
\end{theorem*}

\begin{proof}[Proof of \Cref{theorem:randomized_convergence_rate}]  
\phantom{}

The proof can be summarized as:
\begin{enumerate}
\itemsep0em
  \item  We first write the improvement of the performance $V^{\pi_{m+1}}(s) - V^{\pi_m}(s)$ in state visitation distribution, policy weight, and advantage value in \Cref{lemma:lower_bdd}, and also construct the lower bound of it. Note that the result is the same as \Cref{subsection:cyclic_CAPO}.
  \item We then write the improvement of the performance $V^{\pi_{m+1}}(s) - V^{\pi_m}(s)$ in probability form condition on $(s_m,a_m)$.
  \item By taking expectation of the probability form, we get the upper bound of the expected performance difference $\underset{({s_m},{a_m}) \sim d_{gen}}{\mathbb{E}} \left [ V^{*}(\mu) - V^{\pi_m}(\mu) \right]$ using $\underset{({s_m},{a_m}) \sim d_{gen}}{\mathbb{E}} \left [ V^{\pi_{m+1}}(\mu) - V^{\pi_m}(\mu) \right]$.
  \item Finally, we can show the desired result by induction based on  \Cref{lemma:induction}.
\end{enumerate}

By \Cref{lemma:lower_bdd}, we have for all $m \ge 1$:
\begin{align}
V^{\pi_{m+1}}(s) - V^{\pi_{m}}(s) =
\begin{cases}
\frac{d^{\pi_{m+1}}_{s}(s_m)}{1-\gamma } \cdot \frac{W^+}{1-\pi_{m}(a_m|s_m)} \cdot A^{m}(s_m, a_m)  & \text{, if } A^{m}(s_m, a_m) > 0 \\
\frac{d^{\pi_{m+1}}_{s}(s_m)}{1-\gamma } \cdot \frac{W^-}{1-\pi_{m}(a_m|s_m)} \cdot (-A^{m}(s_m, a_m))  & \text{, if } A^{m}(s_m, a_m) <  0\\
\end{cases}\\
\text{where }
\begin{cases}
(1-\pi_{m}(a_m|s_m)) \ge W^+ \ge \frac{(1-\pi_{m}(a_m|s_m))^2}{2-\pi_{m}(a_m|s_m)}\\
\pi_m(a_m|s_m) \ge W^- \ge \frac{\pi_{m}(a_m|s_m) \cdot (1-\pi_{m}(a_m|s_m))^2}{\pi_{m}(a_m|s_m)^2 - \pi_{m}(a_m|s_m) + 1}
\end{cases}
\end{align}

and it can also be lower bounded by:
\begin{align}
V^{\pi_{m+1}}(s) - V^{\pi_{m}}(s) \ge
\begin{cases}
\frac{d^{\pi_{m+1}}_{s}(s_m)}{2} \cdot A^{m}(s_m, a_m)^2  & \text{, if } A^{m}(s_m, a_m) > 0 \\
d^{\pi_{m+1}}_{s}(s_m) \cdot \pi_m(a_m|s_m) \cdot A^{m}(s_m, a_m)^2  & \text{, if } A^{m}(s_m, a_m) <  0
\end{cases}
\end{align}

Hence, considering the randomness of the generator, it will choose (s,a) with probability $d_{gen}(s,a)$ to update in each episode $m$. Then we can rewrite \Cref{lemma:lower_bdd} in probability form :
\begin{align}
V^{\pi_{m+1}}(s) - V^{\pi_{m}}(s) \ge
\begin{cases}
\frac{d^{\pi_{m+1}}_{s}(s_m)}{2} \cdot A^{m}(\tilde{s_m},\tilde{a_m})^2  & \text{, if } A^{m}(s_m, a_m) > 0 \text{, w.p. } d_{gen}(\tilde{s_m},\tilde{a_m}) \\
\frac{d^{\pi_{m+1}}_{s}(s_m)}{2} \cdot A^{m}(s_m, a_m)^2  & \text{, if } A^{m}(s_m, a_m) > 0 \text{, w.p. } d_{gen}(s,a)\\
d^{\pi_{m+1}}_{s}(s_m) \cdot \pi_m(a_m|s_m) \cdot A^{m}(s_m, a_m)^2  & \text{, if } A^{m}(s_m, a_m) <  0 \text{, w.p. } d_{gen}(s,a)
\end{cases}
\end{align}

Then, by taking expectation, we have :
\begin{align}
\underset{({s_m},{a_m}) \sim d_{gen}}{\mathbb{E}} \left [ V^{\pi_{m+1}}(s) - V^{\pi_m}(s) \right] 
&= \sum_{(s',a') \in \mathcal{S} \times \mathcal{A} } d_{gen}(s',a') \cdot \left [ V^{\pi_{m+1}}(s) - V^{\pi_m}(s) \:|\: (s_m,a_m)=(s',a') \right ]  &&\\
&\ge d_{gen}(\tilde{s_m},\tilde{a_m}) \cdot \left [ V^{\pi_{m+1}}(s) - V^{\pi_m}(s) \:|\: (s_m,a_m)=(\tilde{s_m},\tilde{a_m}) \right ]&&\\
&\ge d_{gen}(\tilde{s_m},\tilde{a_m}) \cdot \frac{d^{\pi_{m+1}}_{s}(s_m)}{2} \cdot A^{m}(\tilde{s_m},\tilde{a_m})^2 &&\\
&\ge d_{gen}(\tilde{s_m},\tilde{a_m}) \cdot \frac{d^{\pi_{m+1}}_{s}(s_m)}{2} \cdot (1-\gamma)^2 \cdot \left ( V^{*}(s) - V^{\pi_m}(s) \right )^2&&\\ 
&= d_{gen}(\tilde{s_m},\tilde{a_m}) \cdot \frac{d^{\pi_{m+1}}_{s}(s_m)}{2} \cdot (1-\gamma)^2 \cdot \underset{({s_m},{a_m}) \sim d_{gen}}{\mathbb{E}} \left [ V^{*}(s) - V^{\pi_m}(s) \right]^2 &&
\end{align}

The third inequality holds by \Cref{lemma:lower_bound_of_performance_difference}.\\
The last equation holds since the performance difference at episode $m$ is independent of $(s_m,a_m)$, which is the state action pair chosen at episode $m$.\\

If we consider the whole starting state distribution $\mu$, we have:
\begin{align}
\underset{({s_m},{a_m}) \sim d_{gen}}{\mathbb{E}} \left [ V^{\pi_{m+1}}(\mu) - V^{\pi_m}(\mu) \right]
&\ge d_{gen}(\tilde{s_m},\tilde{a_m}) \cdot \frac{d^{\pi_{m+1}}_{\mu}(s_m)}{2} \cdot (1-\gamma)^2 \cdot \underset{({s_m},{a_m}) \sim d_{gen}}{\mathbb{E}} \left [ V^{*}(\mu) - V^{\pi_m}(\mu) \right]^2 &&\\
&\ge d_{gen}(\tilde{s_m},\tilde{a_m}) \cdot \frac{\mu(s_m)}{2} \cdot (1-\gamma)^3 \cdot \underset{({s_m},{a_m}) \sim d_{gen}}{\mathbb{E}} \left [ V^{*}(\mu) - V^{\pi_m}(\mu) \right]^2 &&\\
&\ge \underbrace{\underset{(s',a') \in \mathcal{S} \times \mathcal{A}}{\min}  \left \{ d_{gen}(s',a') \cdot \mu(s') \right \}  \cdot \frac{(1-\gamma)^3}{2}}_{:= c' > 0}  \cdot \underset{({s_m},{a_m}) \sim d_{gen}}{\mathbb{E}} \left [ V^{*}(\mu) - V^{\pi_m}(\mu) \right]^2 &&
\end{align}

The second inequality holds since $d^{\pi}_{\mu} \ge (1- \gamma) \cdot \mu(s)$ by \Cref{lemma:lower_bound_of_state_visitation_distribution}.\\

Since $\underset{({s_m},{a_m}) \sim d_{gen}}{\mathbb{E}} \left [ V^{\pi_{m+1}}(\mu) - V^{\pi_m}(\mu) \right] = \underset{({s_m},{a_m}) \sim d_{gen}}{\mathbb{E}} \left [ V^{\pi^*}(\mu) - V^{\pi_m}(\mu) \right] - \underset{({s_m},{a_m}) \sim d_{gen}}{\mathbb{E}} \left [ V^{\pi^*}(\mu) - V^{\pi_{m+1}}(\mu) \right]$, by rearranging the inequality above, we have:
\begin{align}
&\delta_{m+1} \le \delta_{m} - c' \cdot \delta_{m}^2 \quad \text{where $\delta_{m} = \underset{({s_m},{a_m}) \sim d_{gen}}{\mathbb{E}} \left [ V^{\pi^*}(\mu) - V^{\pi_m}(\mu) \right]$}
\end{align}

Then, we can get the following result by \Cref{lemma:induction} :
\begin{align}
\underset{({s_m},{a_m}) \sim d_{gen}}{\mathbb{E}} \left [ V^{*}(\mu) - V^{\pi_m}(\mu) \right] \le \frac{1}{c'} \cdot \frac{1}{m}, \quad \text{for all $m \ge 1$}
\end{align}

\begin{align}
\sum_{m=1}^{M} \underset{({s_m},{a_m}) \sim d_{gen}}{\mathbb{E}} \left [ V^{*}(\mu) - V^{\pi_m}(\mu) \right] \le \min{ \left \{ \sqrt{\frac{M}{c' \cdot (1-\gamma)}}, \frac{\log M +1}{c'} \right \}  }, \quad \text{for all $m \ge 1$}\\
\end{align}

where $c' = \frac{(1-\gamma )^3}{2} \cdot \underset{(s,a)}{\min} \left \{ d_{gen}(s,a)\cdot\mu(s) \right \} > 0$.

Finally, we get the desired result by \Cref{lemma:performance_difference_in_rho}:
\begin{align}
\underset{({s_m},{a_m}) \sim d_{gen}}{\mathbb{E}} \left [ V^{*}(\rho) - V^{\pi_m}(\rho) \right] \le \frac{1}{1-\gamma} \cdot \left \| \frac{1}{\mu} \right \|_{\infty} \cdot \underset{({s_m},{a_m}) \sim d_{gen}}{\mathbb{E}} \left [ V^{*}(\mu) - V^{\pi_m}(\mu) \right] ) \le \frac{1}{c} \cdot \frac{1}{m} , \quad \text{for all $m \ge 1$}
\end{align}
\begin{align}
\sum_{m=1}^{M} \underset{({s_m},{a_m}) \sim d_{gen}}{\mathbb{E}} \left [ V^{*}(\mu) - V^{\pi_m}(\mu) \right] \le \min{ \left \{ \sqrt{\frac{M}{c \cdot (1-\gamma)}}, \frac{\log M +1}{c} \right \}  }, \quad \text{for all $m \ge 1$}
\end{align}
where $c = \frac{(1-\gamma )^4}{2} \cdot \left \| \frac{1}{\mu} \right \|_{\infty}^{-1} \cdot \underset{(s,a)}{min} \left \{ d_{gen}(s,a)\cdot\mu(s) \right \} > 0$.

\end{proof}

%% file: appendices/appendix_OnCAPO.tex
\section{On-Policy CAPO With Global Convergence}
\label{app:OnCAPO}
The main focus and motivation for CAPO is on off-policy RL.
Despite this, we show that it is also possible to apply CAPO to on-policy learning.
While on-policy learning is a fairly natural RL setting, one fundamental issue with on-policy learning is the \textit{committal issue}, which was recently discovered by \citep{chung2021beyond,mei2021understanding}.
In this section, we show that CAPO could tackle the committal issue with the help of variable learning rates.
%\subsection{On-Policy CAPO with State-Action Dependent Learning Rate}
Consider on-policy CAPO with state-action dependent learning rate:
\begin{equation}
\label{eq:onCAPO_alpha_update}
\theta_{m+1}(s, a) = \theta_{m}(s, a) + \alpha^{(m)}(s, a) \cdot \sgn(A^{(m)}(s,a)) \cdot \mathbb{I}\{a = a_m\},
\end{equation}
where $N^{(k)}(s,a) = \sum^k_{m=0} \mathbb{I}\{(s, a) \in \mathcal{B}_m\}$ and $\alpha^{(m)}(s,a)$ is given by:
\begin{align}
    \alpha^{(m)}(s, a)=\begin{cases}
    \log\big(\frac{1}{\pi^{(m)}(a\rvert s)}\big),& \text{ if } A^{(m)}(s,a)\leq 0\\ 
    \log\big(\frac{\beta}{1-\beta}\cdot\frac{1}{\pi^{(m)}(a\rvert s)}\big), &\text{ if }A^{(m)}(s,a)>0 \text{ and } \pi^{(m)}(a\rvert s) < \beta \\ 
    \zeta\log\big(\frac{N^{(m)}(s,a)+1}{{N^{(k)}(s,a)}}\big), &\text{ otherwise }
    \end{cases}
\end{align}

\subsection{Global Convergence of On-Policy CAPO}
\label{app:OnCAPO:convergence}
Recall that in the on-policy setting, we choose the step size of CAPO as 
\begin{align}
    \alpha^{(k)}(\pi^{(k)}(a\rvert s))=\begin{cases}
    \log\big(\frac{1}{\pi^{(k)}(a\rvert s)}\big),& \text{ if } A^{(k)}(s,a)\leq 0\\ 
    \log\big(\frac{\beta}{1-\beta}\cdot\frac{1}{\pi^{(k)}(a\rvert s)}\big), &\text{ if }A^{(k)}(s,a)>0 \text{ and } \pi^{(k)}(a\rvert s) < \beta \\ 
    \zeta\log\big(\frac{N^{(k)}(s,a)+1}{{N^{(k)}(s,a)}}\big), &\text{ otherwise }
    \end{cases}
    \label{eq:on-policy CAPO update}
\end{align}

\begin{theorem}
\label{app:thm:on-policy CAPO}
Under on-policy CAPO with $0<\beta\leq \frac{1}{\lvert \cA\rvert+1}$ and $0<\zeta\leq \frac{1}{\lvert \cA\rvert}$, we have $V_k(s)\rightarrow V^*(s)$ as $k\rightarrow \infty$, for all $s\in \cS$, almost surely. 
\end{theorem}
To prove this result, we start by introducing multiple supporting lemmas.

\begin{lemma}[A Lower Bound of Action Probability]
\label{lemma:pi lower bound}
Under on-policy CAPO, in any iteration $k$, if an action $a$ that satisfies $\pi^{(k)}(a\rvert s)<\beta$ and $A^{(k)}(s,a)>0$ is selected for policy update, then we have $\pi^{(k+1)}(a\rvert s)>\beta$.
\end{lemma}
\begin{proof}[Proof of Lemma \ref{lemma:pi lower bound}]
By the on-policy CAPO update in (\ref{eq:on-policy CAPO update}), we know that if the selected action $a$ satisfies $\pi^{(k)}(a\rvert s)<\beta$ and $A^{(k)}(s,a)>0$, we have
\begin{align}
    \theta^{(k+1)}_{s,a}&=\theta^{(k)}_{s,a}+\log\big(\frac{\beta}{1-\beta}\cdot\frac{1}{\pi^{(k)}(a\rvert s)}\big)\\
    &=\theta^{(k)}_{s,a}+\log\Big(\frac{\beta}{1-\beta}\cdot\frac{\sum_{a'\in\cA} \exp(\theta^{(k)}_{s,a'})}{\exp(\theta^{(k)}_{s,a})}\Big)\\
    &=\log\Big(\frac{\beta}{1-\beta}\cdot\sum_{a'\in\cA} \exp(\theta^{(k)}_{s,a'}) \Big).
\end{align}
Therefore, by the softmax policy parameterization, we have
\begin{align}
    \pi^{(k+1)}(a\rvert s)&=\frac{\frac{\beta}{1-\beta}\cdot\sum_{a'\in\cA} \exp(\theta^{(k)}_{s,a'})}{\frac{\beta}{1-\beta}\cdot\sum_{a'\in\cA} \exp(\theta^{(k)}_{s,a'})+\sum_{a''\in\cA, a''\neq a} \exp(\theta^{(k)}_{s,a''})}\\
    &=\frac{\frac{\beta}{1-\beta}}{\frac{\beta}{1-\beta}+(1-\pi^{(k)}(a\rvert s))}>\beta.
\end{align}
\end{proof}
As we consider tabular policy parameterization, we could discuss the convergence behavior of each state separately.
For ease of exposition, we first fix a state $s\in\cS$ and analyze the convergence regarding the policy at state $s$.
Define the following events:
\begin{align}
    E_0:=&\Big\{\omega: I_s^{+}(\omega)\neq \varnothing \Big\},\\
    E_1:=&\Big\{\omega: \lim_{k\rightarrow \infty}\pi_{s,a}^{(k)}(\omega)= 0 , \forall a\in  I_{s}^{-}(\omega)\Big\},\\
    {E}_{1,1}:=&\Big\{\omega: \exists a\in I_s^{-} \text{ with } N^{(\infty)}_{s,a}(\omega)=\infty \Big\},\\
    E_{1,2}:=&\Big\{\omega:\exists a\in I_s^{+} \text{ with } N^{(\infty)}_{s,a}(\omega)=\infty\Big\},\\
    E_{1,3}:=&\Big\{\omega:\exists a'\in I_s^{0}(\omega) \text{ with } N^{(\infty)}_{s,a'}(\omega)=\infty\Big\}.
    %E_{1,3}:=&\Big\{\omega: N^{(\infty)}_{s,a}(\omega)<\infty, \forall a\in  I_{s}^{+}(\omega);\\
    %&\hspace{22pt}\exists a'\in I_s^{0}(\omega) \text{ with } N^{(\infty)}_{s,a'}(\omega)<\infty \text{ and }\lim_{k\rightarrow \infty}\pi^{(k)}_{s,a'}(\omega)= 1\Big\},\\
    %E_{1,4}:=&\Big\{\omega: N^{(\infty)}_{s,a}(\omega)<\infty, \forall a\in  I_{s}^{+}(\omega);\\
    %&\hspace{22pt}\text{For every } a\in I_{s}^{0} \text{ with } N^{(\infty)}_{s,a}(\omega)=\infty, \pi^{(k)}_{s,a}(\omega)\nrightarrow 1 \Big\}
\end{align}
Since there shall always exist at least one action $a\in\cA$ with $N^{(\infty)}_{s,a}=\infty$ for each sample path, then we have $E_{1,1}\cup E_{1,2} \cup E_{1,3}=\Omega$.
Therefore, we can rewrite the event ${E}_1^{c}$ as ${E}_{1}^{c}=({E}_1^{c}\cap E_{1,1})\cup({E}_1^{c}\cap E_{1,2})\cup({E}_1^{c}\cap E_{1,3})$.
By the union bound, we have
\begin{equation}
    \bbP({E}_{1}^{c}\rvert E_0)\leq \sum_{i=1}^{3}\bbP({E}_{1}^{c}\cap E_{1,i}\rvert E_0).\label{eq:E1 tilde complement}
\end{equation}
\begin{lemma}
\label{lemma:E1,1 and E1 tilde}
Under on-policy CAPO and the condition that $\bbP(E_0)>0$, we have $\bbP({E}_1^{c}\cap {E}_{1,1}\rvert E_0)=0$.
\end{lemma}
\begin{proof}[Proof of Lemma \ref{lemma:E1,1 and E1 tilde}]
Under on-policy CAPO and the condition that $E_0$ happens, for each $\omega$, there exists an action $a'\in I_{s}^{+}(\omega)$ and some finite constant $B_0$ such that $\theta^{(k)}(s,a')\geq B_0$, for all sufficiently large $k\geq T^{+}_{s,a'}(\omega)$.
On the other hand, for each $a''\in I_{s}^{-}(\omega)$, we know that $\theta^{(k)}(s,a'')$ is non-increasing for all $k\geq T^{-}_{s,a''}$.
Therefore, $\pi^{(k)}(s,a'')\leq \frac{\exp\Big(\theta^{\big(T^{-}_{s,a''}\big)}(s,a'')\Big)}{\exp\Big(\theta^{\big(T^{-}_{s,a''}\big)}(s,a'')\Big)+\exp(B_0)}$, for all $k\geq \max\{T^{+}_{s,a'},T^{-}_{s,a''}\}$.
As a result, we know if $(s,a'')$ is contained in $\cB^{(k)}$ with $k\geq \max\{T^{+}_{s,a'},T^{-}_{s,a''}\}$, under CAPO, we must have
\begin{equation}
    \theta^{(k+1)}_{s,a''}-\theta^{(k)}_{s,a''}\leq -\log\bigg(\frac{\exp\big(\theta^{(T^{-}_{s,a''})}(s,a'')\big)+\exp(B_0)}{\exp\big(\theta^{(T^{-}_{s,a''})}(s,a'')\big)}\bigg).
\end{equation}
Therefore, for each $\omega\in E_0$ and for each $a''\in I_{s}^{-}(\omega)$, if $N_{s,a''}^{(\infty)}(\omega)=\infty$, then we have $\theta^{(k)}_{s,a''}(\omega)\rightarrow -\infty$ as $k\rightarrow \infty$.
%$\bbP\big(\{\omega:N^{(\infty)}_{s,a''}(\omega)=\infty\}\rvert E_0\big)=0$, for any $a''\in I_{s}^{-}$.
This implies that $\bbP(E_1^{c}\cap E_{1,1}\rvert E_0)=0$.
\end{proof}

\begin{lemma}
\label{lemma:E1,2 and E1 tilde}
Under on-policy CAPO and the condition that $\bbP(E_0)>0$, we have $\bbP({E}_1^{c}\cap {E}_{1,2}\rvert E_0)=0$.
\end{lemma}
\begin{proof}[Proof of Lemma \ref{lemma:E1,2 and E1 tilde}]
By Lemma \ref{lemma:E1,1 and E1 tilde}, we have $\bbP(E_1^{c}\cap E_{1,2}\rvert E_0)= \bbP(E_1^{c}\cap E_{1,1}^{c}\cap E_{1,2}\rvert E_0)$.
Let $a\in I_{s}^{+}$ be an action with $N^{(\infty)}_{s,a}(\omega)=\infty$, and suppose $N_{s,a'}^{(\infty)}$ are finite for all $a'\in I_{s}^{-}$ (which also implies that $\theta^{(k)}_{s,a'}$ are finite for all $k\in \mathbb{N}$).
Let $\{k_m\}_{m=1}^{\infty}$ be the sequence of iteration indices where $(s,a)$ in included in the batch.
Now we discuss two possible cases as follows:
\begin{itemize}[leftmargin=*]
    \item Case 1: $\pi^{(k_m)}(a\rvert s)\rightarrow 1$ as $m\rightarrow \infty$
    
    Conditioning on $E_0$, both $I_s^+$ and $I_s^-$ are non-empty. Since $\theta^{(k)}_{s,a'}$ is finite for each $a'\in I_s^-$, we know that $\pi^{(k_m)}(a\rvert s)\rightarrow 1$ implies that
    \begin{equation}
    \label{eq:E1,2 eq 1}
    \theta^{(k_m)}_{s,a}\rightarrow \infty, \text{ as } m\rightarrow \infty.   
    \end{equation}
    Moreover, under CAPO, as $\theta^{(k_m)}_{s,a}$ shall be increasing for all sufficiently large $m$ (given that $a\in I_{s}^{+}$), we know (\ref{eq:E1,2 eq 1}) implies that $\theta^{(k)}_{s,a}\rightarrow \infty, \text{ as } k\rightarrow \infty$.
    Therefore, we have $\lim_{k\rightarrow \infty}\pi^{(k)}(a'\rvert s)=0$, for all $a'\in I_{s}^{-}$.
    \item Case 2: $\pi^{(k_m)}(a\rvert s)\nrightarrow 1$ as $m\rightarrow \infty$:
    Since $A^{(k_m)}(s,a)$ shall be positive for all sufficiently large $m$ (given that $a\in I_{s}^{+}$), we know: (i) If $\pi^{(k_m)}(a\rvert s)\geq \beta$, we have $\theta^{(k_m+1)}_{s,a}-\theta^{(k_m)}_{s,a}\geq\zeta\log(\frac{N^{(k_m)}(s,a)+1)}{N^{(k_m)}(s,a)})=\zeta\log(\frac{m+1}{m})$; (ii) Otherwise, if $\pi^{(k_m)}(a\rvert s)<\beta$, we shall have $\theta^{(k_m+1)}_{s,a}-\theta^{(k_m)}_{s,a}\geq\log(\frac{1}{1-\beta})>\zeta\log(\frac{m+1}{m})$, for all sufficiently large $m$.
    This implies that $\theta^{(k_m)}(s,a)\rightarrow \infty$ as $m\rightarrow \infty$.
    As $\theta^{(k_m)}_{s,a}$ shall be increasing for all sufficiently large $m$ (given that $a\in I_{s}^{+}$), we also have $\theta^{(k)}_{s,a}\rightarrow \infty$, as $k\rightarrow \infty$.
    % Old version (2022/05/11)
    %This means that there exists $\epsilon>0$ such that $\pi^{(k_m)}(a\rvert s)\leq 1-\epsilon$, for infinitely many $m$. Under CAPO, if $\pi^{(k_m)}(a\rvert s)\leq 1-\epsilon$, then $\theta^{(k_m+1)}_{s,a}-\theta^{(k_m)}_{s,a}\geq \log(\frac{1}{1-\epsilon})$.
    %Again, since $\theta^{(k)}_{s,a}$ shall be increasing for all sufficiently large $m$ (given that $a\in I_{s}^{+}$), we shall have $\theta^{(k)}_{s,a}\rightarrow \infty$, as $k\rightarrow \infty$.
    As $\theta^{(k)}_{s,a'}$ remains finite for all $a'\in I_{s}^{-}$, we therefore have that $\lim_{k\rightarrow \infty}\pi^{(k)}(a'\rvert s)=0$, for all $a'\in I_{s}^{-}$.
\end{itemize}
\end{proof}

\begin{lemma}
\label{lemma:E1,3 and E1 tilde}
Under on-policy CAPO and the condition that $\bbP(E_0)>0$, we have $\bbP({E}_1^{c}\cap {E}_{1,3}\rvert E_0)=0$.
\end{lemma}
\begin{proof}[Proof of Lemma \ref{lemma:E1,3 and E1 tilde}]
By Lemma \ref{lemma:E1,1 and E1 tilde} and Lemma \ref{lemma:E1,2 and E1 tilde}, we have $\bbP({E}_1^{c}\cap {E}_{1,3}\rvert E_0)=\bbP({E}_1^{c}\cap E_{1,2}^{c}\cap E_{1,1}^{c} \cap {E}_{1,3}\rvert E_0)$.
Under $E_{1,1}^{c}\cap E_{1,2}^{c}$, we know that any action in $I_{s}^{+}\cup I_{s}^{-}$ can appear in $\cB^{(k)}$ only for finitely many times.
This implies that there exists $T_0\in \mathbb{N}$ such that $\cB^{(k)}$ contains only actions in $I_{s}^{0}$, for all $k\geq T_0$.
In order for the above to happen, we must have $\sum_{a\in I_{s}^{0}}\pi^{(k)}(a\rvert s)\rightarrow 1$, as $k\rightarrow \infty$ (otherwise there would exist some $\epsilon>0$ such that $\sum_{a\in I_{s}^{0}}\pi^{(k)}(a\rvert s)\leq 1-\epsilon$ for infinitely many $k$).
This implies that $\lim_{k\rightarrow\infty} \pi^{(k)}(a'\rvert s)=0$, for any $a'\in I_{s}^{-}$.
Hence, $\bbP({E}_1^{c}\cap E_{1,2}^{c}\cap E_{1,1}^{c} \cap {E}_{1,3}\rvert E_0)=0$.
\end{proof}

%\begin{lemma}
%\label{lemma:E1,4 and E1 tilde}
%nder the condition that $\bbP(E_0)>0$, we have $\bbP({E}_1^{c}\cap {E}_{1,4}\rvert E_0)=0$.
%\end{lemma}
%\begin{proof}[Proof of Lemma \ref{lemma:E1,4 and E1 tilde}]
%\end{proof}

\begin{lemma}
\label{lemma:E1}
Under on-policy CAPO and the condition that $\bbP(E_0)>0$, we have $\bbP(E_1\rvert E_0)=1$.
\end{lemma}
\begin{proof}[Proof of Lemma \ref{lemma:E1}]
By (\ref{eq:E1 tilde complement}), Lemma \ref{lemma:E1,1 and E1 tilde}, Lemma \ref{lemma:E1,2 and E1 tilde}, and Lemma \ref{lemma:E1,3 and E1 tilde}, we know $\bbP(E_1^{c}\rvert E_0)=0$.
\end{proof}

Before we proceed, we define the following events:
\begin{align}
     E_2:=&\Big\{\omega: \lim_{k\rightarrow \infty}\pi_{s,a}^{(k)}(\omega)= 0 , \forall a\in  I_{s}^{+}(\omega)\Big\},\\
     E_3:=&\Big\{\omega:\exists a\in  I_{s}^{+}(\omega) \text{ with } N^{(\infty)}_{s,a}(\omega)=\infty\Big\}
\end{align}

\begin{lemma}
\label{lemma:E2}
Under on-policy CAPO and the condition that $\bbP(E_0)>0$, we have $\bbP(E_2\rvert E_0)=1$.
\end{lemma}
\begin{proof}
This is a direct result of Lemma \ref{lemma:E1}.
\end{proof}

\begin{lemma}
\label{lemma:E2 and E3}
Under on-policy CAPO and the condition that $\bbP(E_0)>0$, we have $\bbP(E_2\cap E_3\rvert E_0)=0$.
\end{lemma}
\begin{proof}
Under the event $E_2$, we know that for each action $a\in I_{s}^{+}$, for any $\epsilon>0$, there exists some $T_{a,\epsilon}$ such that $\pi_{s,a}^{(k)}< \epsilon$ for all $k\geq T_{a,\epsilon}$.
On the other hand, by Lemma \ref{lemma:pi lower bound}, under $E_3$, {we know that $\pi_{s,a}^{(k)}>\beta$ infinitely often}.
%, by an argument similar to that in Proposition \ref{prop:policy_lowerbound}.
Hence, we know $\bbP(E_2\cap E_3\rvert E_0)=0$.
\end{proof}
Note that by Lemma \ref{lemma:E2} and Lemma \ref{lemma:E2 and E3}, we have 
$\bbP(E_2\cap E_3^{c}\rvert E_0)=1$.

The main idea of the proof of Theorem \ref{app:thm:on-policy CAPO} is to establish a contradiction by showing that under $E_0$, $E_3^{c}$ cannot happen with probability one. 
Let us explicitly write down the event $E_3^{c}$ as follows: 
\begin{equation}
    E_3^{c}:=\big\{\omega: \exists \tau(\omega)<\infty \text{ such that  }  \cB^{(k)}\subseteq I_s^{0}\cup I_{s}^{-}, \forall k\geq \tau(\omega)\big\}.
\end{equation}

Define
\begin{equation}
\theta^{(k)}_{s,\max}:=\max_{a\in \cA} \theta^{(k)}_{s,a}.
\end{equation}
\begin{lemma}
\label{lemma:theta diff}
%Under the event $E_0\cap E_1\cap E_2\cap E_3^{c}$, 
For any $t\in \mathbb{N}$ and any $K\in \mathbb{N}$, we have
\begin{equation}
    \theta_{s,\max}^{(t+K)} - \theta_{s,\max}^{(t)} \leq \log (K+1),
\end{equation}
for every sample path.
\end{lemma}
\begin{proof}[Proof of Lemma \ref{lemma:theta diff}]
%By the definition of $\bar{\tau}_s$ in (\ref{eq:tau bar}), we know that the following three properties hold in each iteration $k\geq \bar{\tau}_s$: (i) For all $a\in I_s^{+}$, $A^{(k)}(s,a)>\frac{\Delta_s}{2}$; (ii) For all $a\in I_s^{-}$, $A^{(k)}(s,a)<-\frac{\Delta_s}{2}$; (iii) Only the actions in $I_s^0\cup I_s^-$ would appear in $\cB^{(k)}$.
%We consider the changes of $\theta_{s,a}$ of the actions in $I_s^{0}$ and $I_{s}^{+}$ separately:
We consider the changes of $\theta_{s,a}$ of each action separately:
\begin{itemize}[leftmargin=*]
    %\item $a\in I_s^{-}$: Under on-policy CAPO, we shall have $\theta^{(k+1)}_{s,a}<\theta^{(k)}_{s,a} $ since $A^{(k)}(s,a)<-\frac{\Delta_{s}}{2}<0$.
    %Therefore, this change cannot lead to an increase in $\theta_{s,\max}^{(k)}$.
    \item $\theta^{(k)}_{s,a} <\theta^{(k)}_{s,\max}$, and $\pi^{(k)}(a\rvert s)<\beta$: 
    For such an action $a$, we have
    \begin{align}
        \theta^{(k+1)}_{s,a}&\leq \theta^{(k)}_{s,a}+\log\Big(\frac{\beta}{(1-\beta) \cdot\pi^{(k)}(a\rvert s)} \Big)\label{eq:theta max 1}\\
        &\le \log\Big(\frac{\beta}{(1-\beta)} \cdot\lvert\cA \rvert \exp({\theta^{(k)}_{s,\max}})\Big) ,\label{eq:theta max 2}\\
        &\leq \log\Big(\frac{\frac{1}{\lvert\cA \rvert+1}}{(1-\frac{1}{\lvert\cA \rvert+1})} \cdot\lvert\cA \rvert \exp({\theta^{(k)}_{s,\max}})\Big)\label{eq:theta max 3}\\
        &=\theta^{(k)}_{s,\max},
    \end{align} 
    where (\ref{eq:theta max 1}) holds by the design of on-policy CAPO, (\ref{eq:theta max 2}) follows from the softmax policy parameterization, and (\ref{eq:theta max 3}) follows from the definition of $\theta^{(k)}_{s,\max}$ and the condition of $\beta$.
    Note that (\ref{eq:theta max 1}) would be an equality if $A^{(k)}(s,a)>0$.
    As a result, this change cannot lead to an increase in $\theta_{s,\max}^{(k)}$.
    \item $\theta^{(k)}_{s,a}<\theta^{(k)}_{s,\max}$, and $\pi^{(k)}(a\rvert s)\geq \beta $:     
    For such an action $a$, we have
    \begin{align}
        \theta^{(k+1)}_{s,a}&\leq \theta^{(k)}_{s,a}+\zeta\log\Big( \frac{N^{(k)}_{s,a}+1}{N^{(k)}_{s,a}}\Big),\label{eq:theta max 4}
    \end{align}     
    where (\ref{eq:theta max 4}) holds by the design of on-policy CAPO and would be an equality if $A^{(k)}(s,a)>0$.
    \item $\theta^{(k)}_{s,a}=\theta^{(k)}_{s,\max}$:
    Similarly, we have
    \begin{align}
        \theta^{(k+1)}_{s,a}&\leq \theta^{(k)}_{s,a}+\zeta\log\Big( \frac{N^{(k)}_{s,a}+1}{N^{(k)}_{s,a}}\Big),\label{eq:theta max 5}
    \end{align} 
    where (\ref{eq:theta max 5}) holds by the design of on-policy CAPO and would be an equality if $A^{(k)}(s,a)>0$.
\end{itemize}
Based on the above discussion, we thereby know
    \begin{equation}
        \theta^{(k+1)}_{s,\max}- \theta^{(k)}_{s,\max}\leq \zeta \sum_{a\in \cA}\log\Big( \frac{N^{(k+1)}_{s,a}}{N^{(k)}_{s,a}}\Big), \quad\forall k.\label{eq:theta max 6}
    \end{equation} 
Therefore, for any $t\in\mathbb{N}$, the maximum possible increase in $\theta_{s,\max}^{(k)}$ between the $t$-th and the $(t+K)$-th iterations shall be upper bounded as
\begin{align}
    \theta_{s,\max}^{(t+K)} - \theta_{s,\max}^{(t)}&\leq \sum_{k=t}^{t+K-1}\zeta \sum_{a\in \cA}\log\Big( \frac{N^{(k+1)}_{s,a}}{N^{(k)}_{s,a}}\Big)\label{eq:theta max 7}\\
    &\leq \zeta \cdot \sum_{a\in \cA} \log\Big(\frac{N^{(t)}_{s,a}+K}{N^{(t)}_{s,a}}\Big)\label{eq:theta max 8}\\
    &\leq \log(K+1),\label{eq:theta max 9}
\end{align}
where (\ref{eq:theta max 7}) follows directly from (\ref{eq:theta max 6}), (\ref{eq:theta max 7}) is obtained by interchanging the summation operators, and (\ref{eq:theta max 8}) holds by the condition that $\zeta \leq \frac{1}{\lvert \cA\rvert }$.
%$\sum_{k=\bar{\tau}_{s}}^{\bar{\tau}_{s}+K-1}\log(\frac{k+1}{k})\leq \log K$.
Hence, we know $\theta_{s,\max}^{(t+K)} - \theta_{s,\max}^{(t)} \leq \log (K+1)$.
\end{proof}

For any fixed action set $I_{s}^{\dagger}\subset \cA$, define 
\begin{equation}
    E_4(I_{s}^{\dagger}):=\big\{\omega: \text{ For every } a\in I_{s}^{\dagger}, N^{(\infty)}(s,a)<\infty\big\}.
\end{equation}
\begin{lemma}
\label{lemma:E4}
For any $I_{s}^{\dagger}\subset \cA$, we have $\bbP(E_4(I_{s}^{\dagger}))=0$.
\end{lemma}
\begin{proof}[Proof of \Cref{lemma:E4}]
For a given action set $I_{s}^{\dagger}\subset \cA$, define a sequence of events as follows: For each $n\in \mathbb{N}$,
\begin{equation}
    %E_{3,n}:=\big\{\omega: \cB^{(k)}\subseteq \cA \backslash I_{s}^{+},\forall k\geq n\big\}.
    E_{4,n}(I_{s}^{\dagger}):=\big\{\omega: \text{ For every } a\in I_{s}^{\dagger}, (s,a)\notin\cB^{(k)},\forall k\geq n\big\}.\label{eq:theta max 10}
\end{equation}
$\{E_{4,n}(I_{s}^{\dagger})\}_{n=1}^{\infty}$ form an increasing sequence of events, i.e., $E_{4,1}(I_{s}^{\dagger})\subseteq E_{4,2}(I_{s}^{\dagger})\cdots \subseteq  E_{4,n}(I_{s}^{\dagger}) \subseteq E_{4,n+1}(I_{s}^{\dagger})\cdots$. 

Moreover, we have $E_4(I_{s}^{\dagger})=\bigcup_{n=1}^{\infty} E_{4,n}(I_{s}^{\dagger})$.
By the continuity of probability, we have
\begin{equation}
    \bbP(E_{4}(I_{s}^{\dagger}))=\bbP(\lim_{n\rightarrow\infty} E_{4,n}(I_{s}^{\dagger})) = \lim_{n\rightarrow \infty} \bbP(E_{4,n}(I_{s}^{\dagger})).\label{eq:theta max 11}
\end{equation}

Next, we proceed to evaluate $\bbP(E_{4,n}(I_{s}^{\dagger}))$.
%Let $\cH_n$ denote the history up to iteration $n$.
\begin{align}
    \log\big(\bbP(E_{4,n}(I_{s}^{\dagger}))\big)&\leq\log\bigg(\prod_{k\geq n} \frac{\sum_{a'\in I_{s}^{0}\cup I_{s}^{-}} \exp(\theta_{s,a’}^{(k)})}{\sum_{a'\in I_{s}^{0}\cup I_{s}^{-}} \exp(\theta_{s,a’}^{(k)})+\sum_{a \in I_{s}^{+}} \exp(\theta_{s,a}^{(k)})}\bigg)\label{eq:theta max 12}\\
    &\leq \log\bigg(\prod_{k\geq n} \frac{\lvert \cA\rvert\exp(\theta_{s,\max}^{(k)})}{\lvert \cA\rvert\exp(\theta_{s,\max}^{(k)})+\sum_{a \in I_{s}^{+}} \exp(\theta_{s,a}^{(n)})} \bigg)\label{eq:theta max 13}\\
    &\leq \log\bigg(\prod_{m\geq 1} \frac{\lvert \cA\rvert\exp\big(\theta_{s,\max}^{(n)}+\log(m+1)\big)}{\lvert \cA\rvert\exp\big(\theta_{s,\max}^{(n)}+\log(m+1)\big)+\sum_{a \in I_{s}^{+}} \exp(\theta_{s,a}^{(n)})} \bigg)\label{eq:theta max 14}\\
    &\leq \sum_{m\geq 1} \log\bigg(1-\frac{\sum_{a \in I_{s}^{+}} \exp(\theta_{s,a}^{(n)})}{\lvert \cA\rvert(m+1)\exp(\theta_{s,\max}^{(n)})} \bigg)=-\infty,\label{eq:theta max 15}
\end{align}
where (\ref{eq:theta max 12}) holds by the softmax policy parameterization, (\ref{eq:theta max 13}) holds by the definition of $\theta^{(k)}_{s,\max}$ and $E_{4,n}(I_{s}^{\dagger})$, and (\ref{eq:theta max 13}) follows directly from Lemma \ref{lemma:theta diff}.
Equivalently, we have $\bbP(E_{4,n}(I_{s}^{\dagger}))=0$, for all $n\in\mathbb{N}$.
By (\ref{eq:theta max 11}), we conclude that $\bbP(E_4(I_s^{\dagger}))=0$.
\end{proof}
Now we are ready to prove Theorem \ref{app:thm:on-policy CAPO}.
\begin{proof}[Proof of Theorem \ref{app:thm:on-policy CAPO}]
Recall that the main idea is to establish a contradiction by showing that conditioning on $E_0$, $E_3^{c}$ cannot happen with probability one. 
Note that by Lemma \ref{lemma:E2} and Lemma \ref{lemma:E2 and E3}, we have 
$\bbP(E_2\cap E_3^{c}\rvert E_0)=1$.
However, by Lemma \ref{lemma:E4}, we know that for any fixed action set $I_{s}^{\dagger}\subset \cA$, the event that the actions in $I_{s}^{\dagger}$ are selected for policy updates for only finitely many times must happen with probability zero. This contradicts the result in Lemma \ref{lemma:E2 and E3}.
Therefore, we shall have $\bbP(E_0)=0$.
\end{proof}
% \begin{corollary}
% In the one-state MDP example,

% \end{corollary}
\subsection{On-Policy CAPO with Fixed Learning Rate}
\label{app:OnCAPO:onCAPO_fixed}
One interesting question is whether on-policy CAPO can be applied with a fixed learning rate. Through a simple single state bandit example, we show that without the help of variable learning rate, on-policy CAPO with fixed learning rate will stuck in local optimum with positive probability.
Therefore, this fact further motivates the use of variable learning rate in CAPO.
We provide the detailed discussion in \Cref{app:example}.

%% file: appendices/appendix_OnCAPO_fixed.tex
\section{Sub-Optimality of On-Policy CAPO Due to Improper Step Sizes}
\label{app:example}
In this section, we construct a toy example to further showcase how the proposed CAPO benefits from the properly-designed step sizes in Algorithm \ref{algo:CAPO}.
%\subsection{Example: Multi-Armed Bandit}
%\label{SGD:multi-arm-bandit}
We consider a deterministic $K$-armed bandit with a single state and an action set $\left[K\right]$ and a softmax policy $\pi_{\theta}: [K] \rightarrow [0, 1]$, the reward vector $r \in \mathbb{R}^{K}$ is the reward corresponding to each action. This setting is the same as the one in Section 2 of \citep{mei2021understanding}, except that we do not have the assumption of positive rewards such that $r(a) \in [0, 1), \forall a \in [K]$, the reward can be any real number such that $r \in \mathbb{R}^K$.
Our goal here is to find the optimal policy $\pi^{*}$ that maximize the expected total reward. Since there is only one single state, the objective function can by written as:
\begin{equation}
J(\theta) = \mathbb{E}_{{a \sim \pi_{\theta}(\cdot)}}[r(a)].
\end{equation}

The on-policy CAPO with fixed learning rate updates the policy parameters by:
\begin{equation}
\label{eq:onCAPO_fixed_update}
\theta_{m+1}(s, a) = \theta_{m}(s, a) + \eta \cdot \sign(A(s,a)) \cdot \mathbb{I}\{a = a_m\}
\end{equation}
where $\eta$ is a constant representing the fixed learning rate.

To demonstrate that on-policy CAPO with fixed learning rate can get stuck in a sub-optimal policy, we consider a simple three-armed bandit where $K=3$ (i.e. a single state with 3 actions). We set $r = [1, 0.99, -1]$. Then we have:
\begin{theorem} Given a uniform initial policy $\pi_1$ such that $\pi_{1}(a) = \frac{1}{K}, \forall a \in [K]$, under the policy update of (\ref{eq:onCAPO_fixed_update}), we have $\bbP(\pi_{\infty}(a_2) = 1) > 0$.
\label{theorem:onCAPO_fixed_stuck}
\end{theorem}
The idea is that with $\pi_1(a_1) = \pi_1(a_3)$ and $r(a_1) = -r(a_3)$, when we only sample $a_2$ in the first $t$ steps, $A_m(a_2) > 0, \forall m \le t$.
Thus, $\pi_{m}(a_2)$ shall be strictly improving, and the probability of sampling $a_2$ will increase accordingly, thus causing a vicious cycle. 
%Detailed proof can be found in Appendix \ref{app:proof:onCAPO_fixed}.

\Cref{theorem:onCAPO_fixed_stuck} shows that the naive fixed learning rate is insufficient. In the next section, we will show that with a properly chosen variable learning rate, on-policy CAPO can guarantee global convergence. Empirical results can be found in \Cref{app:extra_exp:bandit}.

\begin{proof}[Proof of Theorem \ref{theorem:onCAPO_fixed_stuck}]
\label{app:proof:onCAPO_fixed}

Inspired by the proof in \citep{mei2021understanding} (Theorem 3, second part), we also consider the event $\mathcal{E}_{t}$ such that $a_2$ is chosen in the first $t$ time steps. We will show that there exists some sequence $b_s$ such that $\bbP(\mathcal{E}_{t}) \ge \prod_{s=1}^{t} b_{s} > 0$.

The first part argument is the same as \citep{mei2021understanding}, we restate the argument for completeness:
Let $\mathcal{B}_m=\left\{a_m = a_2 \right\}$ be the event that $a_2$ is sampled at time $m$.
Define the event $\mathcal{E}_{t} = \mathcal{B}_{1} \cap \cdots \cap \mathcal{B}_{t}$ be the event that $a_2$ is chosen in the first $t$ time steps. Since $\left\{\mathcal{E}_{t}\right\}_{t \geq 1}$ is a nested sequence, we have $\lim _{t \rightarrow \infty} \bbP\left(\mathcal{E}_{t}\right)=\bbP(\mathcal{E})$ by monotone convergence theorem.
Following equation (197) and equation (198) in \citep{mei2021understanding}, we will show that a suitable choice of $b_t$ under the On-policy CAPO with fixed learning rate is:
\begin{equation}
b_t = \exp\left\{- \frac{\sum_{a \neq a_2}\exp\left\{\theta_1(a)\right\}}{\exp\left\{\theta_1(a_2)\right\}} \cdot \frac{\exp\left\{\eta\right\}}{\eta} \right\}.
\end{equation}

\begin{lemma} $\pi_m(a_1) = \pi_m(a_3), \forall 1 \le m \le t$.
\label{prop:OnCAPO_fixed:theta_13}
\end{lemma}
\begin{proof}[Proof of \Cref{prop:OnCAPO_fixed:theta_13}]
Under uniform initialization $\theta_1(a_1) = \theta_1(a_3)$, since only $a_2$ is sampled in the first $t$ steps, we have $\forall 1 \le m \le t$:
\begin{align}
&\pi_m(a_1)  = \frac{\exp(\theta_m(a_1))}{\sum_{a}\exp(\theta_m(a))} \\
&=  \frac{\exp(\theta_1(a_1))}{\sum_{a}\exp(\theta_m(a))}  =  \frac{\exp(\theta_1(a_3))}{\sum_{a}\exp(\theta_m(a))} \\
&= \pi_m(a_3).
\end{align}
\end{proof}

\begin{lemma} For all $1 \le m \le t$, we have $A_m(a_2) \ge 0$.
\label{lemma:OnCAPO_fixed:pos_A}
\end{lemma}
\begin{proof}[Proof of \Cref{lemma:OnCAPO_fixed:pos_A}]
Note that under the CAPO update (\ref{eq:onCAPO_fixed_update}), we have
\begin{align}
&A_m(a_2) = r(a_2) - \sum_{a}\pi_m(a) \cdot r(a) \\
&\quad= (1-\pi_m(a_2))r(a_2) - \sum_{a \neq a_2}\pi_m(a) \cdot r(a) \\
&\quad= (1-\pi_m(a_2))r(a_2) - \sum_{a \neq a_2}\pi_m(a) \cdot r(a) \\
&\quad= (1-\pi_m(a_2))r(a_2) \ge 0,
\end{align}
where the last equation comes from \Cref{prop:OnCAPO_fixed:theta_13} and $r(a_1) = -1 \cdot r(a_3)$.

\end{proof}
\begin{lemma} $\theta_t(a_2) = \theta_1(a_2) + \eta \cdot (t-1)$.
\label{lemma:OnCAPO_fixed:theta t-step diff}
\end{lemma}
\begin{proof}[Proof of \Cref{lemma:OnCAPO_fixed:theta t-step diff}]
By \Cref{lemma:OnCAPO_fixed:pos_A} and (\ref{eq:onCAPO_fixed_update}), we have:
\begin{align}
&\theta_t(a_2) = \theta_1(a_2) + \eta \cdot \sum_{s=1}^{t-1} \sign(A_s(a_2)) \cdot \mathbb{I}\{a_2 = a_s\} \\
&= \theta_1(a_2) + \eta \cdot \sum_{s=1}^{t-1} 1 \\
&= \theta_1(a_2) + \eta \cdot (t-1).\\
\end{align}
\end{proof}

\begin{lemma} For all $x \in(0,1)$, we have:
\label{lemma:exp_lower_bound}
\end{lemma}
\begin{equation}
\label{eq:exp_lower_bound}
1-x \geq \exp \left\{\frac{-x}{1 - x}\right\}
\end{equation}
\begin{proof}[Proof of \Cref{lemma:exp_lower_bound}]
This is a direct result of Lemma 14 in \citep{mei2021understanding}. Here we also include the proof for completeness.
\begin{align}
1-x &=\exp \{\log (1-x)\} \\
& \geq \exp \left\{1-e^{-\log (1-x)}\right\} \quad\left(y \geq 1-e^{-y}\right) \\
&=\exp \left\{\frac{-1}{1 / x-1}\right\}
\\
&=\exp \left\{\frac{-x}{1 - x}\right\}.
\end{align}
Then, we can plug in $x$ as $\frac{a}{b}$ for some $a < b$ to obtain a more useful form of this lemma as follows:
\begin{equation}
\label{eq:exp_lower_bound2}
1 -\frac{a}{b} \ge \exp \left\{\frac{-a}{b-a} \right\}.
\end{equation}
\end{proof}

\begin{lemma} $\pi_t(a_2) \ge \exp\left\{\frac{-\sum_{a \neq a_2}\exp\left\{\theta_t(a)\right\} }{\exp\left\{\theta_t(a_2) \right\}}\right\}$.
\label{lemma:bandit positive prob}
\end{lemma}
\begin{proof}[Proof of \Cref{lemma:bandit positive prob}]
\begin{align}
&\pi_t(a_2) = 1 - \sum_{a \neq a_2}\pi_t(a) \\
&= 1 - \frac{\sum_{a \neq a_2}\exp\left\{\theta_t(a)\right\} }{\exp\left\{\theta_t(a_2) \right\} + \sum_{a \neq a_2}\exp\left\{\theta_t(a)\right\}} \\
&\ge \exp\left\{\frac{-\sum_{a \neq a_2}\exp\left\{\theta_t(a)\right\} }{\exp\left\{\theta_t(a_2) \right\}}\right\},
\end{align}
where the last inequality uses (\ref{eq:exp_lower_bound2}).
\end{proof}

Finally, we have
\begin{align}
&\prod_{t=1}^{\infty} \pi_t(a_2) \ge \prod_{t=1}^{\infty} \exp\left\{\frac{-\sum_{a \neq a_2}\exp\left\{\theta_t(a)\right\} }{\exp\left\{\theta_t(a_2) \right\}}\right\} \\
&= \prod_{t=1}^{\infty} \exp\left\{\frac{-\sum_{a \neq a_2}\exp\left\{\theta_1(a)\right\} }{\exp\left\{\theta_1(a_2) + \eta \cdot (t-1) \right\}}\right\} \\
&= \exp\left\{ \sum_{t=1}^{\infty} \frac{-\sum_{a \neq a_2}\exp\left\{\theta_1(a)\right\} }{\exp\left\{\theta_1(a_2) + \eta \cdot (t-1) \right\}}  \right\} \\
&= \exp\left\{- \frac{\sum_{a \neq a_2}\exp\left\{\theta_1(a)\right\}}{\exp\left\{\theta_1(a_2)\right\}} \cdot \exp\left\{\eta\right\} \cdot \sum_{t=1}^{\infty}\frac{1} {\exp\left\{\eta \cdot t \right\}}  \right\} \\
&\ge \exp\left\{- \frac{\sum_{a \neq a_2}\exp\left\{\theta_1(a)\right\}}{\exp\left\{\theta_1(a_2)\right\}} \cdot \exp\left\{\eta\right\} \cdot \int_{t=0}^{\infty}\frac{1} {\exp\left\{\eta \cdot t \right\}}  \right\} \\
&= \exp\left\{- \frac{\sum_{a \neq a_2}\exp\left\{\theta_1(a)\right\}}{\exp\left\{\theta_1(a_2)\right\}} \cdot \frac{\exp\left\{\eta\right\}}{\eta} \right\} \\
&= \Omega(1),
\end{align}
where the last line comes from the fact that $\sum_{a \neq a_2}\exp\left\{\theta_1(a)\right\} \in \Theta(1)$, $\exp\left\{\theta_1(a_2)\right\} \in \Theta(1)$ and $\frac{\exp\left\{\eta\right\}}{\eta} \in \Theta(1)$.
\end{proof}

%% file: appendices/appendix_bandit_exp.tex
\section{A Closer Look at the Learning Rate}
\label{app:extra_exp:bandit}
Unlike most RL algorithms, CAPO leverages variable learning rate that is state action dependent, instead of a fixed learning rate. 
In this section, we provide some insights into why this design is preferred under CAPO from both theoretical and empirical perspectives.
\subsection{Variable Learning Rate v.s. Fixed Learning Rate}
%In this subsection, we give the high-level concept of the reason why we choose the variable learning rate rather than the fixed learning rate.\\
In \Cref{lemma:lower_bdd}, we quantify the one-step improvement $V^{\pi_{m+1}}(s) - V^{\pi_{m}}(s)$ in terms of state visitation distribution, policy weight, and advantage value under learning rate ${\alpha_m(s, a) \ge \log (\frac{1}{\pi_{\theta_m}(a\rvert s)})}$. Now, we provide the one-step improvement \textit{under fixed learning rate}, $\alpha \in \mathbb{R}$, $\alpha > 0$:
\begin{align}
V^{\pi_{m+1}}(s) - V^{\pi_{m}}(s) =
\begin{cases}
\frac{d^{\pi_{m+1}}_{s}(s_m)}{1-\gamma } \cdot \frac{(e^{\alpha}-1) \cdot \pi_m(a_m|s_m)}{(e^{\alpha}-1) \cdot \pi_m(a_m|s_m) + 1} \cdot A^{m}(s_m, a_m)  & \text{, if } A^{m}(s_m, a_m) > 0 \\
\frac{d^{\pi_{m+1}}_{s}(s_m)}{1-\gamma } \cdot \frac{(1-e^{-\alpha}) \cdot \pi_m(a_m|s_m)}{(e^{-\alpha}-1) \cdot \pi_m(a_m|s_m) + 1} \cdot (-A^{m}(s_m, a_m))  & \text{, if } A^{m}(s_m, a_m) <  0\\
\end{cases}\\
\text{where $\alpha \in \mathbb{R}$, $\alpha > 0$}
\end{align}
Note that the result above can be obtained by using the same technique in \Cref{lemma:change_of_policy_weight_1}, \Cref{lemma:change_of_policy_weight_2} and \Cref{lemma:lower_bdd} by substituting the learning rate.

Compared to the one-step improvement under the variable learning rate, the one-step improvement under the fixed learning rate would be tiny as the updated action's policy weight $\pi_m(a_m|s_m) \rightarrow 0$. This property makes it difficult for an action that has positive advantage value but small policy weight to contribute enough to overall improvement, i.e., for those actions, the improvement of the policy weight $\pi_{m+1}(a_m|s_m) - \pi_{m}(a_m|s_m) \rightarrow 0$ under some improper fixed learning rate, leading to small one-step improvement.\\

Now, to provide some further insights into the possible disadvantage of a fixed learning rate, we revisit the proof of the convergence rate of Cyclic CAPO in \Cref{subsection:cyclic_CAPO}. 
By combining the one-step improvement above, the result from \hyperref[item:case1]{Case 1} and \hyperref[item:case2]{Case 2} under the fixed learning rate, $\alpha \in \mathbb{R}$, $\alpha > 0$ can be rewritten as:
\begin{align}
V^{\pi_{m+|\mathcal{S}||\mathcal{A}|}}(s) - V^{\pi_{m}}(s)
&\ge \frac{(1-\gamma)^2}{2} \cdot \frac{1}{{\max} \left \{ \frac{ (1-\pi_{m+T}(a_{m+T}|s_{m+T})) \cdot (e^{\alpha}-1 ) \cdot \pi_{m+T}(a_{m+T}|s_{m+T}) + 1 }{ (1-\gamma) \cdot (e^{\alpha}-1 ) \cdot \pi_{m+T}(a_{m+T}|s_{m+T}) } , \frac{c_m \cdot T}{(1-\gamma)^2} \right \}}  \cdot \left ( V^{*}(s) - V^{\pi_m}(s) \right )^2&&\\
&=\frac{(1-\gamma)^2}{2} \cdot {\min} \left \{ \frac{ (1-\gamma) \cdot (e^{\alpha}-1 ) \cdot \pi_{m+T}(a_{m+T}|s_{m+T}) }{ (1-\pi_{m+T}(a_{m+T}|s_{m+T})) \cdot (e^{\alpha}-1 ) \cdot \pi_{m+T}(a_{m+T}|s_{m+T}) + 1 } , \frac{(1-\gamma)^2}{c_m \cdot T} \right \}&&\\
& \quad \cdot \left ( V^{*}(s) - V^{\pi_m}(s) \right )^2 &&
\end{align}
where $c_m = \underset{k \in [m,m+T-1]}{\max} \left \{ c_{k1}, c_{k2} \right \} \in [0,1]$ \\
and $c_{k1} = \mathbbm{1} \left \{ A^k(s_k,a_k)>0 \right \} \cdot d^{\pi_{k+1}}_{s}(s_k) \cdot \frac{(e^{\alpha}-1 ) \cdot \pi_{k}(a_k|s_k) \cdot (1-\pi_{k}(a_k|s_k)) }{ (e^{\alpha}-1 ) \cdot \pi_{k}(a_k|s_k) + 1 }$, $c_{k2} = \mathbbm{1} \left \{ A^k(s_k,a_k) < 0 \right \} \cdot d^{\pi_{k+1}}_{s}(s_k) \cdot \frac{(1-e^{-\alpha} ) \cdot \pi_{k}(a_k|s_k) \cdot (1-\pi_{k}(a_k|s_k)) }{ (e^{-\alpha}-1 ) \cdot \pi_{m+T}(a_m|s_m) + 1 }$.\\

Note that the first term $\frac{ (1-\gamma) \cdot (e^{\alpha}-1 ) \cdot \pi_{m+T}(a_{m+T}|s_{m+T}) }{ (1-\pi_{m+T}(a_{m+T}|s_{m+T})) \cdot (e^{\alpha}-1 ) \cdot \pi_{m+T}(a_{m+T}|s_{m+T}) + 1 }$ in the ``min'' operator is derived from \hyperref[item:case2]{Case 2} and the second term $\frac{(1-\gamma)^2}{c_m \cdot T}$ is derived from \hyperref[item:case1]{Case 1}. 
Once we cannot guarantee that \hyperref[item:case1]{Case 1} provide enough amount of improvement, we must show that we can get the rest of the required improvement in \hyperref[item:case2]{Case 2}. However, we can find that there is a term $\pi_{m+T}(a_{m+T}|s_{m+T})$ in the numerator of the first term in the ``min'' operator, which is provided by \hyperref[item:case2]{Case 2}, implying that the multi-step improvement $V^{\pi_{m+|\mathcal{S}||\mathcal{A}|}}(s) - V^{\pi_{m}}(s)$ might also be tiny when the improvement provided by \hyperref[item:case1]{Case 1} is insufficient and the policy weight $\pi_{m+T}(a_{m+T}|s_{m+T}) \rightarrow 0$ in \hyperref[item:case2]{Case 2}.\\

Accordingly, we highlight the importance of the choice of the learning rate, especially when the visitation frequency of the coordinate generator is extremely unbalanced (e.g. sampling the optimal action every $(|\mathcal{S}||\mathcal{A}|)^{1000}$ epoch) or the approximated advantage value  is oscillating between positive and negative during the update. The design of the variable learning rate ${\alpha_m(s, a) \ge \log (\frac{1}{\pi_{\theta_m}(a\rvert s)})}$ somehow tackles the difficulty of the insufficient one-step improvement by providing larger step size to the action with tiny policy weight, solving the problem of small improvement of the policy weight. Therefore, we can conclude that under this design of the learning rate, the one-step improvement is more steady with the policy weight of the action chosen for policy update.

\subsection{Demonstrating the Effect of Learning Rate in a Simple Bandit Environment}
In this section, we present the comparison in terms of the empirical convergence behavior of On-policy CAPO and Off-policy CAPO. Specifically, we evaluate the following four algorithms: (i) On-Policy CAPO with state-action-dependent learning rate (cf. (\ref{eq:on-policy CAPO update})), (ii) On-Policy CAPO with fixed learning rate (\ref{eq:onCAPO_fixed_update}),  (iii) Off-Policy CAPO with state-action-dependent learning rate (cf. (\ref{eq:CAPO_form})),  (iv) Off-Policy CAPO with fixed learning rate.

%between Off-policy CAPO with state-action dependent learning rate (Off-CAPO with $\alpha(s, a)$, \eqref{eq:CAPO_update_alpha}), On-policy CAPO with state-action dependent learning rate (On-CAPO with $\alpha(s, a)$, \eqref{eq:on-policy CAPO update} ), On-policy CAPO with fixed learning rate (\eqref{eq:onCAPO_fixed_update}), and stochastic NPG (\eqref{eq:snpg_bandit_update}) under the multi-arm bandit setting in \cref{SGD:multi-arm-bandit}. 

We consider the multi-armed bandit as in \Cref{app:example} with $K=4$, and $r = [10, 9.9, 9.9, 0]$. To further demonstrate the ability of CAPO in escaping from the sub-optimal policies, instead of considering the uniform initial policy where $\pi_1(a) = \frac{1}{K}, \forall a \in [K]$, we initialize the policy to a policy that already prefers the sub-optimal actions ($a_2, a_3$) such that $\theta_1 = [0, 3, 3, 0]$ and $\pi_1 \approx [0.0237, 0.4762, 0.4762, 0.0237]$ under the softmax parameterization. For each algorithm, we run the experiments under 100 random seeds. For all the variants of CAPO, we set $\lvert B_m\rvert=1$.

In \Cref{fig:bandit}, On-policy CAPO with fixed learning rate can get stuck in a sub-optimal policy due to the skewed policy initialization that leads to insufficient visitation to each action, and this serves an example for demonstrating the effect described in \Cref{theorem:onCAPO_fixed_stuck}. On the other hand, on-policy CAPO with state-action dependent learning rate always converges to the global optimum despite the extremely skewed policy initialization. This corroborates the importance of variable learning rate for on-policy CAPO. Without such design, the policies failed to escape from a sub-optimal policy under all the random seeds.

Next, we look at the result of off-policy CAPO: We noticed that off-policy CAPO with fixed learning rate is able to identify the optimal action. However, Off-policy CAPO with fixed learning rate learns much more slowly than its variable learning rate counterpart (notice that the x-axis (Iteration) in each graph is scaled differently for better visualization). 
Also, we notice that the different choices of fixed learning rate have direct impact on the learning speed, and this introduces a hyperparameter that is dependent on the MDP. On the other hand, $\alpha_m(s, a)$ can be used as a general learning rate for different cases (For example, in \Cref{app:generator} where a different environment Chain is introduced, learning rate for off-policy Actor Critic has to be tuned while $\alpha_m(s,a)$ can be used as the go-to learning rate.)
\begin{figure}[ht]
\centering
  \includegraphics[width=0.4\textwidth]{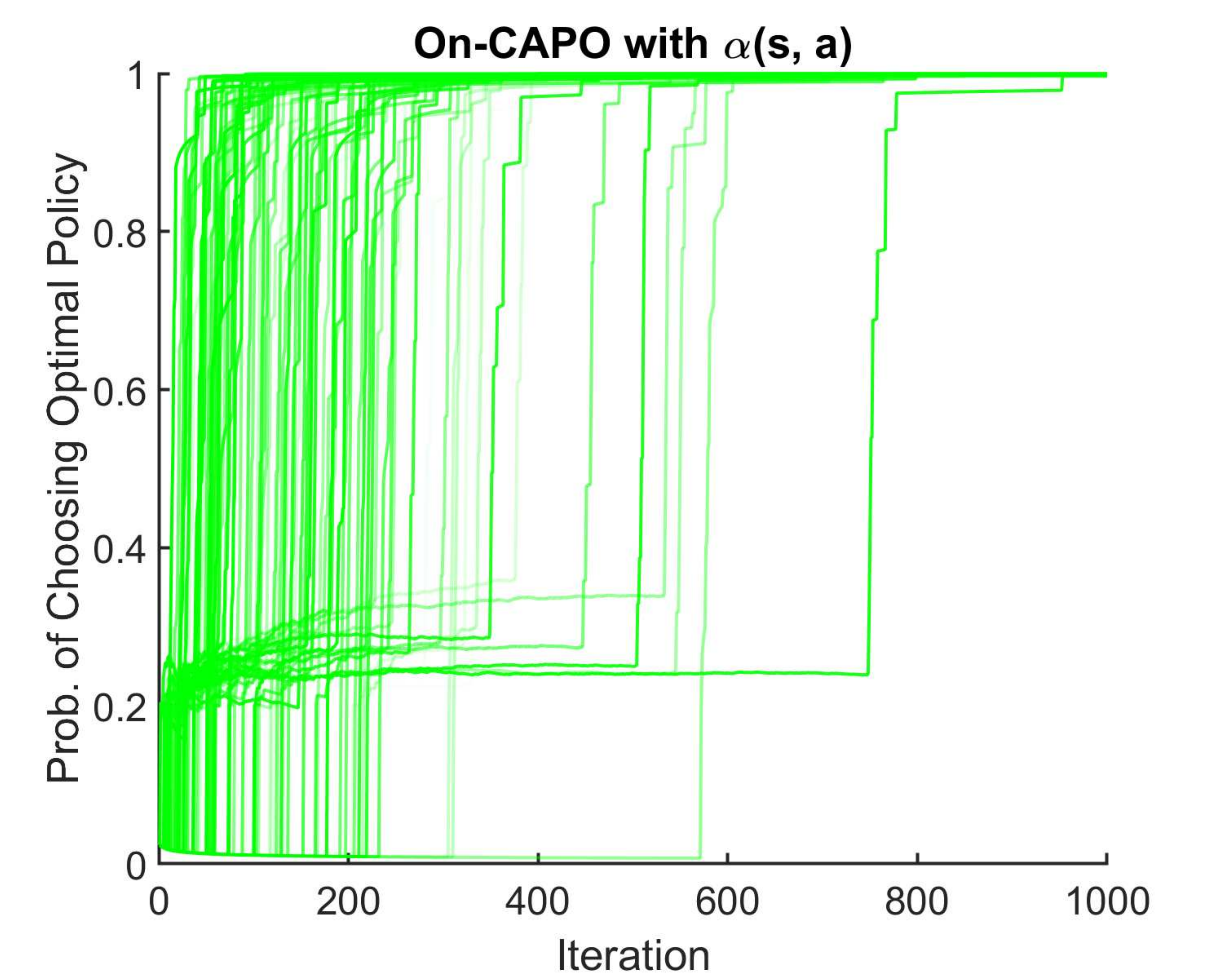}
  \includegraphics[width=0.4\textwidth]{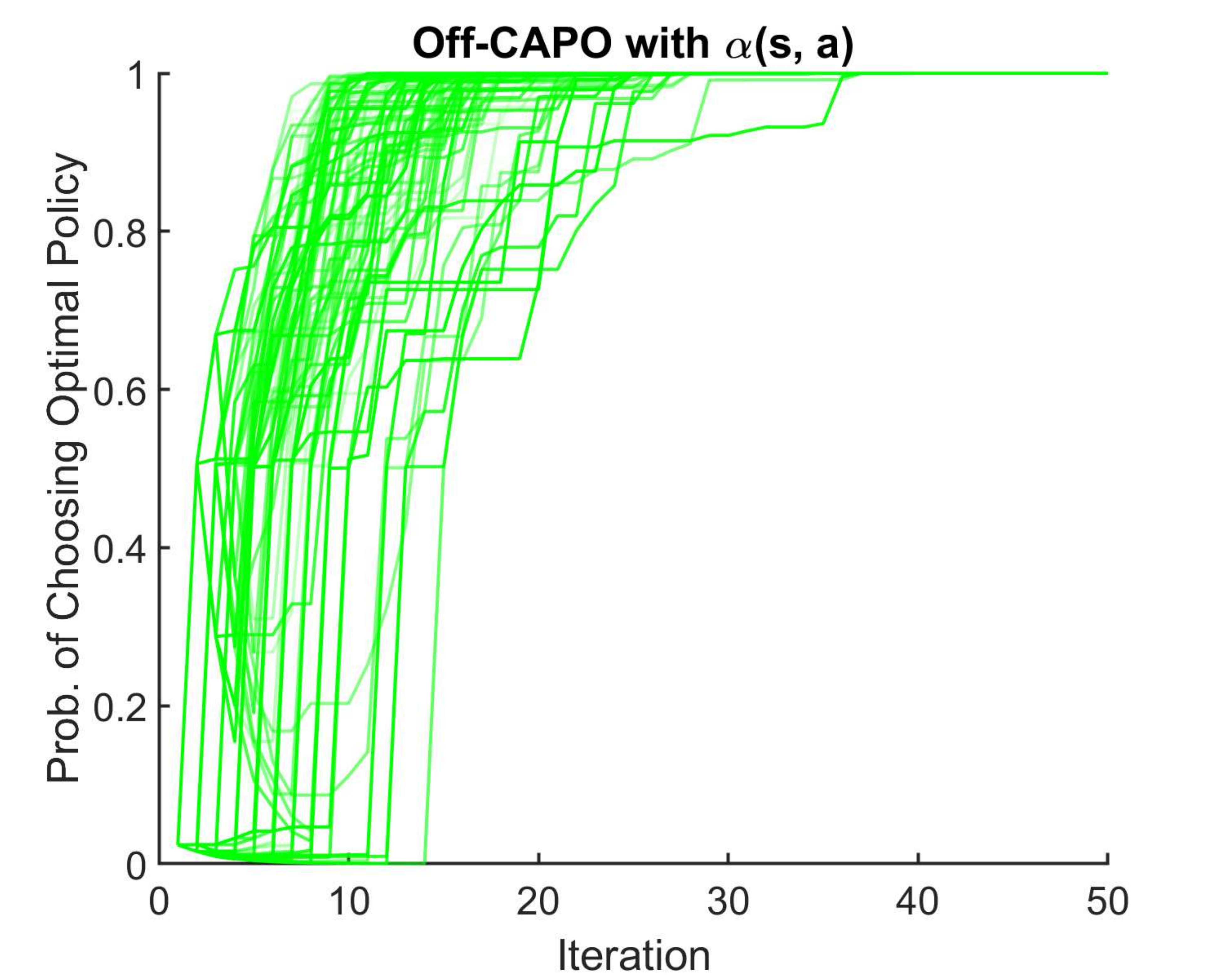}
  \includegraphics[width=0.4\textwidth]{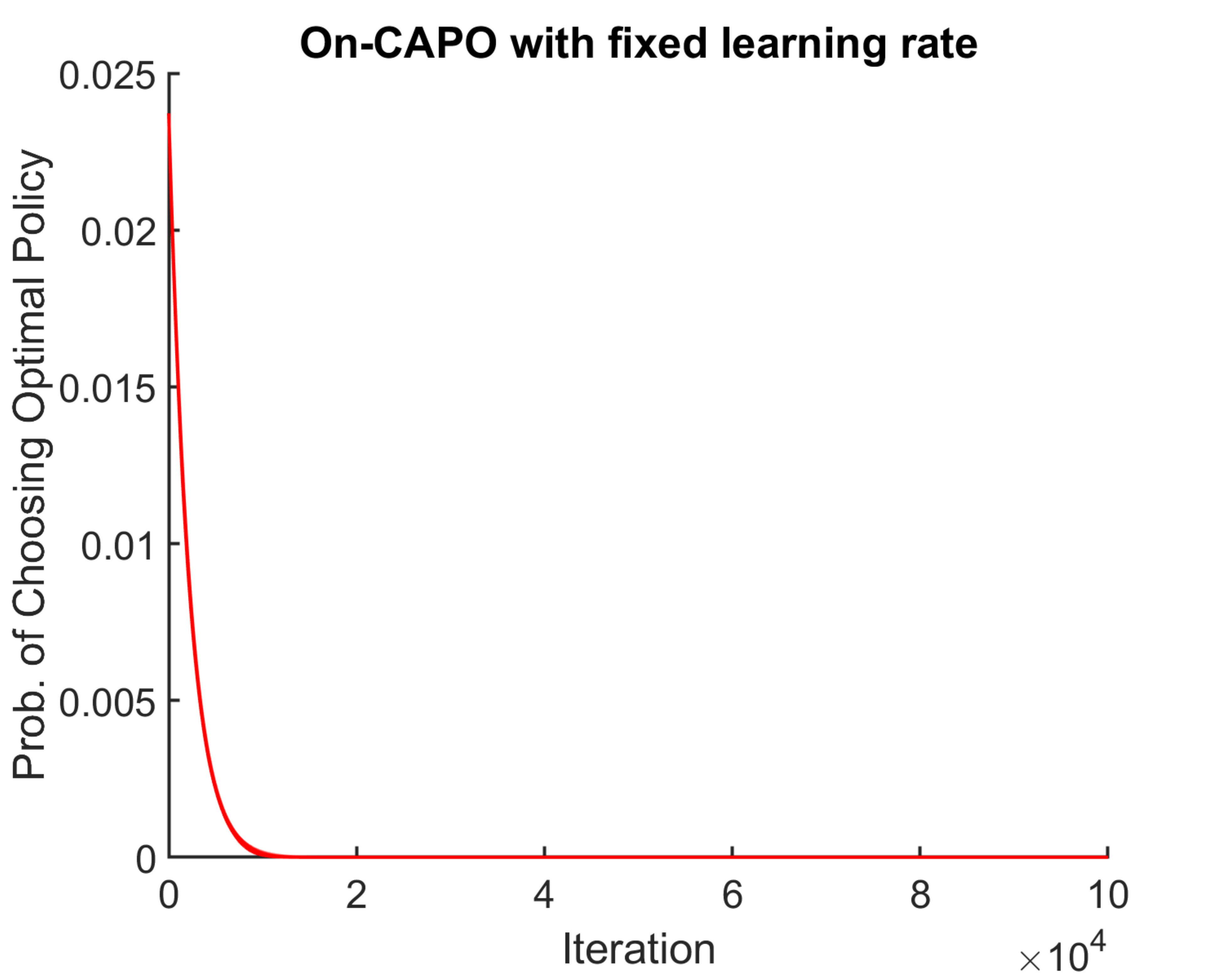}
  \includegraphics[width=0.4\textwidth]{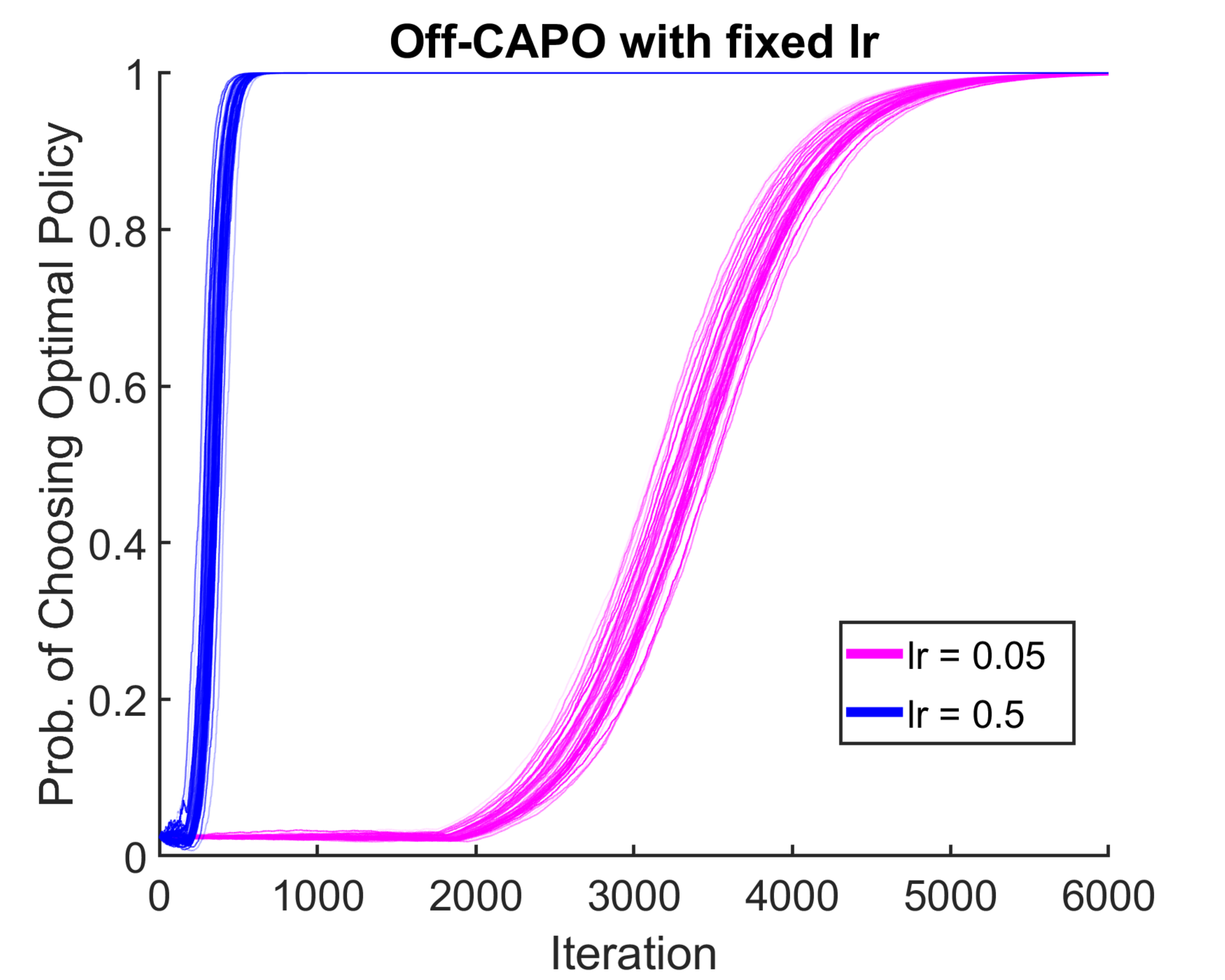}
  \caption{The probability weight of the trained policies on the optimal action at different iterations.}
  \label{fig:bandit}
\end{figure}

%Additional experiments can be found in \ref{app:extra_exp:bandit}.

%% file: appendices/appendix_generator.tex
\section{Exploration Capability Provided by a Coordinate Generator in CAPO}
In this section, we demonstrate empirically the exploration capability provided by the coordinate generator in CAPO.
\label{app:generator}
\subsection{Configuration}
The Chain environment is visualized in \Cref{fig:chain_env}.
This environment is meant to evaluate the agent's ability to resist the temptation of immediate reward and look for the better long-term return.
We compare the performance of Batch CAPO, Cyclic CAPO and Off-policy Actor Critic on Chain with $N=10$, and the result can be found in \Cref{fig:generator}. 
To eliminate the factor of critic estimation, true value of the value function is used during training. All the agents are trained for 1000 iterations with learning rate $=0.001$. The policies are represented by a neural network with a single hidden layer (with hidden layer size 256). 
Both Cyclic CAPO and Off-policy Actor Critic is trained with a batch size of 16 and a replay buffer size of 100. As Batch CAPO shall take all the $\mathcal{S}\mathcal{A}$-pairs into account by design, the effective batch size of Batch CAPO is equal to $\mathcal{S} \times \mathcal{A}$. Unlike the CAPO methods, Off-policy Actor Critic presumes the use of a fixed behavior policy. As a result, similar to the experimental setup of various prior works (e.g., \citep{liu2020off}), we use a uniform behavior policy for Off-policy Actor Critic. 
The use of a fixed behavior policy makes it difficult to identify an optimal policy, and this highlights the benefit of a coordinate generator in terms of exploration.

\subsection{Discussion}
From \Cref{fig:generator} we can see that it is difficult for Off-policy Actor Critic to escape from a sub-optimal policy, despite that the true value of the value function is provided. Since both Cyclic CAPO and Batch CAPO satisfy \Cref{condition:sa}, using such coordinate selection rules provides sufficient exploration for CAPO to identify the optimal policy. This feature can be particularly useful when the reward is sparse and the trajectory is long. 
%as CAPO disregards both problems by coordinate-wise update.

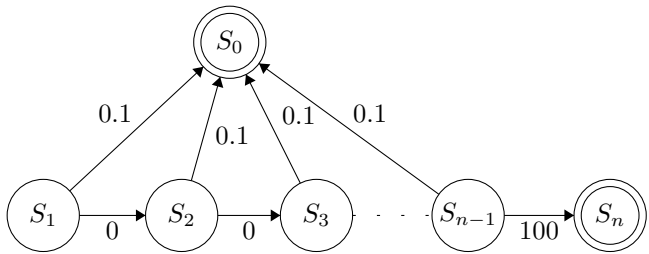
\begin{figure}[h]
\centering
\begin{tikzpicture}[scale=0.16]
\tikzstyle{every node}+=[inner sep=0pt]
\draw [black] (4.9,-35.5) circle (3);
\draw (4.9,-35.5) node {$S_1$};
\draw [black] (16.4,-35.5) circle (3);
\draw (16.4,-35.5) node {$S_2$};
\draw [black] (20.4,-21.1) circle (3);
\draw (20.4,-21.1) node {$S_0$};
\draw [black] (20.4,-21.1)circle (2.4);
\draw [black] (52,-35.5) circle (3);
\draw (52,-35.5) node {$S_n$};
\draw [black] (52,-35.5) circle (2.4);
\draw [black] (40.2,-35.5) circle (3);
\draw (40.2,-35.5) node {$S_{n-1}$};
\draw [black] (27.6,-35.5) circle (3);
\draw (27.6,-35.5) node {$S_3$};
\draw [black] (7.1,-33.46) -- (18.2,-23.14);
\fill [black] (18.2,-23.14) -- (17.28,-23.32) -- (17.96,-24.05);
\draw (10.88,-27.81) node [above] {$0.1$};
\draw [black] (17.2,-32.61) -- (19.6,-23.99);
\fill [black] (19.6,-23.99) -- (18.9,-24.63) -- (19.86,-24.9);
\draw (19.17,-28.84) node [right] {$0.1$};
\draw [black] (7.9,-35.5) -- (13.4,-35.5);
\fill [black] (13.4,-35.5) -- (12.6,-35) -- (12.6,-36);
\draw (10.65,-36) node [below] {$0$};
\draw [black] (43.2,-35.5) -- (49,-35.5);
\fill [black] (49,-35.5) -- (48.2,-35) -- (48.2,-36);
\draw (46.1,-36) node [below] {$100$};
\draw [black] (37.77,-33.74) -- (22.83,-22.86);
\fill [black] (22.83,-22.86) -- (23.18,-23.74) -- (23.77,-22.93);
\draw (32.05,-27.8) node [above] {$0.1$};
\draw [black] (19.4,-35.5) -- (24.6,-35.5);
\fill [black] (24.6,-35.5) -- (23.8,-35) -- (23.8,-36);
\draw (22,-36) node [below] {$0$};
\draw [black] (26.26,-32.82) -- (21.74,-23.78);
\fill [black] (21.74,-23.78) -- (21.65,-24.72) -- (22.55,-24.28);
\draw (24.7,-27.19) node [right] {$0.1$};
\draw [dash pattern=on 3*\pgflinewidth off 8pt] (30.6,-35.5) -- (37.2,-35.5);
\end{tikzpicture}
\caption{The Chain environment has a total of $n+1$ states, and the agent always starts at state $1$. The agent has two actions to choose from at every state, either receive a reward of $0.1$ and terminate immediately, or move one state to the right. While moving right will receive no reward in most states, the transition from $S_{n-1}$ to $S_{n}$ would induce a huge reward of $100$. A well-performing policy should prefer the delayed reward of 100 over the immediate reward of 0.1.}
\label{fig:chain_env}
\end{figure}

\begin{figure}[ht]
\centering
  \includegraphics[width=0.4\textwidth]{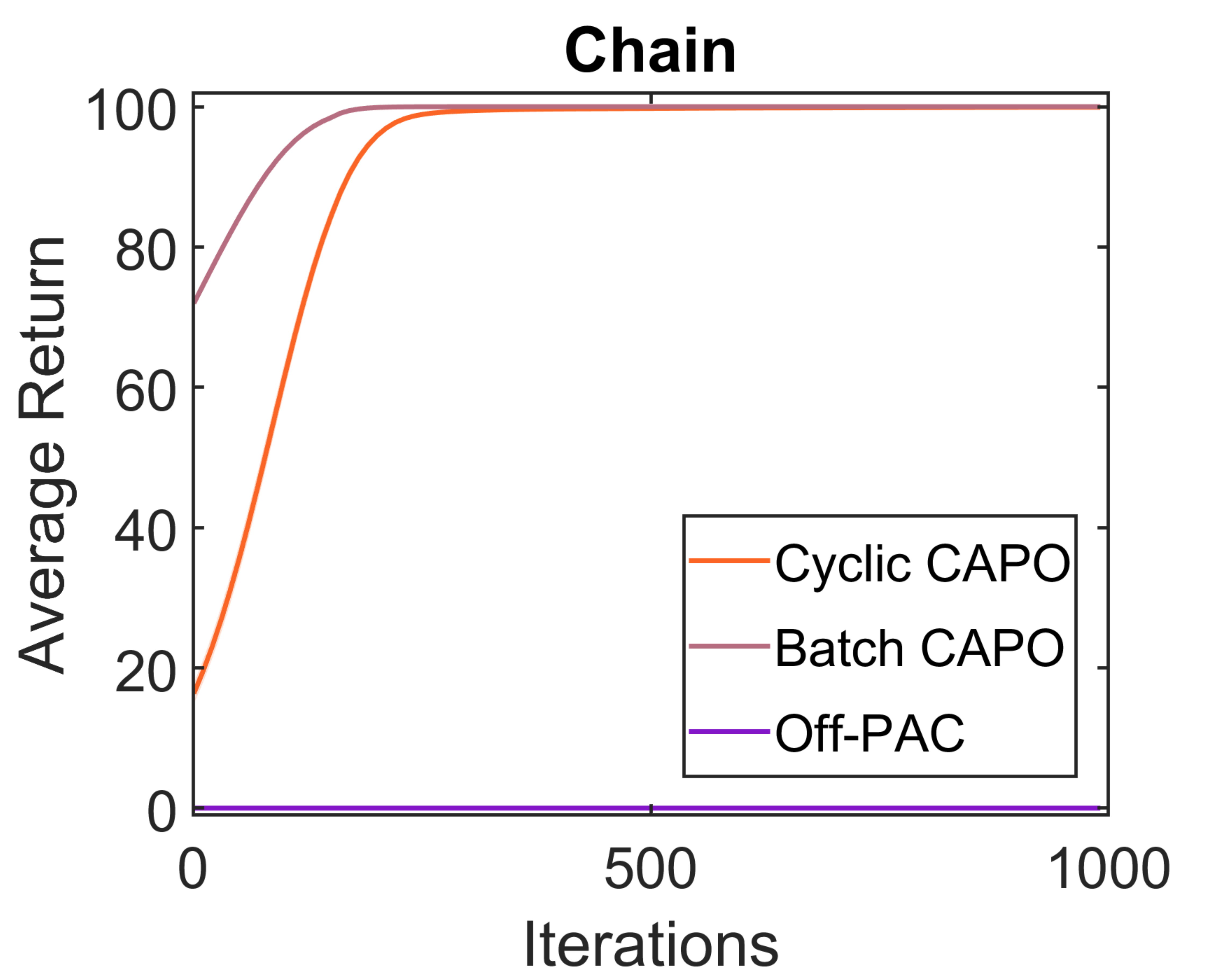}
  \caption{Comparison between Cyclic CAPO, Batch CAPO and off-policy Actor Critic, where the result is the average over 30 runs. We can see that despite the true value function is given and the optimal reward is much larger than the immediate reward (100 v.s. 0.1), Off-policy Actor Critic still suffers from a sub-optimal policy.}
  \label{fig:generator}
\end{figure}

%% file: appendices/appendix_experiments.tex
\section{Detailed Configuration of Experiments}
\label{app:exp}

\subsection{Implementation Detail}
\label{exp:imple}
Algorithm \ref{algo:NCAPO_imple} shows the pseudo code of NCAPO. 
In order to demonstrate the off-policy capability of NCAPO, we use simple $\epsilon$-greedy with initial exploration $\epsilon_{start}$ and decayed exploration for off-policy exploration and estimates $A(s,a)$ with Retrace \cite{munos2016safe}. 
NCAPO uses four simple 2-layer feed-forward neural networks, a behavior network ($\theta^b$), a target network ($\theta$), a critic network ($\theta^Q$) and a target critic network ($\theta^{Q^{\prime}}$).
In each episode, $N_{\text{rollouts}}$ rollouts are collected, 
and each rollout $r = [(s_t, a_t), ..., (s_{t+l}, a_{t+l})]$ has length of $l$. Note that instead of storing a single $(s,a)$-pair in the replay buffer $R$, we store the entire rollout of length $l$ into $R$ to better compute $Q_{\text{retrace}}$.
%Unlike most policy gradient methods, NCAPO does not require samples to follow any distribution (e.g. $d^{\pi}_{\rho}$). To make make maximum use of the replay buffer, in each episode, $\theta^b$ and $\theta^Q$ is updated with $\mathcal{G}$ gradient steps. In each gradient steps, no extra samples are generated by interacting with the environment. Instead, due to the off-policyness of NCAPO, and the capability to escape from early commitment, samples are sampled from the replay buffer multiple times, reusing old trajectories effectively. 
Due to the limited representation capability of floating-point numbers, during the CAPO update, the term $\log \frac{1}{\pi}$ can grow unbounded as $\pi \rightarrow 0$. To address this, we clip the term so that $\alpha(s,a) = \min(\log \frac{1}{\pi(a \mid s)}, \text{clip})$.
As the target networks have demonstrated the ability to stabilize training, the target networks are used and updated by polyak average update with coefficient $\tau_\theta$ and $\tau_{Q}$.
The experiment is conducted on a computational node equipped with Xeon Platinum 8160M CPU with a total of 40 cores. 
%For every seed, it takes approximately 12-24 hours depending on the seed and environment.
Off-PAC shares a similar code base as NCAPO, the major difference is that the use of a fixed behavior policy. We choose such behavior policy to be a uniform policy.
\subsection{Hyperparameters}
We use the hyperparameters for Atari games of stable-baseline3 for  \citep{stable-baselines3} for PPO and A2C, and the exact same hyperparameter from \citep{obando2020revisiting} for Rainbow.
The hyperparameters are listed in Table \ref{table:hyper}.

\begin{table}[!ht]
\caption{Hyperparameters for CAPO and OffPAC}
\label{table:hyper}
\vskip 0.15in
\begin{center}
\begin{small}
\begin{sc}
\begin{tabular}{lcccr}
\toprule
Hyperparameters    & CAPO  & PPO & A2C & OffPAC\\
\midrule
batch size         & 32 & 16 & - & 32\\
learning rate & 5e-4 & 2.5e-4 & 7e-4 & 5e-4 \\
exploration fraction & 10\% & 0 & 0 & 10\%\\
Initial exploration rate*  & 0.3 & 0 & 0 & 0\\ 
Final exploration rate  & 0.05 & 0 & 0 & 0\\
Critic loss coefficient* & 1 & 0.38 & 0.25 & 1\\
max gradient norm & 0.8 & 0.5 & 0.5 & 0.8\\
gradient steps & 30 & 1 & 1 & 30\\ 
train frequency & (64, steps) & (256, steps) & - & (64, steps)\\
$\tau_Q$       & 0.05  & - & - & 0.05\\
$\tau_\theta$ & 1  & - & - & 1\\
gamma         & 0.99 & 0.98 & 0.99 & 0.99\\
replay buffer & 6400 & -  & - & 6400\\
clip value & 50 & - & - & -\\
entropy coef & 0 & - & 4.04e-6 & 0 \\
\bottomrule
\end{tabular}
\end{sc}
\item[*] For \textit{Asterix}, the critic loss coefficient is $0.25$ and the initial exploration rate is $0.8$.
\end{small}
\end{center}
\vskip -0.1in
\end{table}

%% file: appendices/appendix_pseudocode.tex
\section{Pseudo Code of the Proposed Algorithms}
\begin{algorithm}[ht]
\caption{Coordinate Ascent Policy Optimization}
\label{algo:CAPO}
\begin{algorithmic}[1]
 \STATE Initialize policy $\pi_{\theta}$, $\theta \in \mathcal{S} \times \mathcal{A}$
 \FOR{$m = 1, \cdots, M$}
    \STATE Generate $|\mathcal{B}|$ state-action pairs $((s_0, a_0), ..., (s_{|\mathcal{B}|}, a_{|\mathcal{B}|}))$ from some coordinate selection rule satisfying \Cref{condition:sa}.
     \FOR{$i = 1, \cdots, |\mathcal{B}|$}
        %\STATE Execute action $a_i = \pi_m(\cdot | s_i)$ to obtain transition  $(s_i, a_i, r_i, s^{\prime}_{i})$
        \STATE $\theta_{m+1}(s_i, a_i) \leftarrow \theta_{m}(s_i, a_i) + \alpha_{m}(s_i, a_i) \sign\left(A^{m}(s_i, a_i)\right)$
    \ENDFOR
  \ENDFOR
\end{algorithmic}
\end{algorithm}

\begin{algorithm}[h]
\caption{Neural Coordinate Ascent Policy Optimization}
\label{algo:NCDPO}
\begin{algorithmic}[1]
 \STATE Initialize actor network $f_{\theta}$, where policy is parameterized as $\pi_{\theta}(a | s) = \frac{f_\theta(s, a)}{\sum_{a' \in \mathcal{A}} f_\theta(s, a')}$
 \FOR{$m = 1, \cdots, M$}
    \STATE Generate state-action pairs $\left((s_0, a_0), ..., (s_{|\mathcal{B}|}, a_{|\mathcal{B}|}\right))$ from some coordinate selection rule satisfying \Cref{condition:sa}.
    \STATE Evaluate Advantage $A^{\pi_{m}}$ with arbitrary policy evaluation algorithm.
    %\STATE $Loss = 0$
    %\FOR{$i = 1, \cdots, |\mathcal{B}|$}
        %\STATE Execute action $a_i = \pi_m(\cdot | s_i)$ to obtain transition  $(s_i, a_i, r_i, s^{'}_{i})$
        %\STATE $A(s_i, a_i) \leftarrow Evaluate(A)$ 
        % \STATE Initialize target $\hat{\theta}_{m+1}(s_i, \cdot) \leftarrow f_{\theta_{m}}(s_i, \cdot)$ 
        % \STATE $\hat{\theta}_{m+1}(s_i,a_i) = \hat{\theta}_{m+1}(s_i,a_i)   + \log(\frac{1}{\pi_{\hat{\theta}_m}(a_i \mid s_i)})\sign(A^{\pi_{m}}(s_i, a_i))$ 
        \STATE Compute target $\hat{\theta}$ by
        (\ref{eq:NCAPO_update}).
        
        \STATE Compute target policy $\hat{\pi}$ by taking softmax over $\hat{\theta}$.
        
        \STATE Update the policy network with NCAPO loss:
            \STATE $\nabla_{\theta}L = \nabla_{\theta} D_{KL}\left(\pi_{f_{\theta_{m}}} \| \hat{\pi}\right)$
    %\ENDFOR
    %\STATE $\theta_{m+1} = \theta_{m} - \nabla_{\theta}Loss$
  \ENDFOR
\end{algorithmic}
\end{algorithm}

\begin{algorithm}[ht]
\caption{Neural Coordinate Ascent Policy Optimization with Replay Buffer}
\label{algo:NCAPO_imple}
\begin{algorithmic}[1]
 \STATE Initialize behavior network $f(s,a \mid \theta^b)$, critic $Q(s,a \mid \theta^Q)$
 \STATE Initialize Replay Buffer $R$, 
 \STATE Initialize target networks $f(s, a \mid \theta) \leftarrow f(s, a \mid \theta^b)$, $Q(s,a | \theta^{Q^{\prime}}) \leftarrow Q(s,a | \theta^Q)$
 \FOR{episode $m = 1, \cdots, M$}
    \STATE Generate behavior policy and target policy by computing softmax $\pi_{\theta}(a \mid s) = \frac{ e^{f(s, a \mid \theta)} }{\sum_{a^{\prime} \in \mathcal{A}} e^{f\left(s, a^{\prime} \mid \theta \right)}}$.
    \STATE Collect $N_{rollouts}$ rollouts with length $l$ by following $\pi_{\theta_b}$ with decayed $\eps$-greedy, store rollouts to $R$. 
    \STATE Replace old rollouts if $len(R)> R_{max}$ .
    \FOR{gradient steps $= 1,...,\mathcal{G}$}
        \STATE Sample rollout $r$ from $R$.
        \STATE Compute $Q_{retrace}(s, a)$ for $(s,a) \in r$ 
        \STATE $\theta^Q_{Loss} \leftarrow \sum_{(s,a) \in r}(Q_{retrace}(s,a) - Q_{\theta^Q}(s,a))^2$
        \STATE $\theta_{Loss} \leftarrow  D_{KL}\left(\pi_{m}(\cdot \mid s) \mid \pi_{\hat{\theta}}(\cdot \mid s\right)$
        \STATE Update $Q(s,a \mid \theta^Q)$ with gradient $\nabla_{\theta^{Q}}\theta^{Q}_{Loss}$
        \STATE Update $f(s,a \mid \theta^b)$ with gradient $\nabla_{\theta}\theta_{Loss}$
    \ENDFOR
    \STATE Update target networks: \\
        $\theta^{Q^{\prime}}  \leftarrow \tau_{Q} \theta^{Q}+(1-\tau_{Q}) \theta^{Q^{\prime}}$ \\
        $\theta  \leftarrow \tau_{\theta} \theta^{b}+(1-\tau_{\theta}) \theta$

  \ENDFOR
\end{algorithmic}
\end{algorithm}

%% file: appendices/appendix_other.tex
\begin{comment}
\section{Coordinate descent methods.}
\label{app:related:coord}
Coordinate descent is one classic optimization method \citep{tseng2001convergence} and has been widely used for large-scale machine learning problems
\citep{nesterov2012efficiency}.
\citep{gurbuzbalaban2017cyclic}
There are various ways of selecting the coordinates for update, such as cyclic coordinate descent \cite{saha2013nonasymptotic}, randomized coordinate descent, and more.
The above list is by no means exhaustive but only to provide an overview of the relevant works on applying coordinate descent to learning problems.

Coordinate descent has also been utilized in deep learning \citep{zeng2019global}.

Recently, \citet{zhong2021coordinate} proposes to apply coordinate descent to design baseline functions.

\end{comment}